\documentclass{article} 
\usepackage{iclr2023_conference, times}

\usepackage{microtype}
\usepackage{graphicx}
\usepackage{subfigure}
\usepackage{booktabs} 
\usepackage{hyperref}
\usepackage{algorithm}
\usepackage{algpseudocode}

\usepackage{colortbl}

\usepackage{amsmath}
\usepackage{amssymb}
\usepackage{mathtools}
\usepackage{amsthm}

\usepackage[capitalize,noabbrev]{cleveref}
\usepackage{colortbl}
\usepackage{nicematrix}
\usepackage[inline, shortlabels]{enumitem}
\usepackage{multirow}

\usepackage{amsthm}
\usepackage{amsmath,amsfonts,bm}
\usepackage{nicefrac}
\usepackage{array}
\usepackage{wrapfig}

\usepackage{hyperref}
\usepackage{url}

\definecolor{bluex}{rgb}{0.27, 0.42, 0.81}
\definecolor{purplex}{HTML}{9564bf}

\theoremstyle{plain}
\newtheorem{theorem}{Theorem}[section]

\newtheorem{lemma}[theorem]{Lemma}
\newtheorem{corollary}[theorem]{Corollary}
\theoremstyle{definition}

\newtheorem{assumption}[theorem]{Assumption}
\theoremstyle{remark}

\newcommand{\bE}{\mathbb{E}}

\newcommand{\STAB}[1]{\begin{tabular}{@{}c@{}}#1\end{tabular}}

\newcommand{\eg}{e.g.}

\newcommand{\hA}{\mathcal{A}}
\newcommand{\hB}{\mathcal{B}}

\newcommand{\hD}{\mathcal{D}}

\newcommand{\hL}{\mathcal{L}}

\newcommand{\hN}{\mathcal{N}}

\newcommand{\hR}{\mathcal{R}}

\newcommand{\bI}{\mathbb{I}}

\newcommand{\vg}{{\bf g}}
\newcommand{\vh}{{\bf h}}

\newcommand{\vv}{{\bf v}}
\newcommand{\vw}{{\bf w}}
\newcommand{\vx}{{\bf x}}

\newcommand{\vX}{{\bf X}}

\DeclareMathOperator{\var}{\textsf{Var}}

\DeclareMathOperator{\diag}{diag}

\newcommand{\fracsam}{\text{\%SAM}}

\newcommand{\vepsilon}{{\boldsymbol \epsilon}}

\newcommand*\cboxa[1]{\colorbox[HTML]{c8e9c8}{#1}}
\newcommand*\cboxb[1]{\colorbox[HTML]{c5dced}{#1}}

\newcommand\Mark[1]{\textsuperscript#1}

\title{An Adaptive Policy to Employ \\ Sharpness-Aware Minimization}

\author{Weisen Jiang\Mark{{1, 2}}, Hansi Yang\Mark{2}, 
	Yu Zhang\Mark{{1, 3, \thanks{Correspondence to: Yu Zhang}}}, James Kwok\Mark{2} \\	
	\Mark{1}
	{\small
		{Guangdong Provincial Key Laboratory of Brain-inspired Intelligent Computation}} \\
	{\small  \;\; Department of Computer Science and Engineering,  Southern University of Science and Technology } \\
	\Mark{2}
	{\small Department of Computer Science and Engineering, 	Hong Kong University of Science and Technology}\\
	\Mark{3} 
	{\small Peng Cheng Laboratory} \\
	\texttt{\{wjiangar, hyangbw, jamesk\}@cse.ust.hk,  yu.zhang.ust@gmail.com}
}

\iclrfinalcopy

\begin{document}

	\maketitle
	
	\begin{abstract}
		Sharpness-aware minimization (SAM), which searches for flat minima by min-max
		optimization, has been shown to be useful in improving model generalization. 
		However, since each SAM update requires computing two gradients, its computational
		cost and training time are both doubled compared to standard empirical risk
		minimization (ERM).  
		Recent state-of-the-arts 
		reduce the fraction of SAM updates and thus
		accelerate SAM by  switching between
		SAM and ERM 
		updates
		randomly or periodically.  
		In this paper, we design an adaptive policy to employ SAM based on the loss landscape geometry.  
		Two efficient algorithms, AE-SAM and AE-LookSAM,
		are proposed. 
		We theoretically show that 
		AE-SAM
		has the same convergence rate as SAM.
		Experimental results on various datasets and architectures
		demonstrate the efficiency and effectiveness
		of the adaptive policy.
	\end{abstract}
	
	\section{Introduction}
	
	Despite great success 
	in many 
	applications~\citep{he2016deep,
		zagoruyko2016wide,han2017deep},
	deep networks are often over-parameterized 
	and capable of memorizing all training data.
	The training loss landscape 
	is
	complex and nonconvex 
	with many local minima 
	of different generalization abilities. 
	Many studies have
	investigated the relationship 
	between the
	loss surface's
	geometry 
	and generalization
	performance~\citep{hochreiter1994simplifying, mcallester1999pac, keskar2017on, neyshabur2017exploring, Jiang2020Fantastic},
	and found that flatter minima 
	generalize better than sharper minima~\citep{karolina2017, petzka2021relative, chaudhari2017entropysgd, keskar2017on, Jiang2020Fantastic}.
	
	Sharpness-aware minimization~(SAM)~\citep{foret2021sharpness}
	is the current state-of-the-art 
	to seek flat minima
	by solving a 
	min-max optimization problem.
	In the SAM algorithm, 
	each update consists of 
	\textit{two} forward-backward 
	computations:
	one for computing the perturbation
	and 
	the other for computing the actual update direction.
	Since these two computations are not parallelizable,
	SAM doubles the computational overhead as well as the training time compared to empirical risk minimization~(ERM).
	
	Several algorithms~\citep{du2021efficient, zhao2022ss, liu2022towards} have 
	been proposed to improve the efficiency of SAM.
	ESAM~\citep{du2021efficient}
	uses fewer samples to compute the 
	gradients and updates fewer parameters,
	but each update still requires two gradient computations. 
	Thus, ESAM
	does not alleviate the 
	bottleneck of training speed.
	Instead of using the SAM update at every iteration,
	recent state-of-the-arts~\citep{zhao2022ss, liu2022towards} proposed to use SAM randomly or periodically.
	Specifically,
	SS-SAM~\citep{zhao2022ss} selects 
	SAM or ERM  according to a 
	Bernoulli trial, 
	while LookSAM~\citep{liu2022towards} 
	employs SAM at every $k$ step.
	Though more efficient, the random or periodic use of SAM is suboptimal
	as it is not geometry-aware.
	Intuitively, 
	the SAM update is more useful in sharp regions 
	than in flat regions.
	
	In this paper,
	we propose an adaptive policy to 
	employ SAM based 
	on the geometry of the loss landscape.
	The SAM update is used when
	the model is in sharp regions,
	while the ERM update is used in 
	flat regions 
	for reducing the fraction of SAM updates.
	To measure sharpness,
	we use the squared stochastic gradient norm and
	model it by a normal distribution, whose
	parameters
	are estimated 
	by exponential moving average.
	Experimental results
	on standard benchmark datasets demonstrate
	the superiority of the proposed policy.
	
	Our contributions are summarized as follows:
	\begin{enumerate*}[(i), series = tobecont, itemjoin = \quad]
		\item We propose an adaptive policy to use SAM or ERM update based on the loss landscape geometry.
		\item We propose 
		an efficient algorithm, called AE-SAM (\underline{A}daptive policy to \underline{E}mploy
		\underline{SAM}),
		to reduce the fraction of SAM updates. We also 
		theoretically study
		its convergence rate.
		\item 
		The proposed policy is general and can be combined with any SAM variant.
		In this paper,
		we integrate it with 
		LookSAM~\citep{liu2022towards}
		and propose AE-LookSAM.
		\item
		Experimental results on various 
		network architectures
		and 
		datasets 
		(with and
		without label noise) verify the superiority of AE-SAM and AE-LookSAM over
		existing baselines.
	\end{enumerate*}
	
	\textbf{Notations}.
	Vectors (\eg, $\vx$) and matrices (\eg, $\vX$) are denoted by lowercase and uppercase boldface letters, respectively.
	For a vector $\vx$, its $\ell_2$-norm is $\|\vx\|$.
	$\hN(\mu; \sigma^2)$ is the univariate normal distribution with mean $\mu$ and variance $\sigma^2$.
	$\diag(\vx)$ constructs a diagonal matrix with $\vx$ on the diagonal.
	Moreover, $\bI_{A}(x)$ denotes the indicator function for a given set $A$, i.e.,
	$\bI_A(x) = 1$ if $x\in A$, and
	$0$ otherwise.
	
	\section{Related Work}
	
	We are given a training set $\hD$
	with i.i.d. samples 
	$\{(\vx_i, y_i): i=1, \dots, n\}$.
	Let $f(\vx; \vw)$ be a model parameterized by $\vw$. Its empirical risk on $\hD$  is $\hL(\hD; \vw) = \frac{1}{n} \sum_{i=1}^n \ell(f(\vx_i; \vw), y_i)$, where $\ell(\cdot, \cdot)$ is a loss (e.g., cross-entropy loss for classification).  Model training
	aims to learn a model from the training data
	that generalizes well on the test data.
	
	\textbf{Generalization and Flat Minima.}
	The connection between model generalization and loss landscape geometry 
	has been theoretically and empirically studied in \citep{keskar2017on, karolina2017, Jiang2020Fantastic}.
	Recently, \cite{Jiang2020Fantastic}
	conducted large-scale experiments
	and find that sharpness-based measures (flatness)
	are related to generalization of minimizers.
	Although flatness can be 
	characterized by 
	the Hessian's eigenvalues~\citep{keskar2017on, dinh2017sharp},
	handling the Hessian explicitly is computationally prohibitive.  To address this
	issue, practical algorithms propose to seek flat minima by injecting noise into the optimizers \citep{zhu2019anisotropic,
		zhou2019toward, orvieto2022anticorrelated, bisla2022low},
	introducing regularization~\citep{chaudhari2017entropysgd, zhao2022penalizing, du2022sharpness}, 
	averaging model weights during training~\citep{izmailov2018averaging, he2019asymmetric, cha2021swad},
	or sharpness-aware minimization (SAM)~\citep{foret2021sharpness, kwon2021asam,
		zhuang2021surrogate, 
		kim2022fisher}.
	
	\textbf{SAM.}
	The state-of-the-art SAM \citep{foret2021sharpness} and its variants \citep{
		kwon2021asam, zhuang2021surrogate, kim2022fisher, zhao2022penalizing} search for flat minima by
	solving the following min-max optimization problem:
	\begin{align}
		\min_{\vw} \max_{\| \vepsilon\| \leq \rho } \hL(\hD; \vw + \vepsilon),
		\label{eq:sam}
	\end{align}
	where $\rho > 0$ is the radius of perturbation.  The above can also be rewritten as 
	$\min_\vw \hL(\hD; \vw) + \hR(\hD; \vw)$,
	where  $\hR(\hD;\vw)\equiv\max_{\| \vepsilon\| \leq \rho }\hL(\hD; \vw+\vepsilon) - \hL(\hD; \vw)$ is a regularizer that penalizes sharp minimizers~\citep{foret2021sharpness}.
	As solving the inner maximization in \eqref{eq:sam} exactly is computationally
	infeasible for nonconvex losses,
	SAM approximately solves it by first-order
	Taylor approximation, leading to the update rule:
	\begin{align}
		\vw_{t+1} = \vw_{t} - \eta \nabla \hL(\hB_t; \vw_t + \rho_t \nabla \hL(\hB_t; \vw_t)),
		\label{eq:sam-rule}
	\end{align}
	where $\hB_t$ is a mini-batch of data,
	$\eta$ is the step size,
	and $\rho_t=\frac{\rho}{\| \nabla \hL(\hB_t; \vw_{t})\| }$.
	Although SAM has shown to be effective in 
	improving the generalization 
	of deep networks, 
	a major drawback is that each update in \eqref{eq:sam-rule} 
	requires \textit{two} forward-backward calculations.
	Specifically, 
	SAM first calculates the gradient of $\hL(\hB_t; \vw)$ at $\vw_t$ 
	to obtain the perturbation,
	then calculates the gradient of $\hL(\hB_t; \vw)$
	at $\vw_t + \rho_t \nabla \hL(\hB_t; \vw_t)$
	to obtain the update direction for $\vw_t$.
	As a result, SAM doubles the computational overhead compared to ERM.
	
	\textbf{Efficient Variants of SAM.}
	Several algorithms have been proposed to accelerate the SAM algorithm.
	ESAM \citep{du2021efficient} uses fewer samples  to compute the gradients 
	and only updates part of the model in the second step,
	but still requires to compute most of the gradients.
	Another direction is 
	to 
	reduce the number of SAM updates during training. 
	SS-SAM~\citep{zhao2022ss} randomly selects SAM or ERM update according to 
	a Bernoulli trial,
	while
	LookSAM~\citep{liu2022towards} employs SAM at every $k$ iterations.
	Intuitively, the SAM update is more suitable for sharp regions than flat regions.
	However, the mixing policies in SS-SAM and LookSAM are not adaptive to the loss landscape.
	In this paper, we design an adaptive policy to employ SAM based on the loss landscape geometry.
	
	\section{Method}
	\label{sec:method}
	
	In this section, we propose an adaptive policy to employ SAM.  
	The
	idea is to use ERM when $\vw_t$ is in a flat region, and use SAM
	only when the loss landscape is locally sharp.  
	We start by introducing a sharpness
	measure (Section~\ref{sec:measure}),
	then propose an adaptive policy based on this
	(Section~\ref{sec:adaptive}).
	Next, we propose two algorithms (AE-SAM and AE-LookSAM) and study the
	convergence.
	
	\subsection{Sharpness Measure}
	\label{sec:measure}
	Though sharpness can be characterized by Hessian's
	eigenvalues~\citep{keskar2017on, dinh2017sharp}, they are expensive to compute.
	A widely-used approximation 
	is based on  the gradient magnitude
	$\diag([\nabla \hL(\hB_t; \vw_t)]^2)$
	\citep{bottou2018optimization, khan2018fast}, where $[\vv]^2$ denotes the
	elementwise square of a vector $\vv$.  As $\| \nabla \hL(\hB_t; \vw_t) \|^2$ equals the trace
	of $\diag([\nabla \hL(\hB_t; \vw_t)]^2)$, it is reasonable to choose $\| \nabla
	\hL(\hB_t; \vw_t) \|^2$ as a sharpness measure.
	
	$\| \nabla \hL(\hB_t; \vw_t) \|^2$ is also related to 
	the gradient variance 
	$\var(\nabla \hL(\hB_t; \vw_t) )$,  another sharpness measure~\citep{Jiang2020Fantastic}.  
	Specifically,
	\begin{align}
		\!\!
		\var(\nabla \hL(\hB_t; \vw_t) ) \!\equiv\! \bE_{\hB_t} \| \nabla \hL(\hB_t; \vw_t) \!-\! \nabla \hL(\hD; \vw_t)\|^2 \!=\! \bE_{\hB_t} \| \nabla \hL(\hB_t; \vw_t)  \|^2 \!-\! \| \nabla \hL(\hD; \vw_t) \|^2. \!\!\label{eq:var}
	\end{align}
	With appropriate smoothness assumptions on 
	$\hL$,
	both
	SAM and ERM 
	can be shown
	theoretically
	to 
	converge  to
	critical points of 
	$\hL(\hD;\vw)$
	(i.e., $\nabla \hL(\hD;\vw) = 0$)
	\citep{reddi2016stochastic,andriushchenko2022towards}.  Thus, it follows from
	\eqref{eq:var} that 
	$\var(\nabla \hL(\hB_t; \vw_t) ) = \bE_{\hB_t} \| \nabla \hL(\hB_t; \vw_t)  \|^2$ when $\vw_t$ is a critical point of $\hL(\hD; \vw)$.
	\cite{Jiang2020Fantastic} conducted
	extensive experiments and empirically show that 
	$\var(\nabla \hL(\hB_t; \vw_t) )$
	is positively correlated with the generalization gap.  The smaller
	the $\var(\nabla \hL(\hB_t; \vw_t) )$, the better 
	generalization is the model with parameter $\vw_t$.  This finding also explains why SAM generalizes better than ERM.
	Figure \ref{fig:variance-gradient-trend} shows the gradient variance w.r.t.
	the
	number of epochs 
	using SAM and ERM 
	on \textit{CIFAR-100} 
	with various network architectures (experimental
	details are in Section \ref{sec:setup-cifar}).
	As can be seen,
	SAM always has a much smaller 
	variance than ERM.
	Figure
	\ref{fig:sgd-squared-grad-norm-trend} shows the
	expected squared 
	norm
	of the stochastic gradient 
	w.r.t. the number of epochs on \textit{CIFAR-100}.  As shown,
	SAM achieves a much smaller 
	$\bE_{\hB_t} \| \nabla \hL(\hB_t; \vw_t)  \|^2$ than 
	ERM.
	
	\begin{figure}[!t]
		\centering
		\vskip -.1in
		\!\!\!
		\subfigure[\textit{ResNet-18}. \label{fig:variance-grad-norm-resnet18}]{\includegraphics[width=0.32\textwidth]{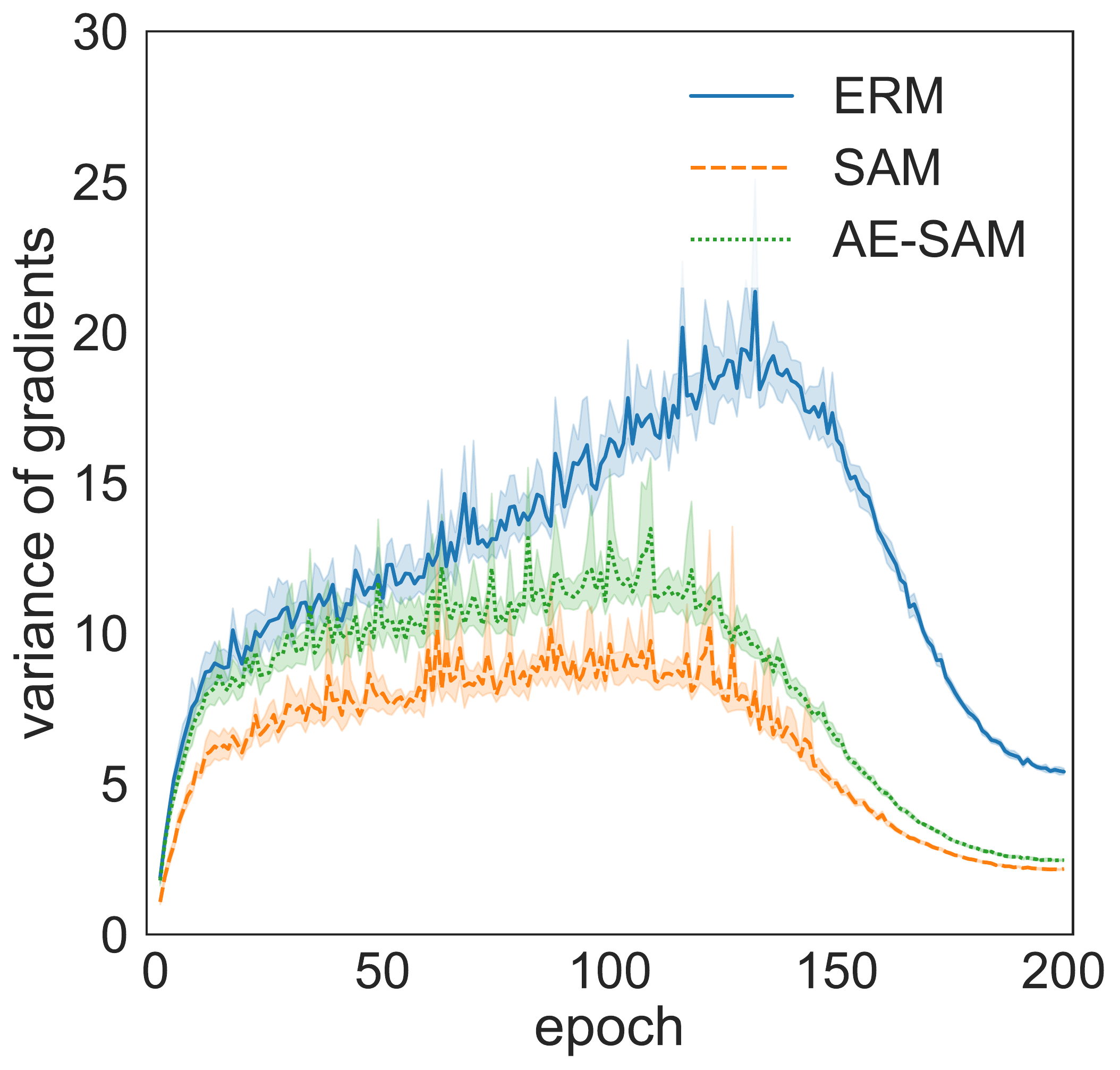}}
		\!\!\!
		\subfigure[\textit{WRN-28-10}. \label{fig:variance-grad-norm-wrn}]{\includegraphics[width=0.32\textwidth]{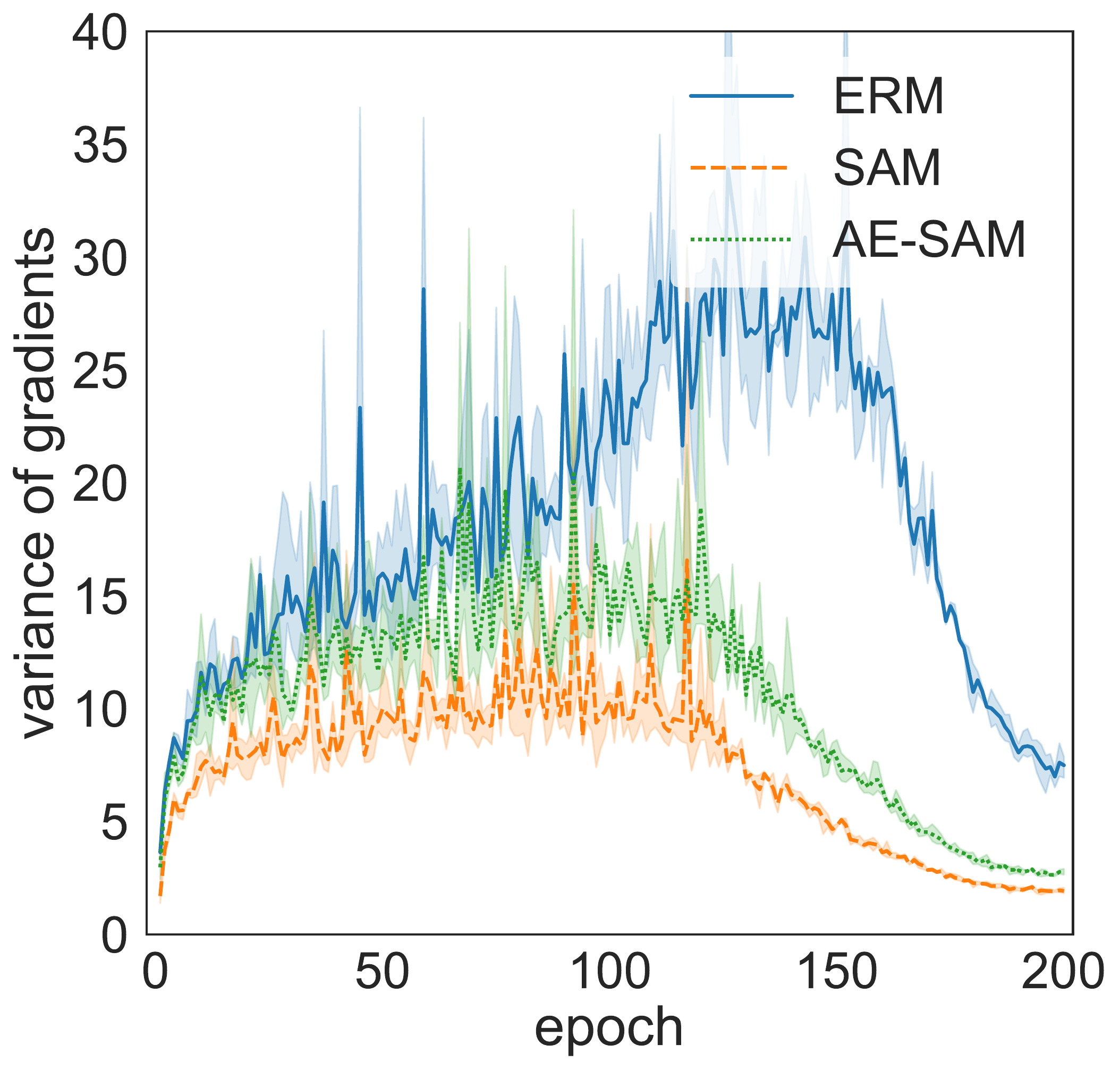}}
		\!\!\!
		\subfigure[\textit{PyramidNet-110}. \label{fig:variance-grad-norm-pyramidnet}]{\includegraphics[width=0.32\textwidth]{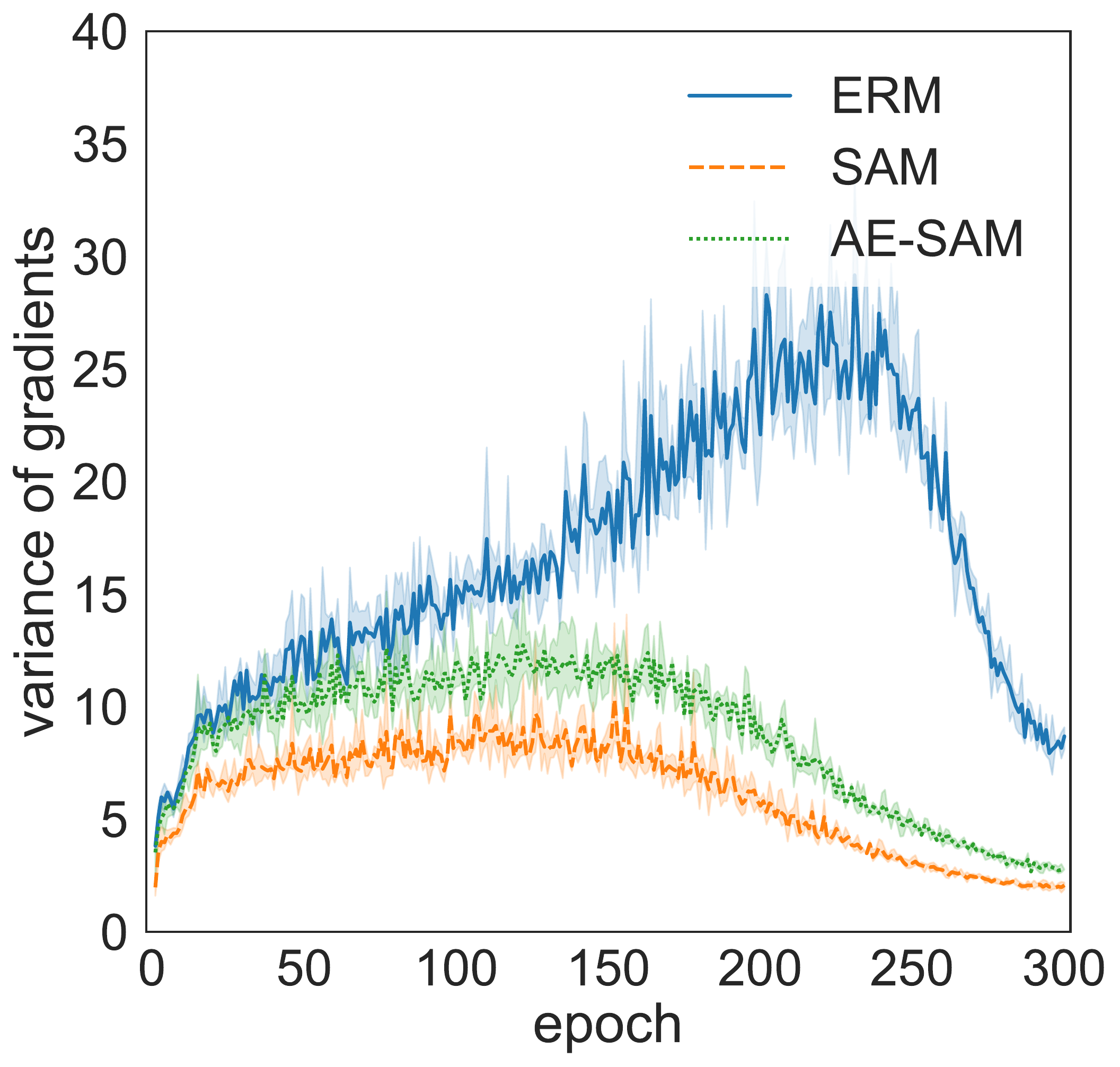}}
		\!\!\!
		\vskip -.2in
		\caption{Variance of gradient
			on \textit{CIFAR-100}. 
			Best viewed in color.
		}
		\label{fig:variance-gradient-trend}
	\end{figure}
	
	\begin{figure}[!t]
		\centering
		\vskip -.1in
		\!\!\!
		\subfigure[\textit{ResNet-18}.]
		{\includegraphics[width=0.32\textwidth]{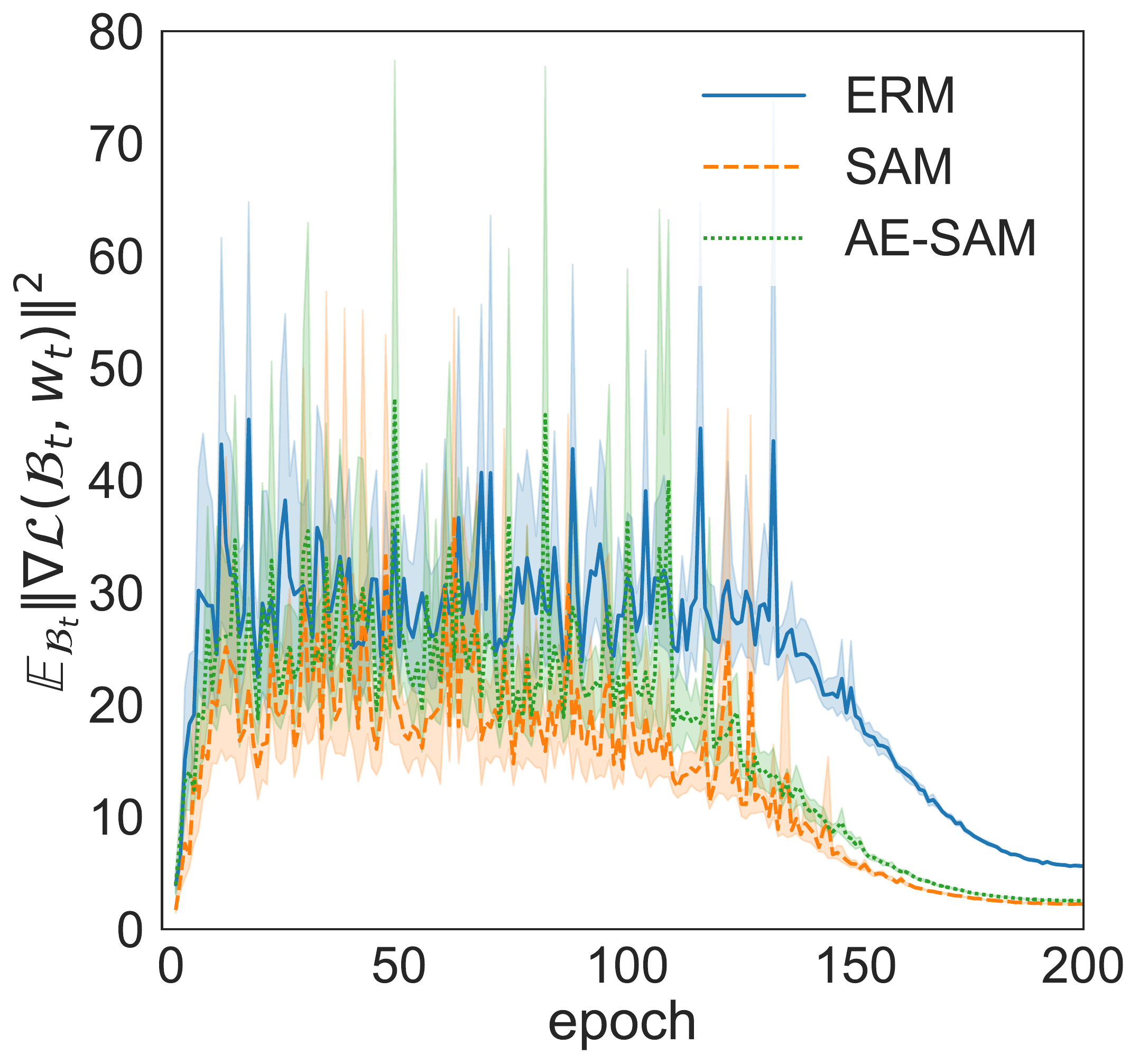}}
		\!\!\!
		\subfigure[\textit{WRN-28-10}.]
		{\includegraphics[width=0.33\textwidth]{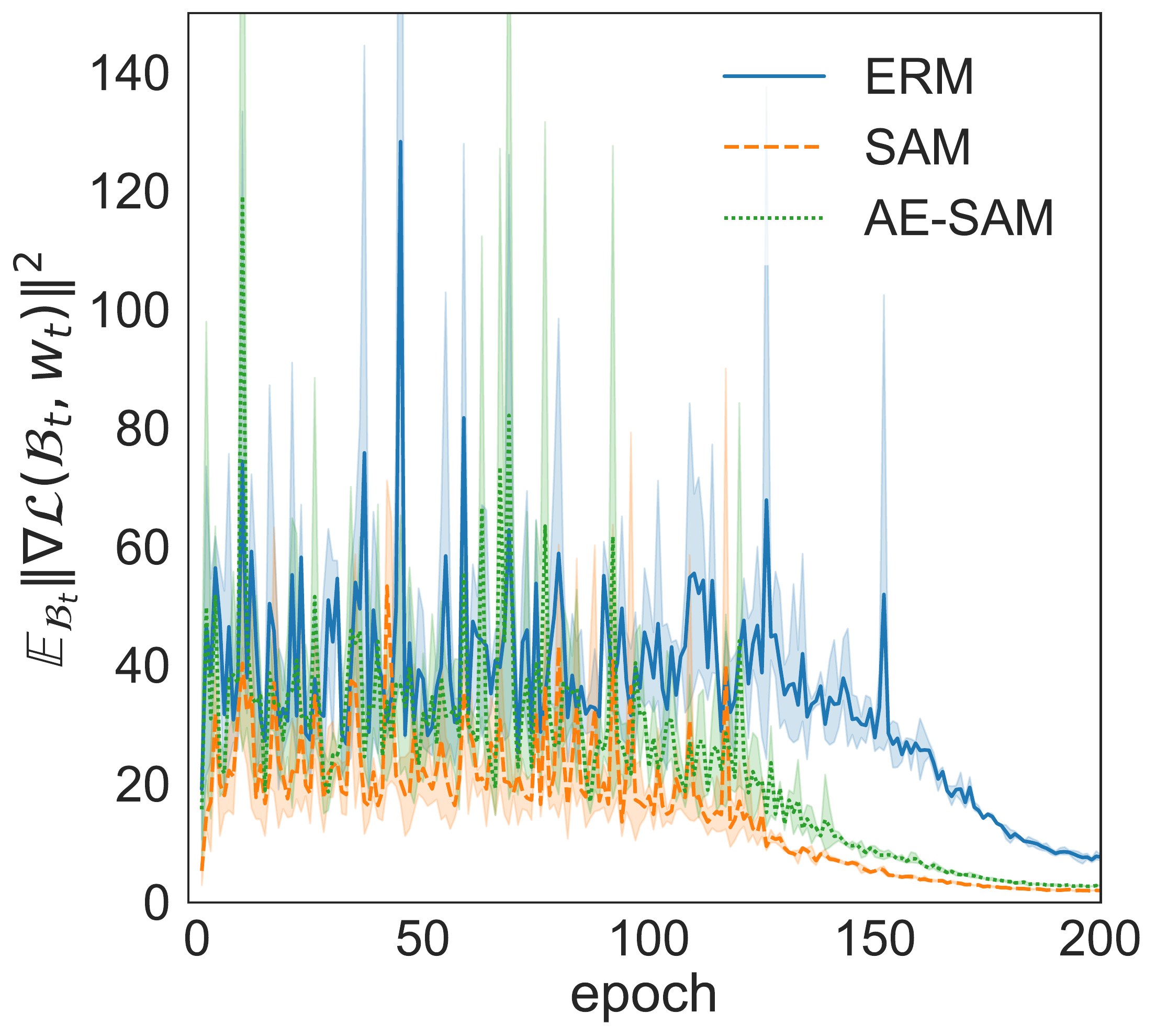}}
		\!\!\!
		\subfigure[\textit{PyramidNet-110}.]
		{\includegraphics[width=0.32\textwidth]{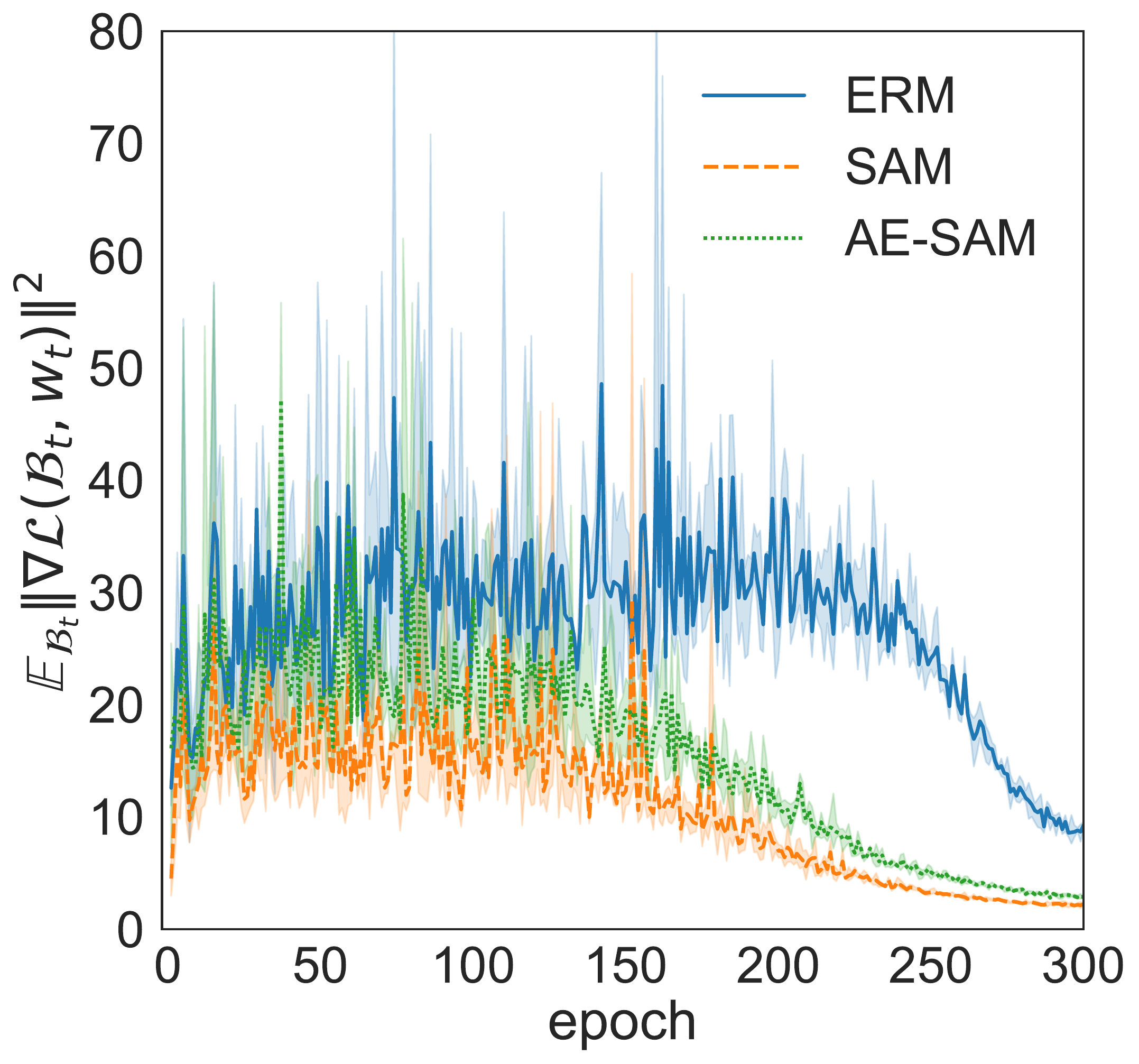}}
		\!\!\!
		\vskip -.2in
		\caption{Squared stochastic gradient norms $\bE_{\hB} \| \nabla \hL(\hB; \vw_t)  \|^2$
			on \textit{CIFAR-100}. 
			Best viewed in color.
		}
		\label{fig:sgd-squared-grad-norm-trend}
	\end{figure}
	
	\subsection{Adaptive Policy to Employ SAM}
	\label{sec:adaptive}
	
	As 
	$\bE_{\hB_t} \| \nabla \hL(\hB_t; \vw_t) \|^2$ changes with $t$ (Figure
	\ref{fig:sgd-squared-grad-norm-trend}),
	the sharpness at $\vw_t$ also changes along the optimization trajectory. 
	As a result,
	we need to estimate $\bE_{\hB_t} \| \nabla \hL(\hB_t; \vw_t) \|^2$
	at every
	iteration.
	One can sample a large number of mini-batches and compute the mean of the
	stochastic gradient norms.  However, this can be computationally
	expensive.
	To address this problem,
	we model $\| \nabla \hL(\hB_t; \vw_t) \|^2$
	with a simple distribution and estimate the distribution parameters in an online manner.
	Figure \ref{fig:grad-norm-dist} shows $\| \nabla \hL(\hB_t; \vw_t) \|^2$ of $400$
	mini-batches at different training stages (epoch $= 60$, $120$, and $180$) on
	\textit{CIFAR-100} using \textit{ResNet-18}\footnote{Results on other
		architectures and \textit{CIFAR-10} are shown in Figures \ref{fig-apd:gradnorm-dist-cifar} and \ref{fig-apd:gradnorm-qq-cifar} of Appendix \ref{apd:dist-gdn}.}.  As can be seen, the distribution
	follows a Bell curve.
	Figure \ref{fig:grad-norm-qqplot} shows the corresponding
	quantile-quantile (Q-Q) plot~\citep{wilk1968probability}.
	The
	closer is the curve to a
	line, the distribution is closer to the normal
	distribution.  Figure~\ref{fig:grad-norm} 
	suggests that $\| \nabla \hL(\hB_t; \vw_t)\|^2$ can be
	modeled\footnote{Note that normality is 
		not needed
		in the theoretical analysis (Section \ref{sec:convergence}).}
	with a 
	normal distribution $\hN(\mu_t, \sigma_t^2)$.
	We use exponential moving average (EMA),
	which is popularly used in adaptive gradient methods (e.g., RMSProp~\citep{tieleman2012lecture}, AdaDelta~\citep{zeiler2012adadelta}, Adam~\citep{kingma2015adam}),
	to
	estimate its mean and variance:
	\begin{align}
		\mu_t &= \delta \mu_{t-1} + (1-\delta)  \|\nabla \hL(\hB_t; \vw_t) \|^2, \label{eq:mu-update} \\
		\sigma_t^2 &= \delta \sigma_{t-1}^2 + (1-\delta)   (\|\nabla \hL(\hB_t; \vw_t) \|^2 - \mu_{t})^2, \label{eq:sigma-update}
	\end{align}
	where $\delta \in (0,1)$ controls the forgetting rate.
	Empirically, we use $\delta = 0.9$.
	Since $\nabla \hL(\hB_t; \vw_t)$
	is already available during training,   this
	EMA update does not involve additional gradient calculations (the cost for the norm operator is negligible).
	
	\begin{figure}[h]
		\vskip -.1in
		\centering
		\subfigure[Distributions. \label{fig:grad-norm-dist}]{\includegraphics[width=0.445\textwidth]{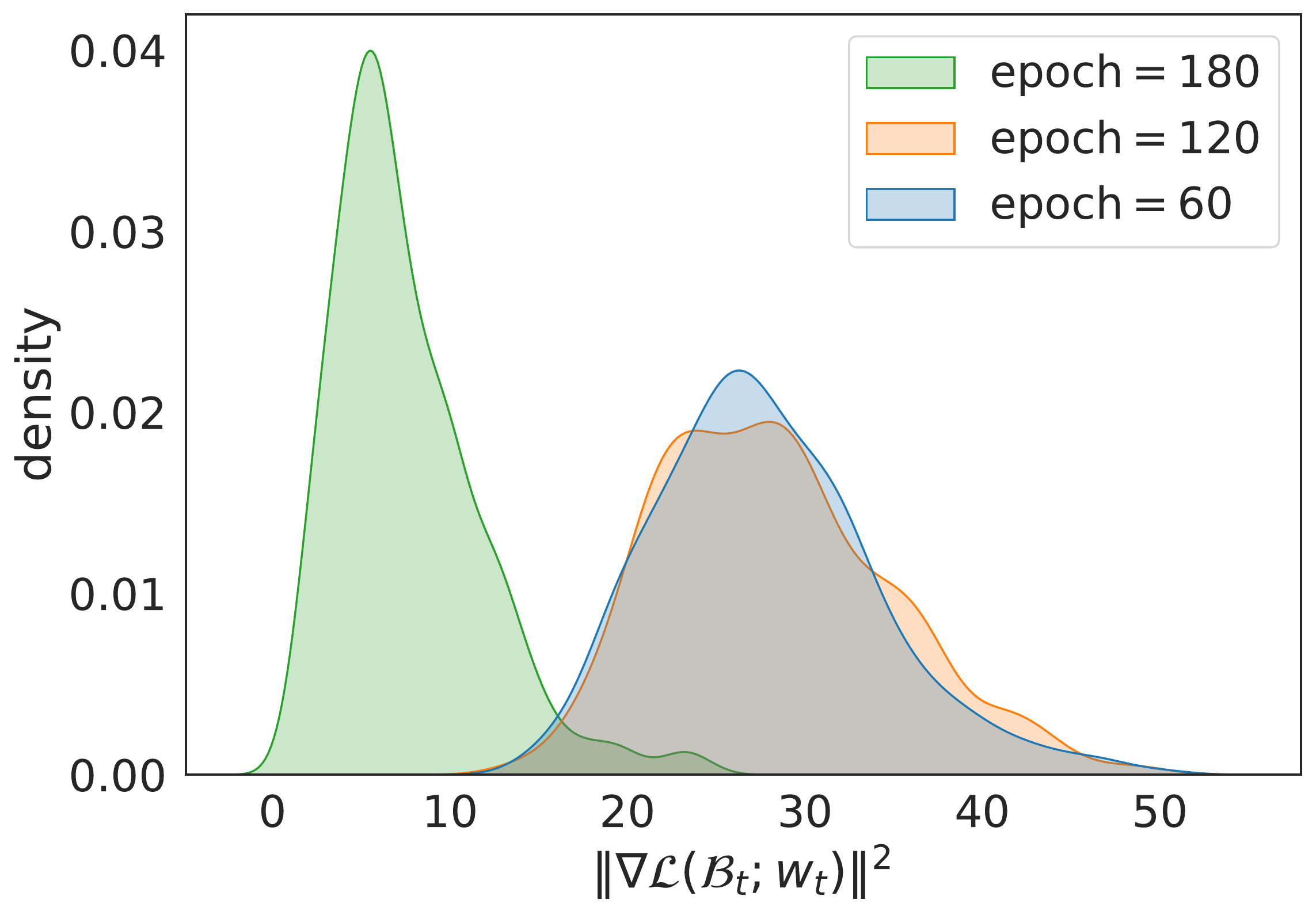}}
		\quad
		\subfigure[Q-Q plots.\label{fig:grad-norm-qqplot}]{\includegraphics[width=0.32\textwidth]{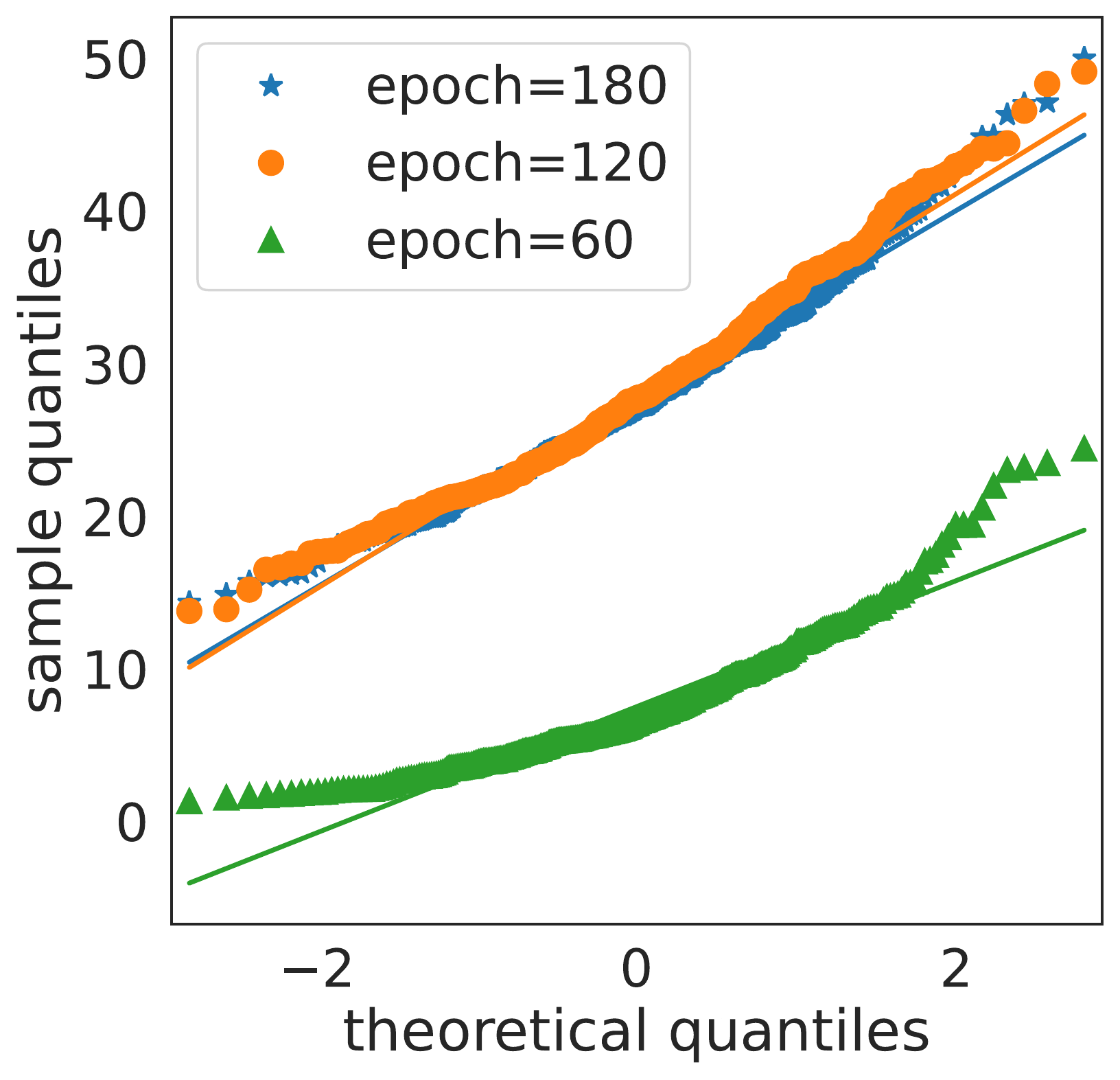}}	
		\vskip -.15in
		\caption{Stochastic gradient norms $\{\| \hL(\hB_t; \vw_t) \|^2: \hB_t\sim \hD\}$ of \textit{ResNet-18} on \textit{CIFAR-100} are approximately normally distributed.
			Best viewed in color.
		}
		\label{fig:grad-norm}
	\end{figure}
	
	Using $\mu_t$ and $\sigma_t^2$, we employ SAM only at iterations where $\| \nabla
	\hL(\hB_t; \vw_t)\|^2$ is relatively large (i.e., the loss landscape is locally sharp).  Specifically, when $\|\nabla
	\hL(\hB_t; \vw_t) \|^2 \geq \mu_t + c_t\sigma_t$ (where $c_t$ is a threshold), SAM
	is used; otherwise, ERM is used.  
	When $c_t\to -\infty$, it reduces to SAM; when $c_t\to \infty$, it becomes ERM.
	Note that during the early
	training stage, the model 
	is still underfitting and 
	$\vw_t$  is far from the region of final convergence.  Thus, minimizing the
	empirical loss is more important than seeking a locally flat region.
	\cite{andriushchenko2022towards} also empirically observe that the SAM update is
	more effective in boosting performance towards the end of training.  We therefore design a schedule 
	that linearly decreases $c_t$ from $\lambda_2$ to $\lambda_1$ (which are pre-set values):
	$c_t=g_{\lambda_1, \lambda_2}(t)\equiv
	\frac{t}{T} \lambda_1 + \left(1-\frac{t}{T}\right)\lambda_2  $,
	where $T$ is the total number of iterations.
	The whole procedure, called \underline{A}daptive policy to \underline{E}mploy
	\underline{SAM} (\textbf{AE-SAM}),
	is shown in Algorithm \ref{alg:etsam}.
	
	\textbf{AE-LookSAM}.
	The proposed adaptive policy can be combined with any SAM variant. Here, we
	consider integrating it with LookSAM~\citep{liu2022towards}.
	When $\|\nabla \hL(\hB_t; \vw_t) \|^2 \geq \mu_t + c_t \sigma_t$, SAM is used and
	the update direction for $\vw_t$ is decomposed into two orthogonal directions as
	in LookSAM: (i) the ERM update direction to reduce training loss, and  (ii)
	the direction  that
	biases the model to a flat region.
	When $\|\nabla \hL(\hB_t; \vw_t) \|^2 < \mu_t + c_t \sigma_t$,
	ERM is performed and the second direction of
	the previous SAM update is reused to
	compose an approximate SAM direction.
	The procedure, called AE-LookSAM, is also shown in Algorithm \ref{alg:etsam}.	
	
	\begin{algorithm}[!h]
		\caption{\cboxb{AE-SAM} and \cboxa{AE-LookSAM}.}
		\label{alg:etsam}
		\begin{algorithmic}[1]
			\Require training set $\hD$, stepsize $\eta$, radius $\rho$; $\lambda_1$ and $\lambda_2$ for $g_{\lambda_1, \lambda_2}(t)$; $\vw_0$, $\mu_{-1} = 0$, $\sigma_{-1}^2=e^{-10}$, and
			$\alpha$ for AE-LookSAM;
			
			\For{$t=0,\dots, T-1$}
			\State sample a mini-batch data $\hB_t$ from $\hD$;
			\State compute $\vg=\nabla \hL(\hB_t; \vw_t)$;
			\State update $\mu_t$ by \eqref{eq:mu-update} and $\sigma_t^2$
			by \eqref{eq:sigma-update};
			\State compute $c_t=g_{\lambda_1, \lambda_2}(t)$;
			\If{$\|\nabla \hL(\hB_t; \vw_t) \|^2 \geq \mu_t + c_t \sigma_t$}
			\State $\vg_s =  \nabla \hL(\hB_t; \vw_t + \rho \nabla \hL(\hB_t;\vw_t))$;
			\State \cboxa{if AE-LookSAM: decompose $\vg_s$ as $\vg_v = \vg_s - \frac{\vg^\top \vg_s}{\| \vg\|^2} \vg$;}
			\Else:
			\State \cboxb{if AE-SAM: $\vg_s = \vg$;}
			\State \cboxa{if AE-LookSAM: $\vg_s = \vg + \alpha \frac{\|\vg\|}{\|\vg_v\|} \vg_v$;}
			\EndIf
			\State $\vw_{t+1} = \vw_t - \eta \vg_s$;
			\EndFor \\
			\Return $\vw_T$.
		\end{algorithmic}
	\end{algorithm}
	
	\subsection{Convergence Analysis}
	\label{sec:convergence}

	In this section,
	we study the convergence of 
	any algorithm $\hA$ whose
	update in each iteration can be either SAM or ERM. 
	Due to this mixing of
	SAM and ERM updates, 
	analyzing
	its convergence is more challenging
	compared with that of SAM.
	
	The following assumptions on 
	smoothness
	and  bounded variance of stochastic gradients
	are standard in the literature on
	non-convex optimization~\citep{ghadimi2013stochastic, reddi2016stochastic}   
	and SAM~\citep{andriushchenko2022towards, abbas2022sharp, qu2022generalized}.
	
	\begin{assumption}[Smoothness]
		\label{ass:smooth}
		$\hL(\hD; \vw)$ is $\beta$-smooth in $\vw$, 
		i.e., $\|\nabla \hL(\hD;\vw) - \nabla \hL(\hD;\vv) \| \leq \beta \| \vw - \vv \|$.
	\end{assumption}
	
	\begin{assumption}[Bounded variance of stochastic gradients]
		\label{ass:bd-noise}
		$\bE_{(\vx_i, y_i) \sim \hD} \| \nabla \ell(f(\vx_i;\vw), y_i) - \nabla \hL(\hD; \vw) \|^2 \leq \sigma^2$.
	\end{assumption}
	
	Let $\xi_t$ be an indicator of whether SAM or ERM is used at iteration $t$ (i.e., $\xi_t=1$ for SAM, and $0$ for ERM).
	For example,
	$\xi_t = \bI_{\{\vw: \|\nabla \hL(\hB_t; \vw) \|^2 \geq \mu_t + c_t\sigma_t\}}(\vw_t)$
	for the proposed AE-SAM,
	and $\xi_t$ is sampled from a Bernoulli distribution
	for SS-SAM~\citep{zhao2022ss}.	
	
	\begin{theorem}
		\label{thm:conv-sgd}
		Let $b$ be the mini-batch size.  If stepsize $\eta = \frac{1}{4\beta\sqrt{T}}$ and
		$\rho = \frac{1}{T^{\frac{1}{4}}}$, algorithm $\hA$
		satisfies
		\begin{align}
			\min_{0\leq t\leq T-1} \bE  \| \nabla \hL(\hD; \vw_t) \|^2
			&\leq  \frac{32\beta \left(\hL(\hD; \vw_0) - \bE \hL(\hD; \vw_T)\right)}{\sqrt{T} \left(7-6\zeta\right)} 
			+  \frac{ (1 + \zeta +  5 \beta^2 \zeta )\sigma^2}{b\sqrt{T}\left(7-6\zeta\right)}, \label{eq:conv-bound-1} 
		\end{align}
		where 
		$\zeta =  \frac{1}{T} \sum_{t=0}^{T-1}\xi_t \in [0,1]$
		is the fraction of SAM updates,
		and  the expectation is taken over the random
		training samples.
	\end{theorem}
	All proofs are in Appendix \ref{sec:proof}.
	Note that 
	a larger $\zeta$
	leads to a larger upper bound in \eqref{eq:conv-bound-1}.
	When $\zeta=1$,
	the above reduces to SAM
	(Corollary \ref{cor:conv-sam} of Appendix \ref{appendix:proof of thm}).
	
	\section{Experiments}
	\label{sec:expt}
	
	In this section, we evaluate the proposed AE-SAM and AE-LookSAM on several
	standard benchmarks.  As the SAM update doubles the computational overhead
	compared to the ERM update, the training speed is mainly determined by how often
	the SAM update is used.  Hence, we evaluate efficiency by measuring the fraction of SAM
	updates used:
	$\fracsam \equiv 
	100 \times \text{\#\{iterations using SAM\}}/T$.
	The total number of iterations, 
	$T$, 
	is the same for all methods.

	\subsection{\textit{CIFAR-10} and \textit{CIFAR-100}}
	\label{sec:setup-cifar}
	
	\textbf{Setup.}
	In this section,	experiments are performed on the \textit{CIFAR-10} and \textit{CIFAR-100} datasets~\citep{krizhevsky2009learning}
	using four 
	network
	architectures:
	\textit{ResNet-18}~\citep{he2016deep}, 
	\textit{WideResNet-28-10} (denoted \textit{WRN-28-10})~\citep{zagoruyko2016wide},
	\textit{PyramidNet-110}~\citep{han2017deep},
	and \textit{ViT-S16}~\citep{dosovitskiy2021an}.
	
	Following the setup in~\citep{liu2022towards, foret2021sharpness, zhao2022penalizing}, we use batch
	size
	$128$,
	initial learning rate of $0.1$, cosine learning rate schedule, SGD optimizer with momentum $0.9$ and weight decay $0.0001$.  The number of training epochs is $300$
	for \textit{PyramidNet-110},
	$1200$ for \textit{ViT-S16},
	and $200$
	for 
	\textit{ResNet-18} and \textit{WideResNet-28-10}.
	$10\%$ of the training set is used as the validation set.
	As in \cite{foret2021sharpness},
	we perform grid search
	for the radius $\rho$
	over $\{0.01, 0.02, 0.05, 0.1, 0.2, 0.5\}$ using the validation
	set.
	Similarly, $\alpha$ is selected
	by grid search over $\{0.1, 0.3, 0.6, 0.9\}$.
	For the $c_t$ schedule $g_{\lambda_1, \lambda_2}(t)$,
	$\lambda_1=-1$ and $\lambda_2=1$ for AE-SAM; $\lambda_1=0$
	and $\lambda_2=2$
	for AE-LookSAM.
	
	\begin{table}[!t]
		\centering
		\caption{Means and standard deviations of testing accuracy and fraction of SAM
			updates (\%SAM) on \textit{CIFAR-10} and \textit{CIFAR-100}. Methods are grouped based on 
			\%SAM.
			The highest accuracy 
			in each group 
			is underlined; while
			the highest accuracy  for each network architecture (across all 
			groups)
			is in bold.
		}
		\begin{tabular}{c c c c @{\hskip 0.4in} c c}
			\toprule
			&  & \multicolumn{2}{c}{\textit{CIFAR-10} } & \multicolumn{2}{c}{\textit{CIFAR-100} }  \\
			& & Accuracy & \fracsam & Accuracy & \fracsam \\
			\arrayrulecolor{black!50}\specialrule{1.5pt}{.3\jot}{0.3pc}
			\multirow{7}{*}{\STAB{\rotatebox[origin=c]{90}{\textit{ResNet-18}}}}
			& ERM & $95.41$ \mbox{\scriptsize $\pm 0.03$} & $0.0$ \mbox{\scriptsize $\pm 0.0$} 
			& $78.17$ \mbox{\scriptsize $\pm 0.05$} &  $0.0$ \mbox{\scriptsize $\pm 0.0$} 
			\\ \cmidrule{2-6} 
			& SAM~\citep{foret2021sharpness} &  $96.52$ \mbox{\scriptsize $\pm 0.12$}  & $100.0$ \mbox{\scriptsize $\pm 0.0$} 
			& $80.17$ \mbox{\scriptsize $\pm 0.15$}  & $100.0$ \mbox{\scriptsize $\pm 0.0$} 
			\\        & ESAM~\citep{du2021efficient} & $96.56$ \mbox{\scriptsize $\pm 0.08$}  & $100.0$ \mbox{\scriptsize $\pm 0.0$} & $80.41$ \mbox{\scriptsize $\pm 0.10$}  & $100.0$ \mbox{\scriptsize $\pm 0.0$} 
			\\
			\cmidrule{2-6}
			& SS-SAM~\citep{zhao2022ss}  & $96.40$ \mbox{\scriptsize $\pm 0.16$}& $50.0$ \mbox{\scriptsize $\pm 0.0$} & $80.10$ \mbox{\scriptsize $\pm 0.16$}& $50.0$ \mbox{\scriptsize $\pm 0.0$}  
			\\
			& AE-SAM   & $\underline{\mathbf{96.63}}$ \mbox{\scriptsize $\pm 0.04$}& $50.1$ \mbox{\scriptsize $\pm 0.1$} & $\underline{\mathbf{80.48}}$ \mbox{\scriptsize $\pm 0.11$} & $49.8$ \mbox{\scriptsize $\pm 0.0$}
			\\
			\cmidrule{2-6}
			& LookSAM~\citep{liu2022towards} & $96.32$ \mbox{\scriptsize $\pm 0.12$}  & $20.0$ \mbox{\scriptsize $\pm 0.0$} & $79.89$ \mbox{\scriptsize $\pm 0.29$} & $20.0$ \mbox{\scriptsize $\pm 0.0$} 
			\\
			& AE-LookSAM  & $\underline{96.56}$ \mbox{\scriptsize $\pm 0.21$}  & $20.0$ \mbox{\scriptsize $\pm 0.1$} & $\underline{80.29}$ \mbox{\scriptsize $\pm 0.37$} & $20.0$ \mbox{\scriptsize $\pm 0.0$} 
			\\
			\arrayrulecolor{black!50}\specialrule{1.5pt}{.3\jot}{0.3pc}
			\multirow{7}{*}{\STAB{\rotatebox[origin=c]{90}{\textit{WRN-28-10}}}}
			& ERM & $96.34$ \mbox{\scriptsize $\pm 0.12$} & $0.0$ \mbox{\scriptsize $\pm 0.0$} & $81.56$ \mbox{\scriptsize $\pm 0.14$} & $0.0$ \mbox{\scriptsize $\pm 0.0$} 
			\\ \cmidrule{2-6}
			& SAM~\citep{foret2021sharpness} & $97.27$ \mbox{\scriptsize $\pm 0.11$} & $100.0$ \mbox{\scriptsize $\pm 0.0$} & $83.42$ \mbox{\scriptsize $\pm 0.05$}  & $100.0$ \mbox{\scriptsize $\pm 0.0$} 
			\\
			& ESAM~\citep{du2021efficient} & $97.29$ \mbox{\scriptsize $\pm 0.11$}  & $100.0$ \mbox{\scriptsize $\pm 0.0$} & $\mathbf{84.51}$ \mbox{\scriptsize $\pm 0.02$} & $100.0$ \mbox{\scriptsize $\pm 0.0$} 
			\\ 
			\cmidrule{2-6}
			& SS-SAM~\citep{zhao2022ss}  & $97.09$  \mbox{\scriptsize $\pm 0.11$}& $50.0$ \mbox{\scriptsize $\pm 0.0$} & $82.89$  \mbox{\scriptsize $\pm 0.02$}&$50.0$ \mbox{\scriptsize $\pm 0.0$} 
			\\
			& AE-SAM & $\underline{\mathbf{97.30}}$ \mbox{\scriptsize $\pm 0.10$} & $49.5$ \mbox{\scriptsize $\pm 0.1$}  & $\underline{\mathbf{84.51}}$ \mbox{\scriptsize $\pm 0.11$}  & $49.6$ \mbox{\scriptsize $\pm 0.0$} 
			\\
			\cmidrule{2-6}
			& LookSAM~\citep{liu2022towards}  & $97.02$  \mbox{\scriptsize $\pm 0.12$}& $20.0$ \mbox{\scriptsize $\pm 0.0$} & 83.70  \mbox{\scriptsize $\pm 0.12$}& $20.0$ \mbox{\scriptsize $\pm 0.0$}  
			\\
			&AE-LookSAM & $\underline{97.15}$ \mbox{\scriptsize $\pm 0.08$} & $20.0$ \mbox{\scriptsize $\pm 0.0$} &$\underline{83.92}$ \mbox{\scriptsize $\pm 0.07$} & $20.2$ \mbox{\scriptsize $\pm 0.0$} 
			\\	
			\arrayrulecolor{black!50}\specialrule{1.5pt}{.3\jot}{0.3pc}
			\multirow{7}{*}{\STAB{\rotatebox[origin=c]{90}{\textit{PyramidNet-110}}}}
			& ERM & $96.62$ \mbox{\scriptsize $\pm 0.10$} & $0.0$ \mbox{\scriptsize $\pm 0.0$} & 	$81.89$ \mbox{\scriptsize $\pm 0.15$}  & $0.0$ \mbox{\scriptsize $\pm 0.0$} 
			\\ \cmidrule{2-6}
			& SAM~\citep{foret2021sharpness} & $97.30$ \mbox{\scriptsize $\pm 0.10$}  & $100.0$ \mbox{\scriptsize $\pm 0.0$} & $84.46$ \mbox{\scriptsize $\pm 0.05$} & $100.0$ \mbox{\scriptsize $\pm 0.0$}
			\\
			& ESAM~\citep{du2021efficient} & $97.81$ \mbox{\scriptsize $\pm 0.01$} & $100.0$ \mbox{\scriptsize $\pm 0.0$} & $85.56$ \mbox{\scriptsize $\pm 0.05$} & $100.0$ \mbox{\scriptsize $\pm 0.0$}   
			\\
			\cmidrule{2-6}
			& SS-SAM~\citep{zhao2022ss}   & $97.22$ \mbox{\scriptsize $\pm 0.10$}& $50.0$ \mbox{\scriptsize $\pm 0.0$} & $84.90$ \mbox{\scriptsize $\pm 0.05$}& $50.0$ \mbox{\scriptsize $\pm 0.0$} 
			\\
			& AE-SAM   &$\underline{\mathbf{97.90}}$ \mbox{\scriptsize $\pm 0.05$}  & $50.2$ \mbox{\scriptsize $\pm 0.1$} & $\underline{\mathbf{85.58}}$ \mbox{\scriptsize $\pm 0.10$}  & $49.8$ \mbox{\scriptsize $\pm 0.1$} 
			\\
			\cmidrule{2-6}
			& LookSAM~\citep{liu2022towards} & $97.10$ \mbox{\scriptsize $\pm 0.11$} & $20.0$ \mbox{\scriptsize $\pm 0.0$} & $84.01$ \mbox{\scriptsize $\pm 0.06$} & $20.0$ \mbox{\scriptsize $\pm 0.0$} 
			\\
			&AE-LookSAM & $\underline{97.22}$ \mbox{\scriptsize $\pm 0.11$}& $20.3$ \mbox{\scriptsize $\pm 0.0$} & $\underline{84.80}$ \mbox{\scriptsize $\pm 0.13$} & $20.2$ \mbox{\scriptsize $\pm 0.1$} 
			\\
			\arrayrulecolor{black!50}\specialrule{1.5pt}{.3\jot}{0.3pc}
			\multirow{7}{*}{\STAB{\rotatebox[origin=c]{90}{\textit{ViT-S16}}}} 
			&ERM & $86.69$ \mbox{\scriptsize $\pm 0.11$}  &  $0.0$ \mbox{\scriptsize $\pm 0.0$} & $62.42$ \mbox{\scriptsize $\pm 0.22$} & $0.0$ \mbox{\scriptsize $\pm 0.0$}  \\
			\cmidrule{2-6}
			&SAM~\citep{foret2021sharpness} & $87.37$ \mbox{\scriptsize $\pm 0.09$}  &  $100.0$ \mbox{\scriptsize $\pm 0.0$} & $63.23$ \mbox{\scriptsize $\pm 0.25$} & $100.0$ \mbox{\scriptsize $\pm 0.0$}  \\
			&ESAM~\citep{du2021efficient} & $84.27$ \mbox{\scriptsize $\pm 0.11$}  &  $100.0$ \mbox{\scriptsize $\pm 0.0$} & $62.11$ \mbox{\scriptsize $\pm 0.15$} & $100.0$ \mbox{\scriptsize $\pm 0.0$}  \\
			\cmidrule{2-6}
			&SS-SAM~\citep{zhao2022ss} & $87.38$ \mbox{\scriptsize $\pm 0.14$}  &  $50.0$ \mbox{\scriptsize $\pm 0.0$} & $63.18$ \mbox{\scriptsize $\pm 0.19$} & $50.0$ \mbox{\scriptsize $\pm 0.0$} \\
			&AE-SAM & \underline{$\mathbf{87.77}$} \mbox{\scriptsize $\pm 0.13$}  &  $49.7$ \mbox{\scriptsize $\pm 0.1$} & \underline{$63.68$} \mbox{\scriptsize $\pm 0.23$} & $49.5$ \mbox{\scriptsize $\pm 0.2$}  \\
			\cmidrule{2-6}
			&LookSAM~\citep{liu2022towards} & $87.12$ \mbox{\scriptsize $\pm 0.20$}  &  $20.0$ \mbox{\scriptsize $\pm 0.0$} & $63.52$ \mbox{\scriptsize $\pm 0.19$} & $20.0$ \mbox{\scriptsize $\pm 0.0$}  \\
			&AE-LookSAM & \underline{$87.32$} \mbox{\scriptsize $\pm 0.11$}  &  $20.2$ \mbox{\scriptsize $\pm 0.2$} & \underline{$\mathbf{64.16}$} \mbox{\scriptsize $\pm 0.23$} & $20.3$ \mbox{\scriptsize $\pm 0.2$} \\
			\arrayrulecolor{black!50}\specialrule{1.5pt}{.3\jot}{0.3pc}
		\end{tabular}
		\label{table:result-cifar} 
	\end{table}
	
	\textbf{Baselines.}
	The proposed AE-SAM and AE-LookSAM are compared with the following baselines:
	\begin{enumerate*}[(i), series = tobecont, itemjoin = \quad]
		\item ERM; \item SAM~\citep{foret2021sharpness}; and its more efficient variants including
		\item ESAM~\citep{du2021efficient}
		which uses part of the weights to compute the
		perturbation and part of the samples to compute the SAM update direction.
		These two techniques can reduce the computational cost,
		but may not always accelerate SAM, particularly in parallel
		training~\citep{li2020pytorch};
		\item SS-SAM~\citep{zhao2022ss},
		which randomly selects SAM or ERM according to a Bernoulli trial with success
		probability $0.5$. This is
		the scheme with the best performance in
		\citep{zhao2022ss};
		\item LookSAM~\citep{liu2022towards} which uses SAM at every $k=5$ steps.
	\end{enumerate*}
	The experiment is repeated five times with different random seeds.
	
	\textbf{Results.} Table \ref{table:result-cifar} shows the testing accuracy and
	fraction of SAM updates (\%SAM).
	Methods are grouped based on \fracsam. 
	As can be seen, 
	AE-SAM has higher accuracy than SAM while using only 50\% of SAM updates.
	SS-SAM
	and AE-SAM have comparable \fracsam\   (about $50\%$), and
	AE-SAM achieves higher accuracy than SS-SAM (which is statistically significant
	based on 
	the pairwise t-test 
	at $95\%$ significance level).
	Finally, LookSAM and
	AE-LookSAM 
	have comparable \fracsam\  (about $20\%$), and
	AE-LookSAM also has higher accuracy
	than LookSAM.
	These improvements
	confirm that the adaptive policy is better.

	\subsection{\textit{ImageNet}}
	\textbf{Setup.}
	In this section, we perform 
	experiments on 
	the 
	\textit{ImageNet}~\citep{russakovsky2015imagenet}, 
	which contains $1000$ classes and $1.28$ million
	images.
	The \textit{ResNet-50}~\citep{he2016deep}
	is used.
	Following the setup 
	in~\cite{du2021efficient},
	we train the network
	for $90$ epochs 
	using a SGD optimizer with momentum $0.9$, weight decay $0.0001$,
	initial learning rate $0.1$, cosine learning rate schedule,
	and batch size $512$.
	As in~\citep{foret2021sharpness,du2021efficient},
	$\rho=0.05$.
	For the $c_t$ schedule $g_{\lambda_1, \lambda_2}(t)$,
	$\lambda_1=-1$ and $\lambda_2=1$ for AE-SAM; $\lambda_1=0$
	and $\lambda_2=2$
	for AE-LookSAM.
	$k=5$ is used for LookSAM.
	Experiments are repeated with three different
	random seeds.
	
	\textbf{Results.}
	Table \ref{table:summary-imagnet} shows the testing accuracy and fraction of SAM
	updates.  As can be seen, with only half of the iterations using SAM,
	AE-SAM achieves comparable performance
	as SAM.
	Compared with LookSAM,
	AE-LookSAM has
	better
	performance (which is also statistically significant),
	verifying
	the proposed adaptive policy is 
	more effective than
	LookSAM's periodic policy.
	
	\begin{table}[!h]
		\centering
		\vskip -.05in
		\caption{Means and standard deviations of testing accuracy and fraction of SAM
			updates (\%SAM) on \textit{ImageNet}
			using \textit{ResNet-50}.
			Methods are grouped based on 
			\%SAM.
			The highest accuracy 
			in each group 
			is underlined; while
			the highest across all groups
			is in bold.}
		\begin{tabular}{c c c}
			\toprule
			& 
			Accuracy & \fracsam \\
			\midrule
			ERM & $77.11$ \mbox{\scriptsize $\pm 0.14$}  & $0.0$ \mbox{\scriptsize $\pm 0.0$}
			\\
			\midrule
			SAM~\citep{foret2021sharpness} &  $\mathbf{77.47}$ 
			\mbox{\scriptsize $\pm 0.12$}  & $100.0$ \mbox{\scriptsize $\pm 0.0$}  
			\\
			ESAM~\citep{du2021efficient} & $77.25$  \mbox{\scriptsize $\pm 0.75$}   & $100.0$ \mbox{\scriptsize $\pm 0.0$}
			\\
			\midrule
			SS-SAM~\citep{zhao2022ss}  & $77.38$ \mbox{\scriptsize $\pm 0.06$}&   $50.0$ \mbox{\scriptsize $\pm 0.0$} 
			\\
			AE-SAM & \underline{$77.43$} \mbox{\scriptsize $\pm 0.06$} & $49.4$  \mbox{\scriptsize $\pm 0.0$}
			\\
			\midrule
			LookSAM~\citep{liu2022towards} & $77.13$ \mbox{\scriptsize $\pm 0.09 $} & $20.0$ 
			\mbox{\scriptsize $\pm 0.0$}
			\\
			AE-LookSAM & $\underline{77.29}$ \mbox{\scriptsize $\pm 0.08$}   & $20.3$ \mbox{\scriptsize $\pm 0.0$}
			\\
			\bottomrule
		\end{tabular}
		\label{table:summary-imagnet} 
	\end{table}
	
	\subsection{Robustness to Label Noise}
	\label{sec:noisy-level}
	
	\textbf{Setup.}
	In this section, we study 
	whether the more-efficient SAM variants will affect its robustness to training
	label noise.
	Following the setup in~\cite{foret2021sharpness}, we conduct
	experiments on a corrupted version of \textit{CIFAR-10}, with some of its training
	labels randomly flipped  (while its testing set is
	kept clean).
	The \textit{ResNet-18} and \textit{ResNet-32}
	networks
	are used. 
	They  are trained for $200$ epochs using SGD with momentum $0.9$, weight decay
	$0.0001$, batch size $128$, initial learning rate $0.1$, and cosine learning rate
	schedule.  
	For LookSAM, the SAM update is used every $k=2$ steps.\footnote{The performance
		of
		LookSAM can be sensitive to the value of $k$. 
		Table \ref{table:looksam-noisy-label} of
		Appendix~\ref{sec:appendix-label-noise} shows that
		using $k=2$
		leads to the best performance in this experiment.}
	For AE-SAM and AE-LookSAM, 
	we set
	$\lambda_1=-1$ and $\lambda_2=1$
	in their $c_t$ schedules $g_{\lambda_1, \lambda_2}(t)$, 
	such that their fractions of SAM updates 
	(approximately $50\%$)
	are comparable with
	SS-SAM and LookSAM.
	Experiments are repeated with five different
	random seeds.
	
	\textbf{Results.}
	Table \ref{table:noisy-labels}
	shows the testing accuracy
	and fraction of SAM updates.
	As can be seen,
	AE-LookSAM achieves comparable performance with SAM but is faster as
	only half of the iterations use the SAM update.
	Compared with 
	ESAM, SS-SAM, and LookSAM,
	AE-LookSAM performs better.
	The improvement is particularly noticeable
	at the higher noise levels (e.g., $80\%$).   
	
	Figure \ref{fig:curve-noisy-resnet18} shows the training and testing accuracies
	with number of epochs
	at a noise level of $80\%$
	using \textit{ResNet-18}\footnote{Results for other 
		noise levels and
		\textit{ResNet-32} are shown in Figures
		\ref{fig:curve-noisy-resnet18-apd} and \ref{fig:curve-noisy-resnet32-apd} of
		Appendix \ref{sec:robust}, respectively.}. As can be seen, SAM is robust to the
	label noise, while
	ERM and SS-SAM heavily suffer from overfitting.
	AE-SAM and LookSAM 
	can alleviate the overfitting problem to a certain extent.
	AE-LookSAM,
	by combining the adaptive policy with LookSAM,
	achieves the same high level of robustness as SAM.

	\begin{table}[!t]
		\centering
		\vskip -.05in		
		\caption{Testing accuracy and fraction of SAM updates on \textit{CIFAR-10}
			with different levels of label noise.
			The best accuracy is in bold
			and the second best is underlined.
		}
		\resizebox{.98\textwidth}{!}{
			\begin{tabular}{c @{\hskip .07in} c @{\hskip .05in} c @{\hskip .05in} c @{\hskip .2in}  c @{\hskip .05in}c @{\hskip .2in} c @{\hskip .05in}c @{\hskip .2in} c@{\hskip .05in} c}
				\toprule
				\multicolumn{2}{c}{ }& \multicolumn{2}{c}{$\text{noise }=20\%$}
				& \multicolumn{2}{c}{$\text{noise }=40\%$}
				& \multicolumn{2}{c}{$\text{noise }=60\%$}
				& \multicolumn{2}{c}{$\text{noise }=80\%$} \\
				\multicolumn{2}{c}{ } &accuracy & \fracsam
				&accuracy & \fracsam
				&accuracy & \fracsam
				&accuracy & \fracsam	\\
				\midrule
				\multirow{7}{*}{\STAB{\rotatebox[origin=c]{90}{\textit{ResNet-18}}}}
				& ERM & $87.92$ & $0.0$  & $70.82$ 
				&$0.0$  & $49.61$  
				&$0.0$ & $28.23$   & $0.0$  
				\\
				& SAM~\citep{foret2021sharpness} &   
				$\mathbf{94.80}$    & $100.0$ &	\underline{$91.50$} 
				&$100.0$ & $\mathbf{88.15}$  
				&$100.0$  & $\mathbf{77.40}$  & $100.0$ 
				\\
				&ESAM~\citep{du2021efficient} & $94.19$ &$100.0$  & $91.46$   
				&$100.0$ & $81.30$  
				&$100.0$  & $15.00$  & $100.0$ 
				\\ 
				& SS-SAM~\citep{zhao2022ss}  
				& 	$90.62$   &  $50.0$ 
				& $77.84$   & $50.0$  
				& $61.18$  	& $50.0$  
				& $47.32$ & $50.0$  
				\\
				&LookSAM~\citep{liu2022towards} 
				& $92.72$ 	& $50.0$  
				& $88.04$   	&$50.0$   
				& $72.26$   &$50.0$  
				& $69.72$  & $50.0$ 
				\\
				& AE-SAM 
				& 	$92.84$  & $50.0$ 
				&	$84.17$   &$50.0$   
				& $73.54$   	&$49.9$ 
				& $65.00$   	& $50.0$   
				\\
				&AE-LookSAM 
				&  \underline{$94.34$}    & $49.9$   
				&  $\mathbf{91.58}$   &$50.0$  
				& \underline{$87.85$}   &$50.0$ 
				& \underline{$76.90$}   & $50.0$  
				\\
				\midrule 
				\multirow{7}{*}{\STAB{\rotatebox[origin=c]{90}{\textit{ResNet-32}}}}
				& ERM & 	$87.43$   & $0.0$   & $70.82$     &$0.0$ & $46.26$  & $0.0$  & $29.00$  
				& $0.0$  
				\\
				&SAM~\citep{foret2021sharpness} &   
				$\mathbf{95.08}$  & $100.0$  &	$91.01$  &$100.0$  & $\mathbf{88.90}$  & $100.0$  & $\mathbf{77.32}$   & $100.0$ 
				\\
				& ESAM~\citep{du2021efficient} & $93.42$    & $100.0$   & $\underline{91.63}$   & $100.0$   & 	$82.73$    & $100.0$ 
				& $10.09$ & $100.0$ 
				\\ 
				&SS-SAM~\citep{zhao2022ss}  & $89.63$   & $50.0$  & $74.17$  & $50.0$ & $58.40$  
				&$50.0$  & $59.53$   & $50.0$
				\\
				&LookSAM~\citep{liu2022towards} & $92.49$   & $50.0$   & $86.56$ 
				& $50.0$    & $63.35$ 
				& $50.0$   & $68.01$  &
				$50.0$  
				\\
				& AE-SAM & $92.87$     & $50.0$     & 	$82.85$ & $50.0$   & $71.50$   & $50.0$  & $65.43$   & $50.3$  
				\\
				&AE-LookSAM &  \underline{$94.70$}   & $50.0$   & $\mathbf{91.80}$  
				& $50.0$   & \underline{$88.22$}  
				& $50.0$   & \underline{$77.03$}    & $49.8$   
				\\
				\bottomrule
			\end{tabular}
		}
		\label{table:noisy-labels} 
	\end{table}
	
	\begin{figure}[!t]
		\centering
		\subfigure[Training accuracy. \label{fig:noisy-cifar10-80-resnet18-train}]{\includegraphics[width=0.32\textwidth]{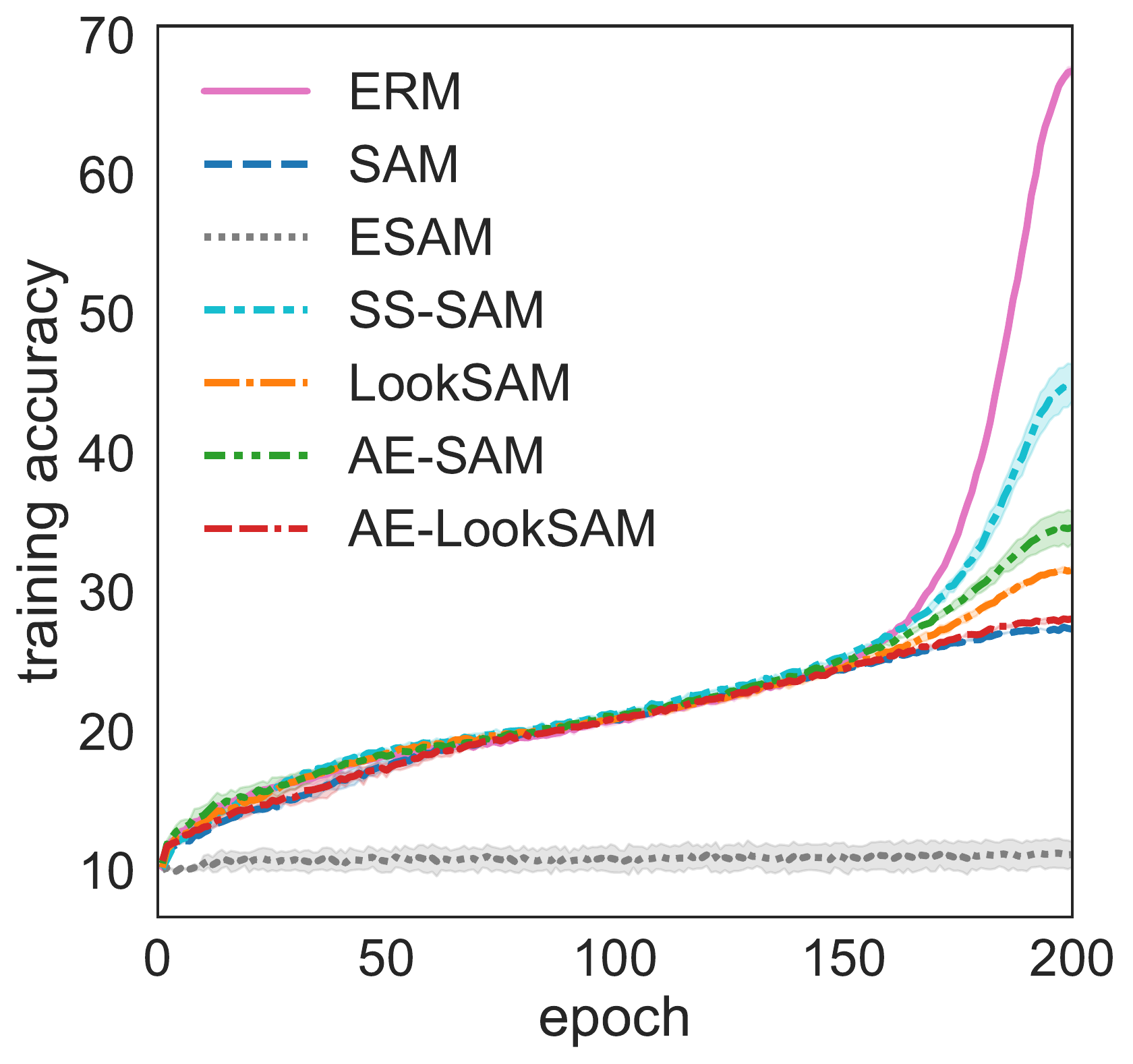}}
		\quad \quad 
		\subfigure[Testing accuracy. \label{fig:noisy-cifar10-80-resnet18-valid}]{\includegraphics[width=0.32\textwidth]{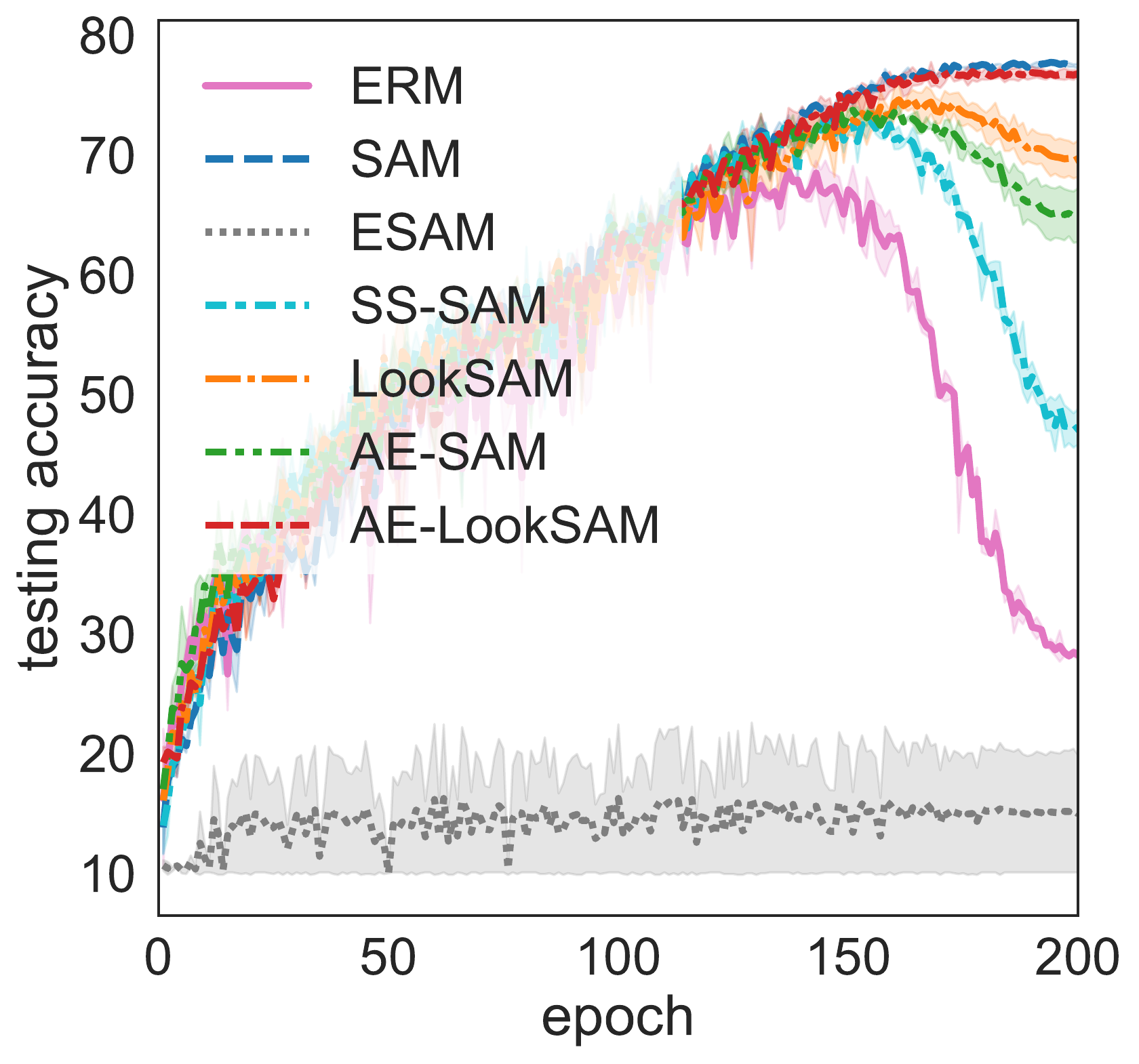}}
		\vskip -.2in
		\caption{Accuracies with number of training epochs on \textit{CIFAR-10} (with $80\%$ noise labels) using \textit{ResNet-18}.
			Best viewed in color.
		}
		\label{fig:curve-noisy-resnet18}
	\end{figure}
	
	\subsection{Effects of $\lambda_1$ and $\lambda_2$}
	\label{sec:ablation study ae-sam}
	
	In this experiment, we study the effects of $\lambda_1$ and $\lambda_2$ on AE-SAM.
	We use the same setup as in Section \ref{sec:setup-cifar},
	where
	$\lambda_1$ 
	and $\lambda_2$ (with $\lambda_1\leq \lambda_2$)
	are chosen from $\{0, \pm 1, \pm 2\}$.
	Results on AE-LookSAM using the label noise setup in Section
	\ref{sec:noisy-level} are shown in Appendix \ref{sec:lambda}.
	
	Figure \ref{fig:ablation-samfrac} shows the effect
	on the fraction of SAM updates.  
	For a fixed $\lambda_2$, increasing $\lambda_1$ 
	increases  the threshold
	$c_t$,
	and the condition $\| \nabla \hL(\hB_t; \vw_t) \|^2\geq \mu_t + c_t \sigma_t$
	becomes more difficult to satisfy.
	Thus,  as can be seen,
	the fraction of SAM updates is 
	reduced. The same applies when 
	$\lambda_2$ increases.
	A similar trend is also observed 
	on the testing accuracy
	(Figure \ref{fig:ablation-acc}).
	
	\begin{figure}[!t]
		\centering
		\begin{minipage}{.46\textwidth}
			\vskip .14in
			\centering
			\!\!\!
			\subfigure[\textit{CIFAR-10}.\label{fig:ablation-cifar10-samfrac}]{\includegraphics[width=0.50\textwidth]{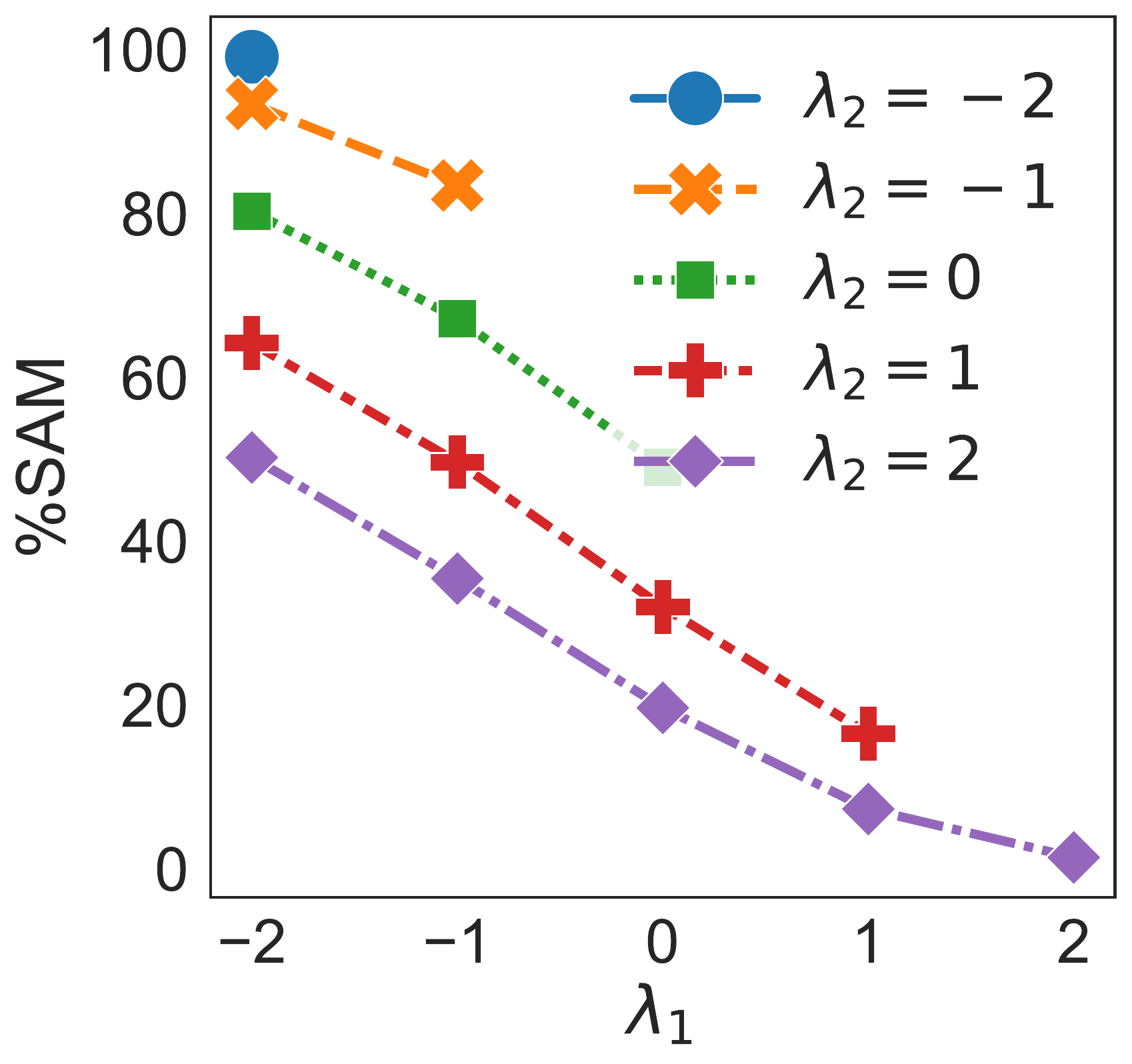}}	
			\subfigure[\textit{CIFAR-100}. \label{fig:ablation-cifar100-samfrac}]{\includegraphics[width=0.50\textwidth]{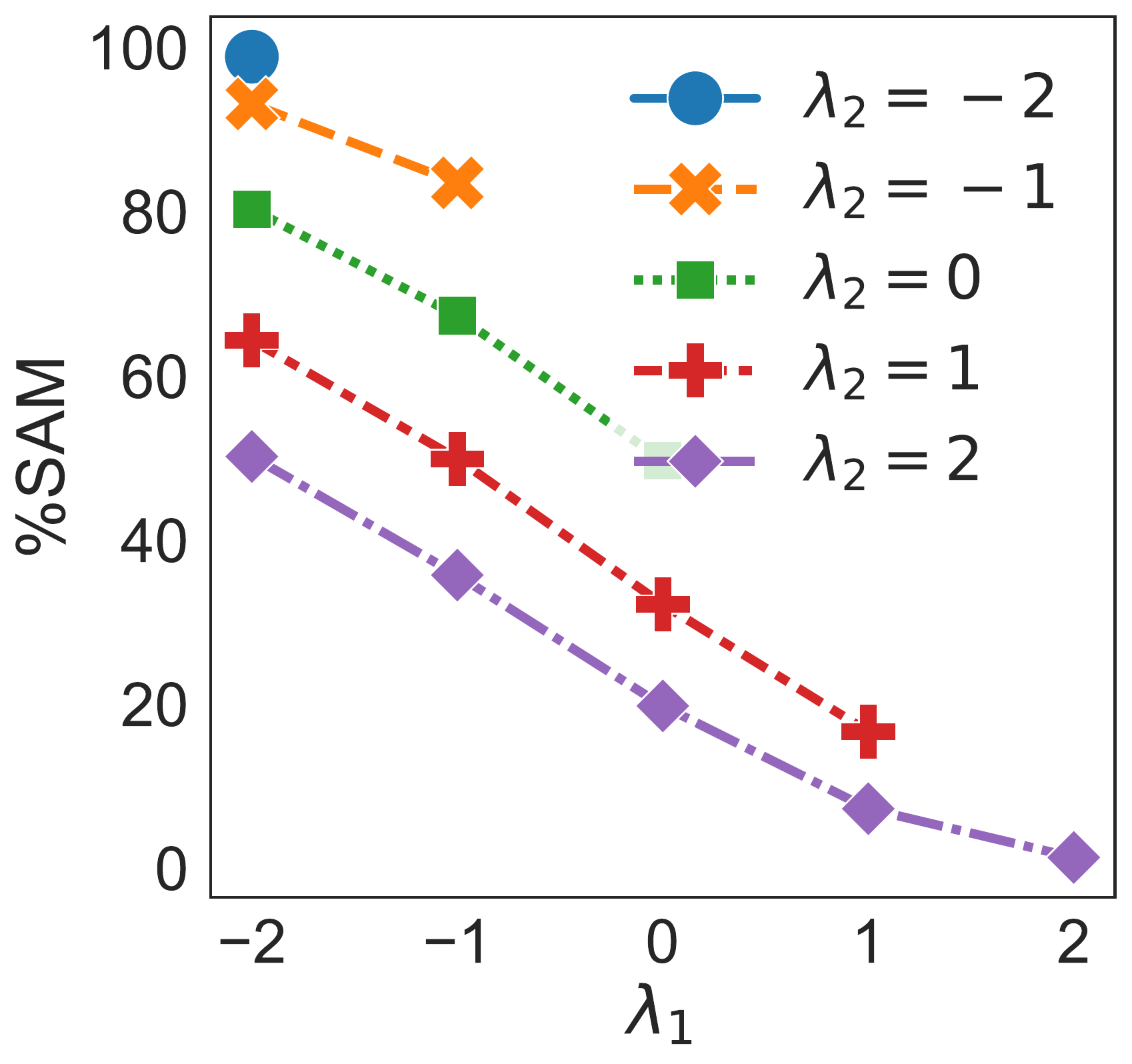}}
			\!\!\!
			\vskip -.1in
			\caption{
				Effects of $\lambda_1$ and $\lambda_2$ on fraction of SAM updates
				using \textit{ResNet-18}. 
				Best viewed in color.
			}
			\label{fig:ablation-samfrac}
		\end{minipage} \hfill
		\begin{minipage}{.49\textwidth}
			\centering
			\!\!\!
			\subfigure[\textit{CIFAR-10}.\label{fig:ablation-cifar10-acc}]{\includegraphics[width=0.49\textwidth]{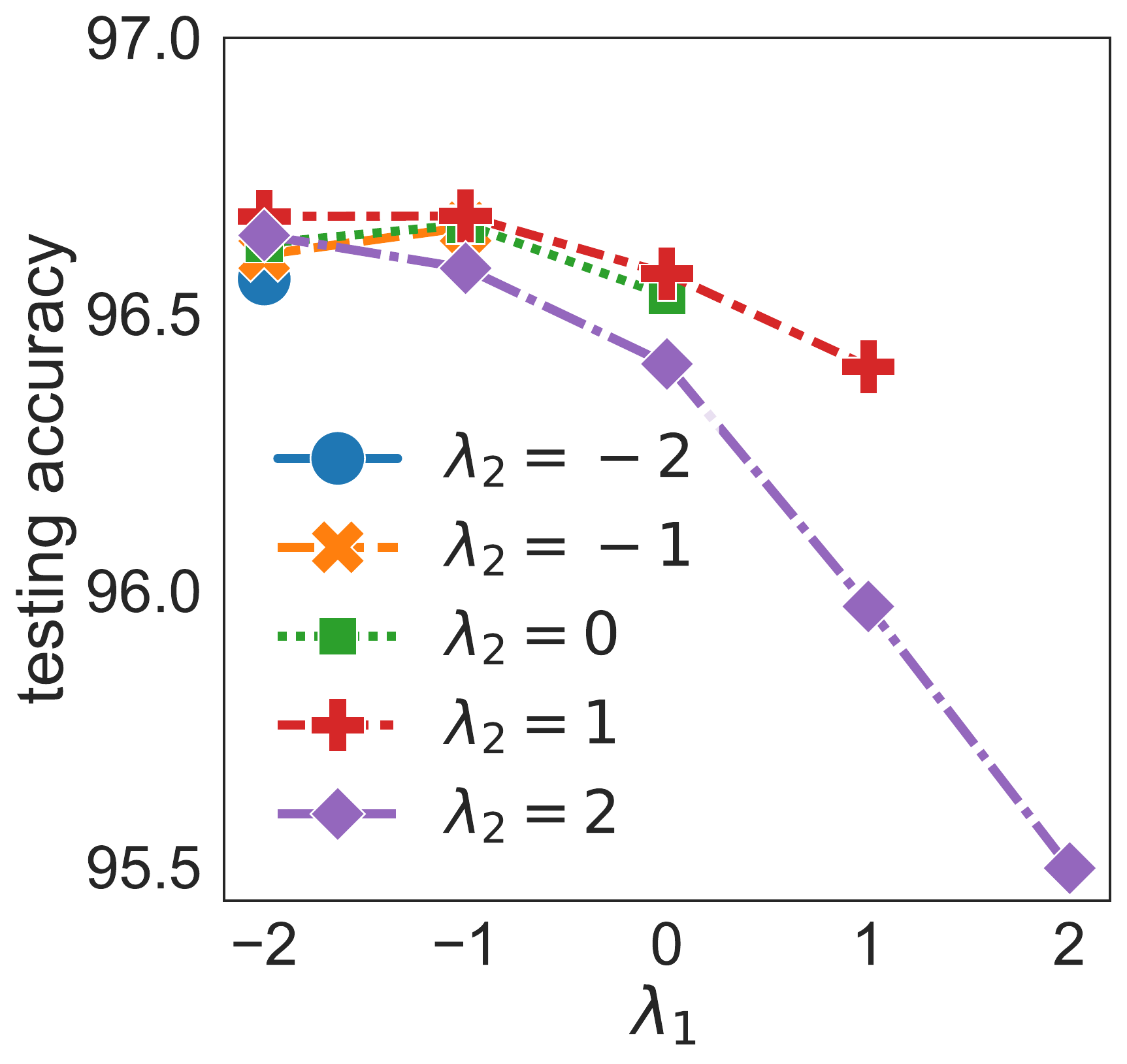}}	
			\!\!\!
			\subfigure[\textit{CIFAR-100}. \label{fig:ablation-cifar100-acc}]{\includegraphics[width=0.475\textwidth]{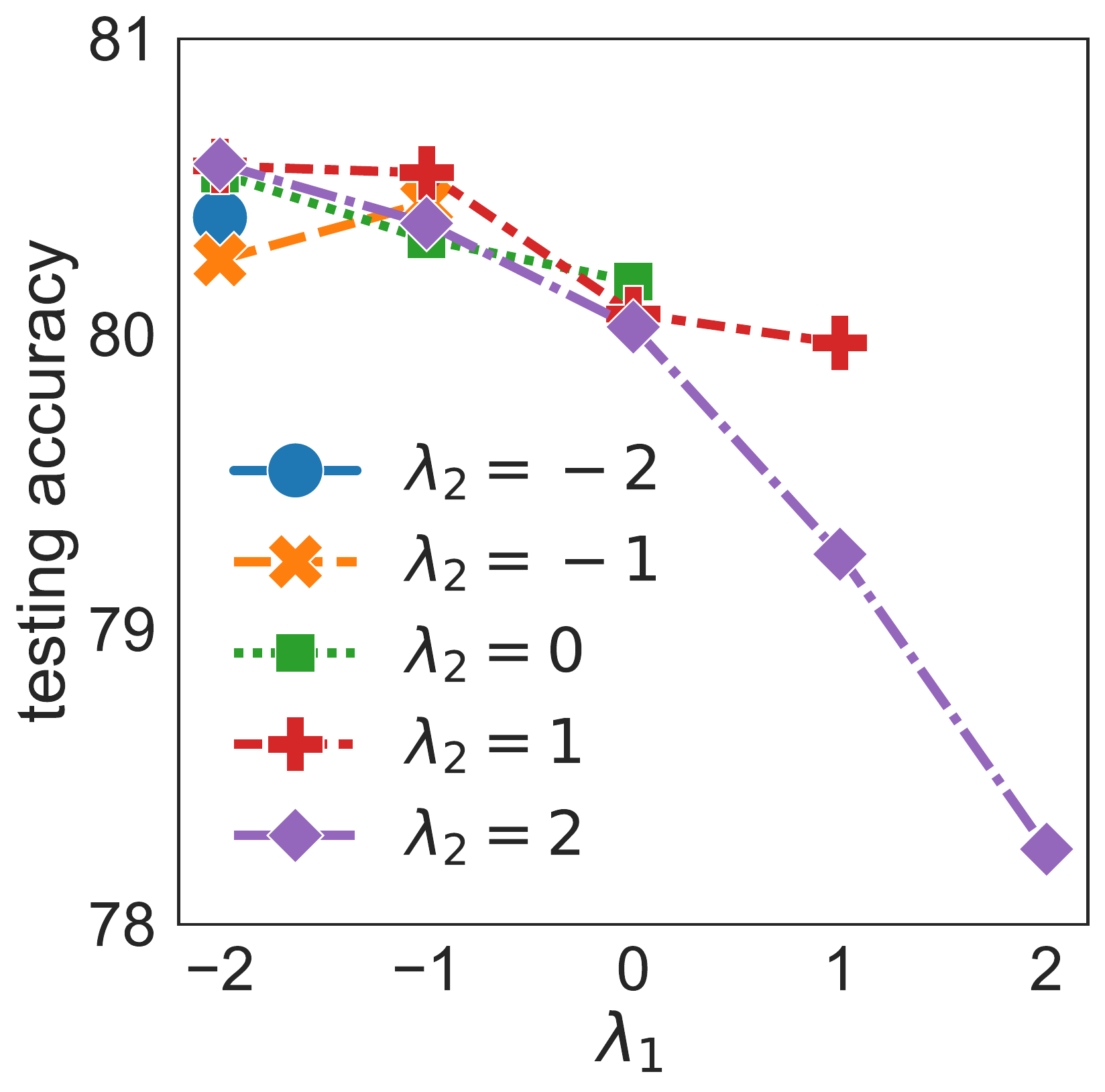}}
			\!\!\!
			\vskip -.1in
			\caption{
				Effects of $\lambda_1$ and $\lambda_2$ on testing accuracy
				using \textit{ResNet-18}. 
				Best viewed in color.
			}
			\label{fig:ablation-acc}
		\end{minipage}
	\end{figure}

	\subsection{Convergence}
	In this experiment,
	we study 
	whether $\vw_t$'s (where
	$t$ is the number of epochs)
	obtained from  AE-SAM 
	can reach critical points of $\hL(\hD; \vw)$, as suggested in Theorem
	\ref{thm:conv-sgd}.
	Figure  \ref{fig:gradient-norm-trend}
	shows $\|\nabla \hL(\hD; \vw_t)\|^2$ w.r.t. $t$
	for the experiment in Section \ref{sec:setup-cifar}.
	As can be seen,
	in all settings,
	$\|\nabla \hL(\hD; \vw_t) \|^2$
	converges to $0$.
	In Appendix \ref{sec:appendix-expt-result}, we
	also verify the convergence of AE-SAM's training loss on \textit{CIFAR-10} and \textit{CIFAR-100}
	(Figure \ref{fig:loss-trend}), and that
	AE-SAM and SS-SAM  have comparable convergence speeds
	(Figure \ref{fig:conv-compar-loss-trend}),
	which agrees with Theorem \ref{thm:conv-sgd}
	as both have comparable fractions of SAM updates (Table \ref{table:result-cifar}).
	
	\begin{figure}[!h]
		\centering
		\vskip -.15in
		\!\!
		\subfigure[\textit{ResNet-18}. \label{fig:grad-norm-trend-resnet}]{\includegraphics[width=0.32\textwidth]{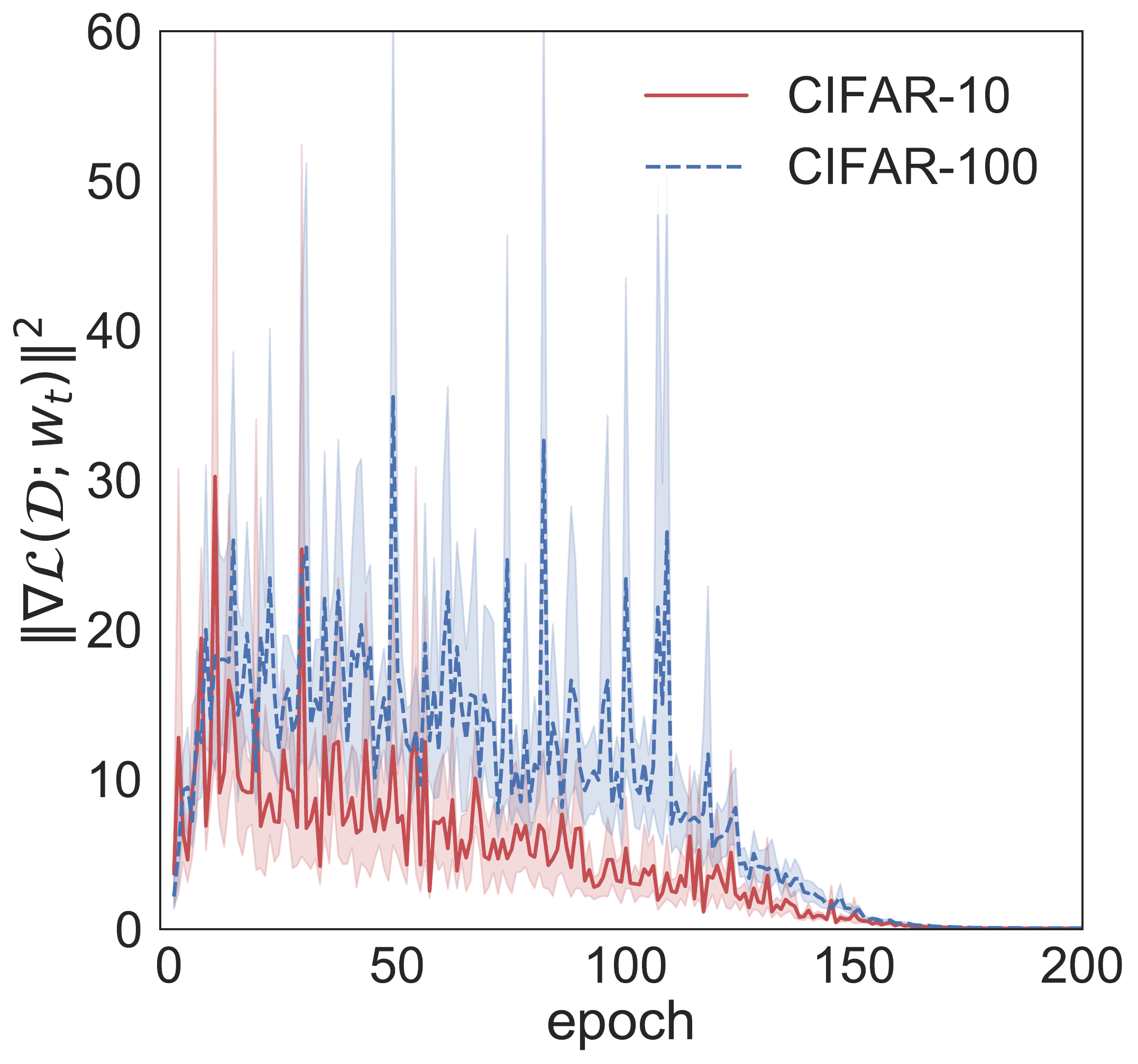}}
		\!\!\!
		\subfigure[\textit{WRN-28-10}. \label{fig:grad-norm-trend-wrn}]{\includegraphics[width=0.33\textwidth]{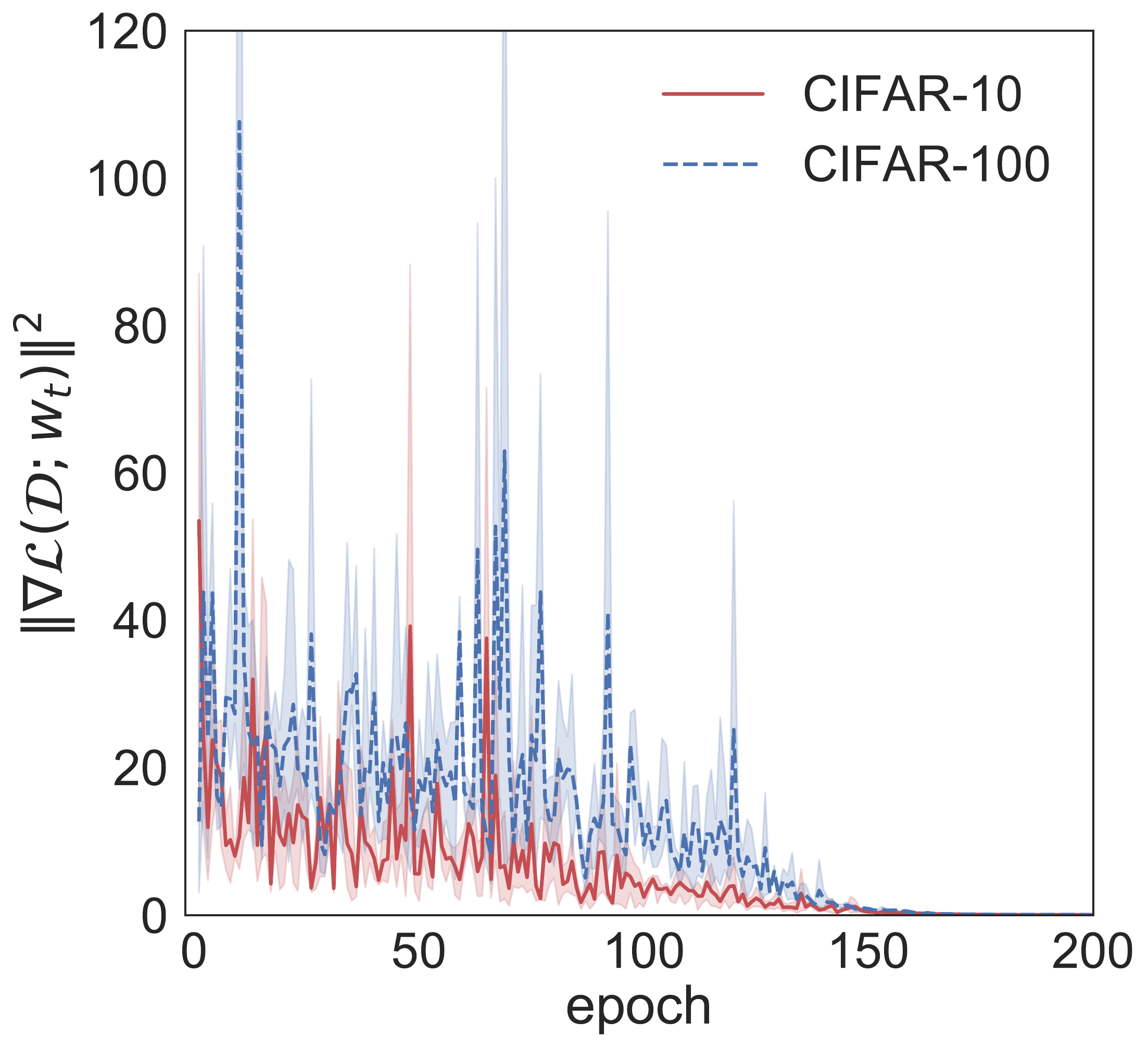}}
		\!\!\!
		\subfigure[\textit{PyramidNet-110}. \label{fig:grad-norm-trend-pyramidnet}]{\includegraphics[width=0.32\textwidth]{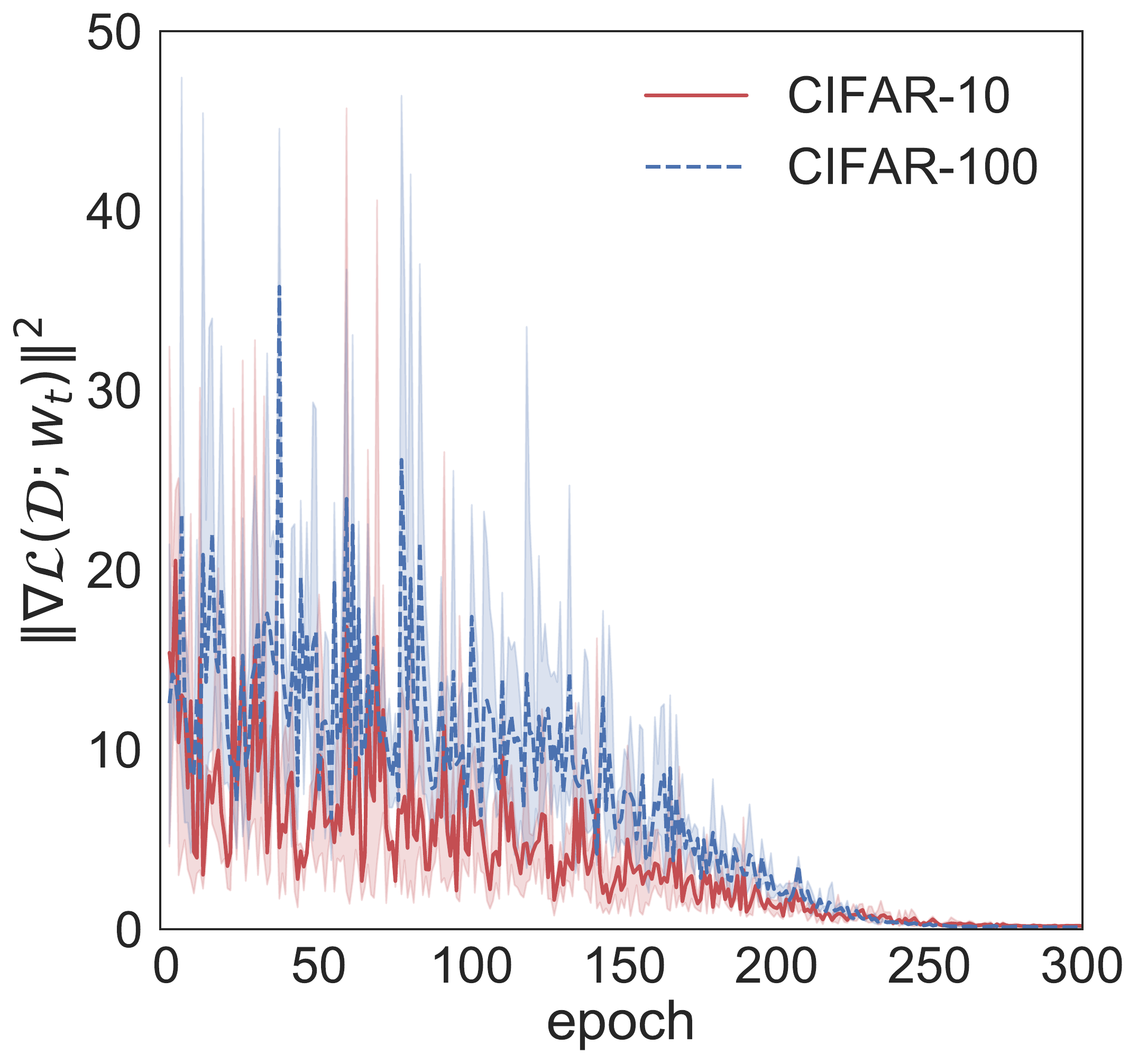}}
		\!\!\!
		\vskip -.1in
		\caption{Squared gradient norms 
			of AE-SAM
			with number of epochs.
			Best viewed in color.
		}
		\label{fig:gradient-norm-trend}
	\end{figure}

	\section{Conclusion}
	
	In this paper, 
	we proposed an adaptive policy 
	to employ SAM 
	based on the loss landscape geometry. 
	Using the policy, we proposed an efficient algorithm (called AE-SAM) to 
	reduce the fraction of SAM updates during training. 
	We theoretically and empirically analyzed the convergence of AE-SAM.
	Experimental results on a number of datasets and network architectures verify the
	efficiency and
	effectiveness of the adaptive policy. Moreover,
	the proposed policy is general 
	and can be combined 
	with other SAM variants, 
	as demonstrated by the success of AE-LookSAM.
	
	\newpage
	\subsubsection*{Acknowledgments}
	
	This work was supported by NSFC key grant 62136005, NSFC general grant 62076118, and Shenzhen fundamental research program JCYJ20210324105000003.
	This research was supported in part by the Research Grants Council of the Hong Kong Special Administrative Region (Grant 16200021).

	\bibliography{paper}
	\bibliographystyle{iclr2023_conference}
	
	
	\appendix
	\section{Proofs}
	\label{sec:proof}

	\subsection{Proof of Theorem \ref{thm:conv-sgd}}
	\label{appendix:proof of thm}
	\textbf{Theorem \ref{thm:conv-sgd}.}
	\textit{
		Let $b$ be the mini-batch size.  If $\eta = \frac{1}{4\beta\sqrt{T}}$ and $\rho = 1/T^{\frac{1}{4}}$, algorithm $\hA$
		satisfies
		\begin{align}
			\min_{0\leq t\leq T-1} \bE  \| \nabla \hL(\hD; \vw_t) \|^2
			&\leq  \frac{32\beta \left( \hL(\hD; \vw_0) - \bE \hL(\hD; \vw_T)\right)}{\sqrt{T} \left(7-6\zeta\right)} 
			+  \frac{ (1 + \zeta +  5 \beta^2 \zeta )\sigma^2}{b\sqrt{T}\left(7-6\zeta\right)}, 
		\end{align}
		where 
		$\zeta =  \frac{1}{T} \sum_{t=0}^{T-1}\xi_t \in [0,1]$.
	}
	
	\begin{lemma}[\cite{andriushchenko2022towards}]
		\label{lem:lem-sgd}
		Under Assumptions \ref{ass:smooth} and \ref{ass:bd-noise}
		for all $t$ and $\rho > 0$,
		we have 
		\begin{align}
			\bE \nabla \hL(\hB_t; \vw + \rho \nabla \hL(\hB_t; \vw))^\top \nabla \hL(\hD; \vw) 
			\geq \left( \frac{1}{2} - \rho \beta  \right)
			\| \nabla \hL(\hD;\vw) \|^2 - \frac{\beta^2 \rho^2 \sigma^2}{2b}.
		\end{align}
	\end{lemma}
	\begin{proof}
		Let $\vg_t \equiv \frac{1}{b} \sum_{(\vx_i, y_i) \in \hB_t} \nabla \ell(f(\vx_i; \vw_t), y_i) $,
		$\vh_t \equiv \frac{1}{b} \sum_{(\vx_i, y_i) \in \hB_t} \nabla \ell(f(\vx_i; \vw_t + \rho \vg_t), y_i) $,
		and $\hat{\vg}_t \equiv \nabla \hL(\hD; \vw_t)$.
		
		By Taylor expansion
		and $\hL(\hD; \vw)$ is $\beta$-smooth, we have  
		\begin{align}
			&\hL(\hD; \vw_{t+1})  \nonumber\\
			\leq &\hL(\hD; \vw_t) + \hat{\vg}_t^\top (\vw_{t+1} - \vw_t ) 
			+ \frac{\beta}{2}  \| \vw_{t+1} - \vw_t \|^2 \nonumber \\
			\leq &\hL(\hD; \vw_t) 
			- \eta \hat{\vg}_t^\top  \left((1-\xi_t)\vg_t   + \xi_t \vh_t\right) + \frac{\beta \eta^2}{2}  \| (1-\xi_t)\vg_t   + \xi_t \vh_t \|^2 \nonumber\\
			=& \hL(\hD; \vw_t) \!
			-\! \eta (1\!-\!\xi_t) \hat{\vg}_t^\top \vg_t
			\!-\! \eta \xi_t \hat{\vg}_t^\top \vh_t
			\!+\! \frac{\beta \eta^2}{2}\! \left(\! (1\!-\!\xi_t) \| \vg_t\|^2 \!+\! \xi_t \| \vh_t\|^2 \!+\! \underbrace{2\xi_t (1-\xi_t) \vg_t^\top \vh_t}_{=0} \!\right)\! \label{eq:temp-asjk7} \\
			=& \hL(\hD; \vw_t) 
			- \eta (1-\xi_t) \hat{\vg}_t^\top \vg_t
			- \eta \xi_t \hat{\vg}_t^\top \vh_t
			+ \frac{\beta \eta^2}{2} \left( (1-\xi_t) \| \vg_t\|^2 + \xi_t \| \vh_t\|^2 \right), \label{eq:temp-xasj1}
		\end{align}
		where we have used $\xi_t (1-\xi_t) = 0$ as $\xi_t\in \{0,1\}$, $\xi_t^2=\xi_t$, and $(1-\xi_t)^2=1-\xi_t$ to obtain \eqref{eq:temp-asjk7}. 
		Taking expectation w.r.t. $\vw_t$ on both sides of \eqref{eq:temp-xasj1}, 
		we have
		\begin{align}
			\!\!\bE \hL(\hD; \vw_{t+1})
			\!\leq  \! \bE \hL(\hD; \vw_t) 
			\!-\! \eta  (1\!-\!\xi_t) \bE \| \hat{\vg}_t \|^2 \!-\! \eta \xi_t \bE \hat{\vg}_t^\top \vh_t
			\!+\! \frac{\beta \eta^2 (1-\xi_t)}{2} \bE \| \vg_t \|^2
			\!+\! \frac{\beta \eta^2 \xi_t}{2} \bE \| \vh_t \|^2.
			\label{eq:temp-dfhk}
		\end{align}
		\vskip -.15in
		\underline{Claim 1: $\bE \| \vg_t \|^2 \!=\! \bE \| \vg_t \!-\! \hat{\vg}_t \|^2 \!+\! \bE \| \hat{\vg}_t \|^2 \!=\! \frac{\sigma^2}{b} \!+\!  \bE\| \hat{\vg}_t \|^2$,}
		which follows from Assumption \ref{ass:bd-noise}.
		
		\underline{Claim 2: $\bE \| \vh_t \|^2 \leq 2(1+ \rho^2 \beta^2 ) \frac{\sigma^2}{b} - (1 - 2\rho^2 \beta^2) \bE \|\hat{\vg}_t\|^2  + 2 \bE  \hat{\vg}_t^\top \vh_t$}, which is derived as follows:
		\begin{align}
			\bE \| \vh_t \|^2 
			&= \bE \| \vh_t - \hat{\vg}_t\|^2 - \bE \| \hat{\vg}_t \|^2 + 2 \bE  \hat{\vg}_t^\top \vh_t \nonumber\\
			&= 2 \bE \| \vh_t - \vg_t \|^2  +2 \bE \| \vg_t - \hat{\vg}_t \|^2 - \bE \| \hat{\vg}_t \|^2 + 2 \bE  \hat{\vg}_t^\top \vh_t \nonumber\\
			&\leq 2 \rho^2 \beta^2 \bE\|\vg_t \|^2 
			+  \frac{2\sigma^2}{b} - \bE \| \hat{\vg}_t \|^2 + 2 \bE  \hat{\vg}_t^\top \vh_t \label{eq:temp-asj8} \\
			&\leq 2 \rho^2 \beta^2 \left(\frac{\sigma^2}{b} +  \bE\| \hat{\vg}_t \|^2\right)
			+  \frac{2\sigma^2}{b} - \bE \| \hat{\vg}_t \|^2 + 2 \bE  \hat{\vg}_t^\top \vh_t \label{eq:temp-asj9} \\
			&=  2(1+ \rho^2 \beta^2 ) \frac{\sigma^2}{b} - (1 - 2\rho^2 \beta^2) \bE \|\hat{\vg}_t\|^2  + 2 \bE  \hat{\vg}_t^\top \vh_t,  \label{eq:temp-asj1} 
		\end{align}
		\vskip -.15in
		where  
		\eqref{eq:temp-asj8}
		follows from $\| \vh_t - \vg_t \| \leq \rho \beta \| \vg_t\|$
		and 
		Assumption \ref{ass:bd-noise},
		\eqref{eq:temp-asj9}
		follows from Claim 1.
		
		Substituting Claims 1 and 2 into \eqref{eq:temp-dfhk},
		we obtain 
		\begin{align}
			&\bE \hL(\hD; \vw_{t+1})  \nonumber  \\
			&\leq  \bE \hL(\hD; \vw_t) 	- \eta  \left(1-\xi_t \right) \bE \| \hat{\vg}_t \|^2 -  \eta \xi_t \bE \hat{\vg}_t^\top \vh_t  +\frac{\beta \eta^2 (1-\xi_t)}{2} \left(\frac{\sigma^2}{b} + \bE \| \hat{\vg}_t \|^2\right) \nonumber\\
			&\quad \quad +  \frac{\beta \eta^2 \xi_t}{2} \left(2(1+ \rho^2 \beta^2 ) \frac{\sigma^2}{b} - (1 - 2\rho^2 \beta^2) \bE \|\hat{\vg}_t\|^2  + 2 \bE  \hat{\vg}_t^\top \vh_t\right) \label{eq:tmep-ashj1}\\
			& = \bE \hL(\hD; \vw_t) 
			- \eta  \left(1-\xi_t -\frac{\beta \eta (1-\xi_t)}{2}
			+  \frac{\beta \eta  \xi_t (1-2\rho^2\beta^2) }{2} 
			\right) \bE \| \hat{\vg}_t \|^2  - \eta \xi_t \left(1-\eta \beta\right) 
			\bE  \hat{\vg}_t^\top \vh_t \nonumber\\
			&\quad \quad + \left( \frac{\beta \eta^2 (1-\xi_t)}{2}  +\beta \eta^2 \xi_t(1+ \rho^2 \beta^2 )\right) \frac{\sigma^2}{b} \nonumber\\
			& \leq \bE \hL(\hD; \vw_t) 
			- \eta  \left(1-\xi_t -\frac{\beta \eta (1-\xi_t)}{2}
			+  \frac{\beta \eta  \xi_t (1-2\rho^2\beta^2) }{2} 
			+  \xi_t \left(1-\eta \beta\right)( \frac{1}{2}  - \rho\beta)
			\right) \bE \| \hat{\vg}_t \|^2  \nonumber\\
			&\quad \quad + \left(\frac{\beta \eta^2 (1-\xi_t)}{2} + \beta \eta^2 \xi_t(1+ \rho^2 \beta^2 ) + \eta \xi_t \left(1-\eta \beta\right)\frac{\beta^2 \rho^2}{2} \right) \frac{\sigma^2}{b} \label{eq:tmep-ashj2} \\
			&\leq  \bE \hL(\hD; \vw_t) 
			- \eta \left( 1- (1+\beta \eta - 2 \rho \beta )\frac{\xi_t}{2} - \frac{\beta \eta }{2}  \right) \bE \| \hat{\vg}_t \|^2  \nonumber \\
			&\quad \quad + \left(\eta +\xi_t (\eta + 2 \eta \rho^2 \beta^2 + \beta \rho^2 - \eta \beta^2 \rho^2)\right) \frac{\eta \beta\sigma^2}{2b},
			\label{eq:temp-ashj3}
		\end{align}
		\vskip -.1in
		where \eqref{eq:tmep-ashj1}
		follows from Claims 1 and 2,
		\eqref{eq:tmep-ashj2}
		follows from Lemma \ref{lem:lem-sgd}
		and $1-\eta \beta > 0$.
		As $\eta < \frac{1}{4\beta}$, we have $1+\beta \eta - 2\rho \beta \leq 3/2$ and $ \beta \eta  < 1/4$,
		thus,
		$ 1- (1+\beta \eta - 2 \rho \beta )\frac{\xi_t}{2} - \frac{\beta \eta }{2}  > 0$.
		
		Summing over $t$ on both sides of \eqref{eq:temp-ashj3} and rearranging, 
		we obtain
		\begin{align}
			\min_{0\leq t\leq T-1}  
			\bE \| \hat{\vg}_t \|^2  
			&\leq \frac{ \hL(\hD; \vw_0) - \bE \hL(\hD; \vw_T)}{ \eta \sum_{t=0}^{T-1}  \left( 1- (1+\beta \eta - 2 \rho \beta )\frac{\xi_t}{2} - \frac{\beta \eta }{2}  \right) } 
			\nonumber \\
			&\quad + \frac{\sum_{t=0}^{T-1}\left(\eta +\xi_t (\eta +  \eta \rho^2 \beta^2 + \beta \rho^2 )\right)}{\sum_{t=0}^{T-1} \left( 1- (1+\beta \eta - 2 \rho \beta )\frac{\xi_t}{2} - \frac{\beta \eta }{2}  \right) } \frac{\beta \sigma^2}{2b} \nonumber \\
			&=\frac{ \hL(\hD; \vw_0) - \bE \hL(\hD; \vw_T)}{ T\eta(1-\frac{\gamma \zeta}{2} - \frac{\beta\eta }{2})}
			+ \frac{T (\eta + \eta \kappa \zeta + \beta \rho^2 \zeta)\beta\sigma^2}{2bT(1-\frac{\gamma\zeta}{2} - \frac{\beta\eta }{2})} \label{eq:sgd-temp-asju}\\
			&=\frac{ \hL(\hD; \vw_0) - \bE \hL(\hD; \vw_T)}{ T\eta(1-\frac{\gamma \zeta}{2} - \frac{\beta\eta }{2})}
			+ \frac{ (1 +  \kappa \zeta + 4 \beta^2 \zeta)\eta \beta\sigma^2}{2b(1-\frac{\gamma\zeta}{2} - \frac{\beta\eta }{2})} \nonumber\\
			&= \frac{ \hL(\hD; \vw_0) - \bE \hL(\hD; \vw_T)}{T\eta (1-\frac{\gamma \zeta}{2} - \frac{\beta\eta}{2})} 
			+  \frac{ (1 +  \kappa \zeta +4 \beta^2 \zeta)\sigma^2}{8b\sqrt{T}(1-\frac{\gamma\zeta}{2} - \frac{\beta\eta }{2})} \\
			&\leq \frac{32\beta \left( \hL(\hD; \vw_0) - \bE \hL(\hD; \vw_T)\right)}{\sqrt{T} \left(7-6\zeta\right)} 
			+  \frac{ (1 + \zeta +  5 \beta^2 \zeta )\sigma^2}{b\sqrt{T}\left(7-6\zeta\right)}, \label{eq:sgd-temp-ask}
		\end{align}
		\vskip -.1in
		where $\gamma = 1+ \beta \eta - 2 \rho \beta \leq 3/2$, $\kappa = 1 +  \rho^2 \beta^2$,
		$\rho^2 = 1/\sqrt{T}$,
		and $\zeta = \frac{1}{T} \sum_{t=0}^{T-1} \xi_t \in [0,1]$.
		We thus finish the proof.
	\end{proof}
	
	\begin{corollary}
		\label{cor:conv-sam}
		Let $b$ be the mini-batch size.  If $\eta = \frac{1}{4\beta\sqrt{T}}$ and $\rho = 1/T^{\frac{1}{4}}$, SAM~\citep{foret2021sharpness}
		satisfies
		\begin{align}
			\min_{0\leq t\leq T-1} \bE  \| \nabla \hL(\hD; \vw_t) \|^2
			&\leq  \frac{32\beta \left( \hL(\hD; \vw_0) - \bE \hL(\hD; \vw_T)\right)}{\sqrt{T}} 
			+  \frac{ (2 +  5 \beta^2 )\sigma^2}{b\sqrt{T}}.
		\end{align}
	\end{corollary}	
	
	\begin{corollary}
		\label{thm:conv-sgd-batchsize}
		Let $b$ be the mini-batch size.  If $\eta = \frac{\sqrt{b}}{4\beta\sqrt{T}}$ and $\rho = 1/T^{\frac{1}{4}}$, algorithm $\hA$
		satisfies
		\begin{align}
			\min_{0\leq t\leq T-1} \bE  \| \nabla \hL(\hD; \vw_t) \|^2
			\leq 
			\frac{32\beta( \hL(\hD; \vw_0) - \bE \hL(\hD; \vw_T))}{\sqrt{Tb}(7-6\zeta )}
			+ \frac{(1 +  \zeta + 5\beta^2   \zeta )\sigma^2}{\sqrt{Tb}(7-6\zeta)},
		\end{align}
		where 
		$\zeta =  \frac{1}{T} \sum_{t=0}^{T-1}\xi_t\in [0,1]$.
	\end{corollary}
	
	\begin{proof}
		It follows from \eqref{eq:sgd-temp-asju}
		that 
		\begin{align}
			\min_{0\leq t\leq T-1}  
			\bE \| \hat{\vg}_t \|^2  
			&\leq \frac{ \hL(\hD; \vw_0) - \bE \hL(\hD; \vw_T)}{T\eta(1-\frac{\gamma \zeta}{2} - \frac{\beta\eta }{2})}
			+ \frac{\eta \beta (1 +  \kappa\zeta  + 4\beta^2 \zeta)\sigma^2}{2b(1-\frac{\gamma\zeta}{2} - \frac{\beta\eta }{2})} \\
			&\leq \frac{4\beta( \hL(\hD; \vw_0) - \bE \hL(\hD; \vw_T))}{\sqrt{Tb}(\frac{7}{8}-\frac{3}{4}\zeta )}
			+ \frac{(1 +  \zeta + 5\beta^2   \zeta )\sigma^2}{8\sqrt{Tb}(\frac{7}{8}-\frac{3\zeta}{4})} \\
			& =\frac{32\beta( \hL(\hD; \vw_0) - \bE \hL(\hD; \vw_T))}{\sqrt{Tb}(7-6\zeta )}
			+ \frac{(1 +  \zeta + 5\beta^2   \zeta )\sigma^2}{\sqrt{Tb}(7-6\zeta)}. 
		\end{align} 
	\end{proof}

	\subsection{Convergence of full-batch gradient descent for AE-SAM}
	
	\begin{theorem}
		\label{thm:conv-gd}
		Under Assumption \ref{ass:smooth}, 
		with full-batch gradient descent,
		if $\rho < \frac{1}{2\beta}$ and $\eta<  \frac{1}{\beta}$,
		algorithm $\hA$
		satisfies
		\begin{align}
			\min_{0\leq t\leq T-1} \| \nabla \hL(\hD; \vw_t) \|^2  
			\leq \frac{  \hL(\hD; \vw_0) -  \hL(\hD; \vw_{T})}{ T\eta \left(1-\frac{\beta \eta }{2} - \beta \rho \zeta \right)},
		\end{align}
		where $\zeta =  \frac{1}{T} \sum_{t=0}^{T-1}\xi_t \in [0,1]$.
	\end{theorem}
	\begin{lemma}[Lemma 7 in 	
		\citet{andriushchenko2022towards}]
		\label{lem:lemma-est}
		Let $\hL(\hD;\vw)$ be a $\beta$-smooth function.
		For any $\rho >0$, we have
		\begin{align}
			\nabla \hL(\hD; \vw)^\top  \nabla \hL(\hD; \vw+ \rho \nabla \hL(\hD; \vw))
			\geq (1-\rho \beta) \| \nabla \hL(\hD; \vw)\|^2.
		\end{align}
	\end{lemma}
	
	\begin{proof}[Proof of Theorem \ref{thm:conv-gd}]
		Let $\vg_t \equiv \nabla \hL(\hD; \vw_t)$ and $\vh_t \equiv \nabla \hL(\hD; \vw_t + \rho \nabla\hL(\hD; \vw_t))$ be the update direction of ERM and SAM, respectively.
		By Taylor expansion and $\hL(\hD; \vw)$ is $\beta$-smooth, we have  
		\begin{align}
			&\hL(\hD; \vw_{t+1})  \nonumber\\
			&\leq \hL(\hD; \vw_t) + \vg_t^\top (\vw_{t+1} - \vw_t ) 
			+ \frac{\beta}{2}  \| \vw_{t+1} - \vw_t \|^2 \nonumber\\
			&\leq \hL(\hD; \vw_t) 
			- \eta \vg_t^\top  \left((1-\xi_t)\vg_t   + \xi_t \vh_t\right) + \frac{\beta \eta^2}{2}  \| (1-\xi_t)\vg_t   + \xi_t \vh_t \|^2 \nonumber\\
			&=\! \hL(\hD; \vw_t) 
			\!-\! \eta (1\!-\!\xi_t) \| \vg_t\|^2 
			\!- \!\eta \xi_t \vg_t^\top \vh_t
			\!+\! \frac{\beta \eta^2}{2} \!\!\left(\! (1\!-\!\xi_t) \| \vg_t\|^2 \!+\! \xi_t \| \vh_t\|^2 \!+\! \underbrace{2\xi_t (1\!-\!\xi_t) \vg_t^\top \vh_t}_{=0} \!\right)\! \label{eq:temp-xxa1} \\
			&=  \hL(\hD; \vw_t)  - \eta \left( 1-\xi_t -\frac{ \beta\eta (1-\xi_t)}{2} \right) \| \vg_t\|^2 + \frac{\beta \eta^2 \xi_t}{2} \| \vh_t \|^2 - \eta \xi_t \vg_t^\top \vh_t, \label{eq:temp-xxa2}
		\end{align}
		where we have used $\xi_t (1-\xi_t) = 0$ as $\xi_t\in \{0,1\}$, $\xi_t^2=\xi_t$, and $(1-\xi_t)^2=1-\xi_t$ to obtain \eqref{eq:temp-xxa1}. 
		
		As  $ \| \vh_t \|^2 = \| \vh_t - \vg_t \|^2 - \| \vg_t \|^2  + 2 \vg_t^\top \vh_t$,
		it follows from  \eqref{eq:temp-xxa2}
		that
		\begin{align}
			& \hL(\hD;\vw_{t+1})  \nonumber\\
			=&  \hL(\hD;\vw_t) \! -\! \eta \left(\!1\! -\! \xi_t - \frac{\beta \eta (1-\xi_t)}{2}\!\right) \!\| \vg_t \|^2\! +\! \frac{\beta \eta^2 \xi_t}{2} \left(\| \vh_t \!-\! \vg_t \|^2 \!-\!\| \vg_t \|^2  + 2 \vg_t^\top \vh_t  \right)\! -\!  \eta \xi_t \vg_t^\top  \vh_t  \nonumber \\
			\leq &
			\hL(\hD; \vw_t)  
			\!-\! \eta \left(\!1 \!-\! \xi_t\! -\! \frac{\beta \eta (1-\xi_t)}{2} + \frac{\beta \eta \xi_t}{2} \right) \| \vg_t \|^2 
			+ \frac{\beta \eta^2 \xi_t  }{2} \| \vh_t - \vg_t\|^2 
			- \eta  (1-\beta \eta ) \xi_t\vg_t^\top \vh_t   \nonumber
			\\
			\leq &  \hL(\hD; \vw_t)  \!-\! \eta \!\left(\!1\! -\! \xi_t\! -\! \frac{\beta \eta (1\!-\!\xi_t)}{2} \!+\! \frac{\beta \eta \xi_t}{2} \right) \| \vg_t \|^2 + \frac{\beta^3 \eta^2 \rho^2\xi_t  }{2} \| \vg_t\|^2 - \eta  (1-\beta \eta )\xi_t \vg_t^\top \vh_t   \label{eq:tempx-asak1}\\
			=&   \hL(\hD; \vw_t) - \eta \left( 1 - \xi_t - \frac{\beta \eta (1-\xi_t)}{2} + \frac{\beta \eta \xi_t}{2} + \frac{\beta^3 \eta  \rho^2 \xi_t}{2} +   (1-\beta \eta)(1-\beta \rho ) \xi_t \right) \| \vg_t \|^2 
			\label{eq:tempx-asak2}\\
			=&   \hL(\hD; \vw_t) - \eta \left( 1 -  \frac{\beta \eta (1-\xi_t)}{2} + \frac{\beta \eta \xi_t}{2} + \frac{\beta^3 \eta \xi_t \rho^2 }{2} -  \beta \eta \xi_t -  \beta \rho \xi_t +  \beta^2 \eta \rho \xi_t \right) \| \vg_t \|^2 \nonumber\\
			\leq 	&   \hL(\hD; \vw_t) - \eta \left( 1- \frac{\beta\eta }{2} -    \beta \rho \xi_t\right) \| \vg_t \|^2, \label{eq:temp-ajsk2}
		\end{align}
		where we have used $\| \vh_t - \vg_t \|^2 = \|  \nabla \hL(\hD; \vw_t + \rho \nabla\hL(\hD; \vw_t)) -  \nabla \hL(\hD; \vw_t)\|^2 \leq \beta^2 \rho^2 \|  \nabla \hL(\hD; \vw_t) \|^2 = \beta^2 \rho^2 \|\vg_t\|^2$ to obtain \eqref{eq:tempx-asak1},
		and Lemma \ref{lem:lemma-est} to obtain \eqref{eq:tempx-asak2}.
		
		Summing over $t$ from $t=0$ to $T-1$ on both sides
		of \eqref{eq:temp-ajsk2} and rearranging, we have
		\begin{align}
			\sum_{t=0}^{T-1} \eta  \left( 1- \frac{\beta\eta }{2} -  \beta \rho \xi_t \right) \| \vg_t \|^2 
			\leq   \hL(\hD; \vw_0) -  \hL(\hD;\vw_{T}).
			\label{eq:temp-axsk}
		\end{align}
		As $\rho < \frac{1}{2\beta}$ and $\eta <  \frac{1}{\beta}$,
		it follows that $1- \frac{\beta\eta }{2} -  \beta \rho \xi_t>0$ for all $t$.
		Thus,  \eqref{eq:temp-axsk} 
		implies
		\begin{align}
			\min_{0\leq t\leq T-1} \| \vg_t \|^2  \leq \frac{  \hL(\hD;\vw_0) -  \hL(\hD; \vw_{T})}{\sum_{t=0}^{T-1} \eta  \left( 1- \frac{\beta\eta }{2} -  \xi_t \beta \rho \right)} 
			&= \frac{  \hL(\hD; \vw_0) -  \hL(\hD; \vw_{T})}{ T\eta \left(1-\frac{\beta \eta }{2} - \beta \rho \zeta \right)},
		\end{align}
		where $\zeta = \frac{1}{T} \sum_{t=0}^{T-1} \xi_t  \in [0,1]$
		and we finish the proof.
	\end{proof}
	
	\section{Additional Experimental Results}
	\subsection{Distribution of Stochastic Gradient Norms}
	\label{apd:dist-gdn}
	
	\vskip -.1in
	Figure \ref{fig-apd:gradnorm-dist-cifar}
	shows the distributions
	of stochastic gradient norms 
	for \textit{ResNet-18}, \textit{WRN-28-10}
	and \textit{PyramidNet-110}
	on \textit{CIFAR-10}
	and
	\textit{CIFAR-100}.
	As can be seen, 
	the distribution follows a Bell curve
	in all settings.
	Figure \ref{fig-apd:gradnorm-qq-cifar}
	shows the Q-Q plots.
	We can see that
	the curves are close to the lines.
	
	\begin{figure}[!h]
		\vskip -.15in
		\centering
		\!\!\!
		\subfigure[\textit{ResNet-18}. \label{fig:gradnorm-dist-cifar10-resnet-18}]{\includegraphics[width=0.33\textwidth]{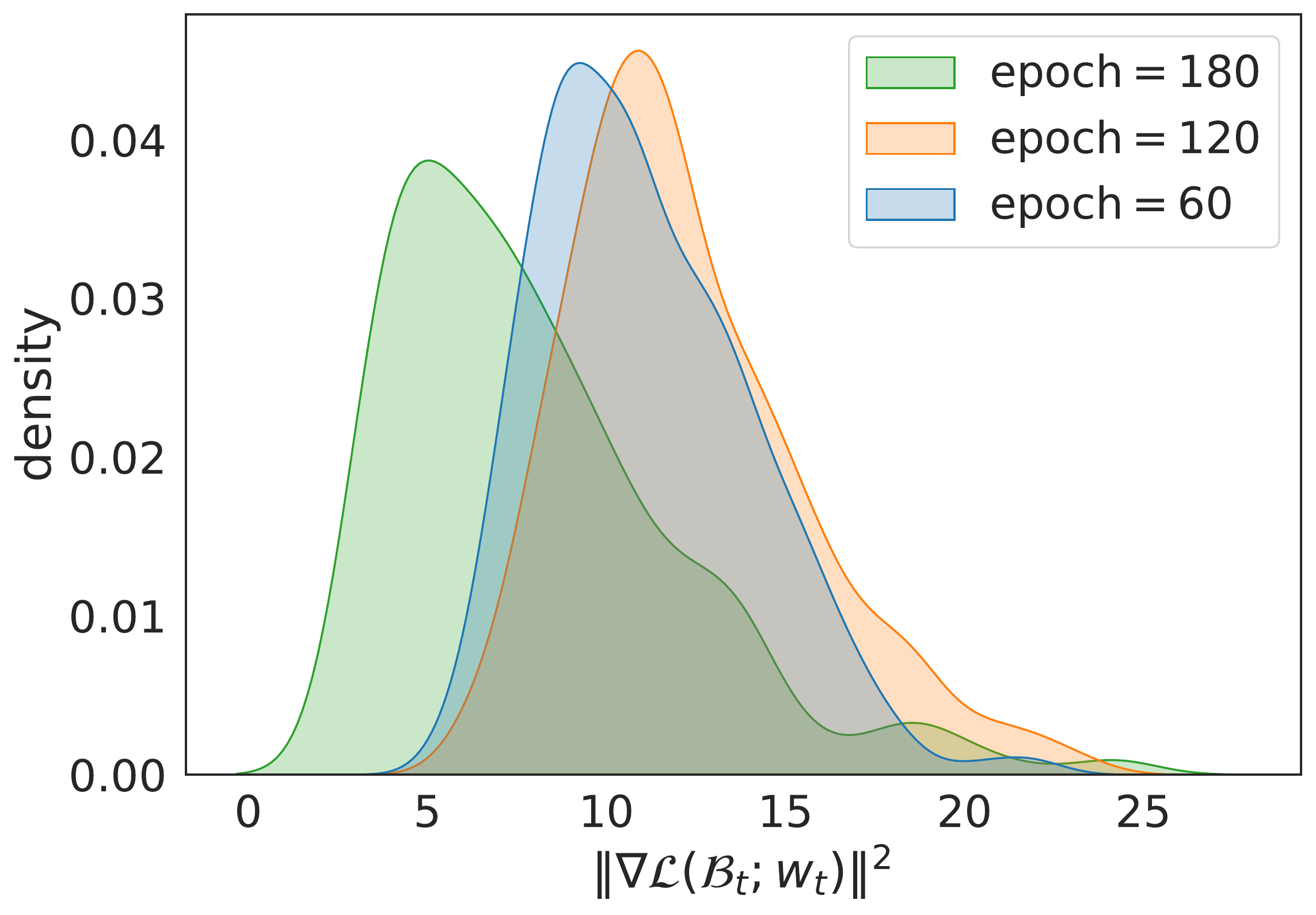}}
		\subfigure[\textit{WRN-28-10}. \label{fig:gradnorm-dist-cifar10-wrn2810}]{\includegraphics[width=0.33\textwidth]{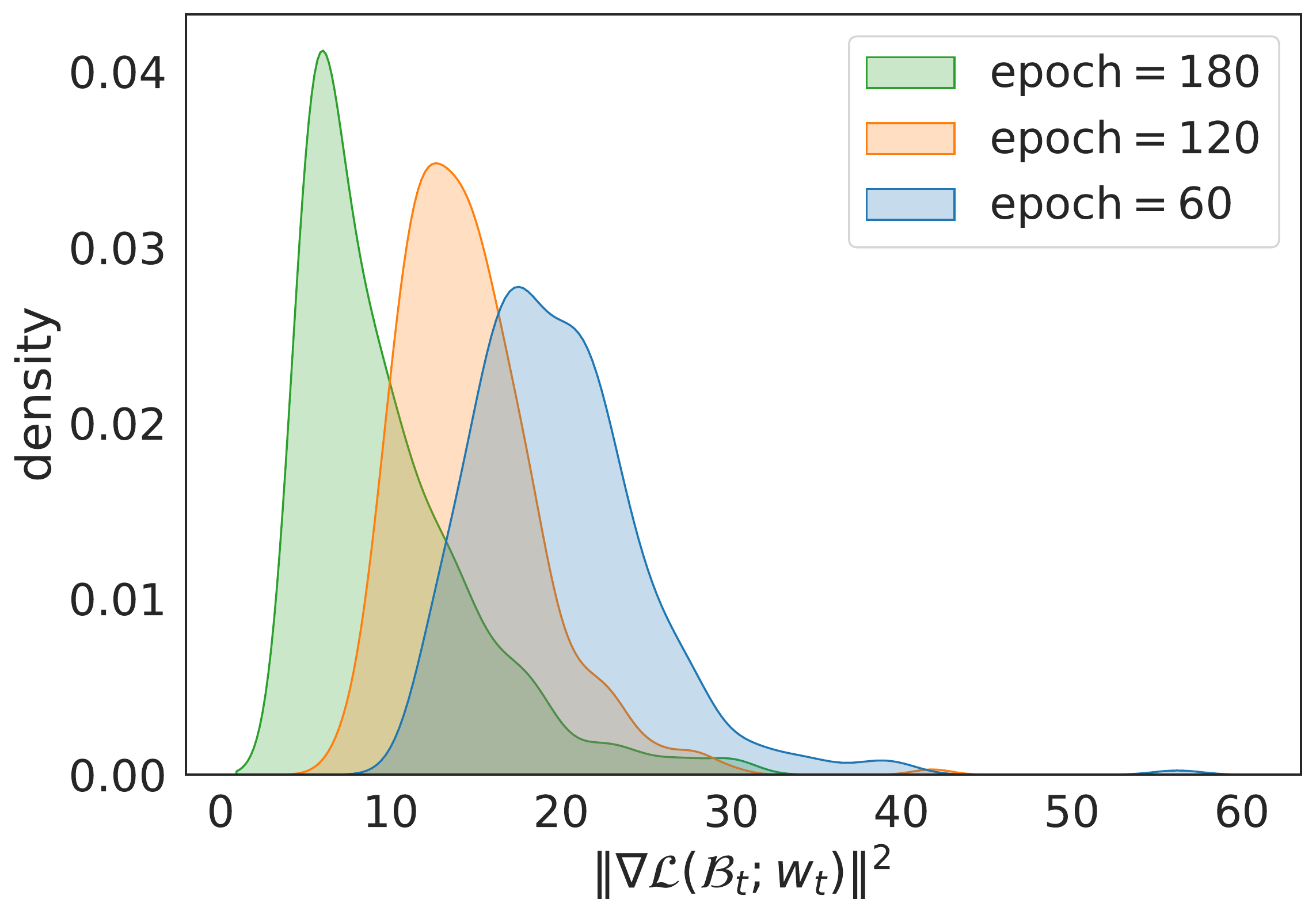}}
		\subfigure[\textit{PyramidNet-110}. \label{fig:gradnorm-dist-cifar10-pyramidnet}]{\includegraphics[width=0.33\textwidth]{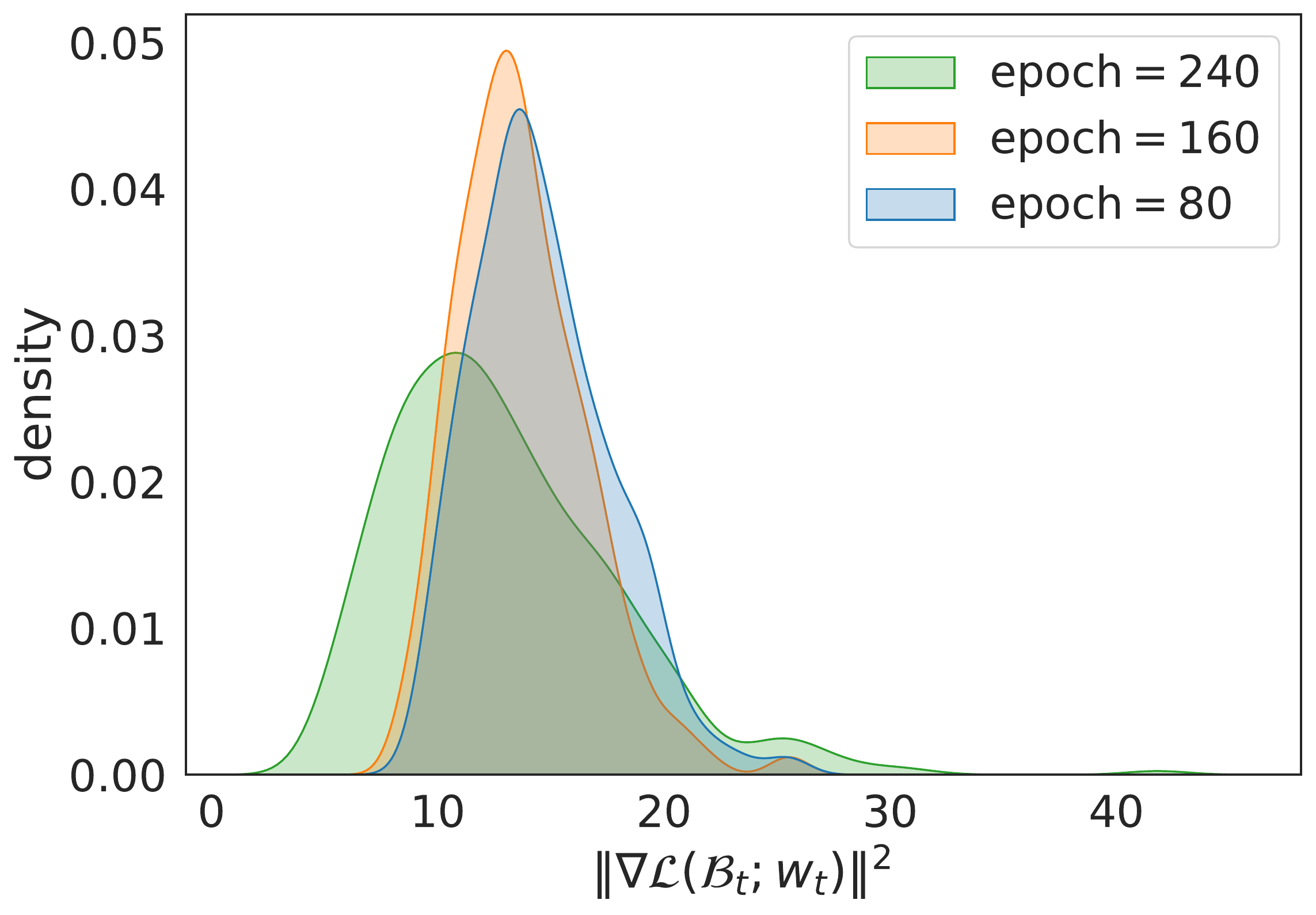}} \!\!
		\\
		\vskip -.15in
		\!\!
		\subfigure[\textit{ResNet-18}. \label{fig:apd-grad-norm-dist}]{\includegraphics[width=0.33\textwidth]{figs/gaussian_cifar100_resnet-18_grad_norm.pdf}}  \!\!\!
		\subfigure[\textit{WRN-28-10}. \label{fig:gradnorm-dist-cifar100-wrn2810}]{\includegraphics[width=0.33\textwidth]{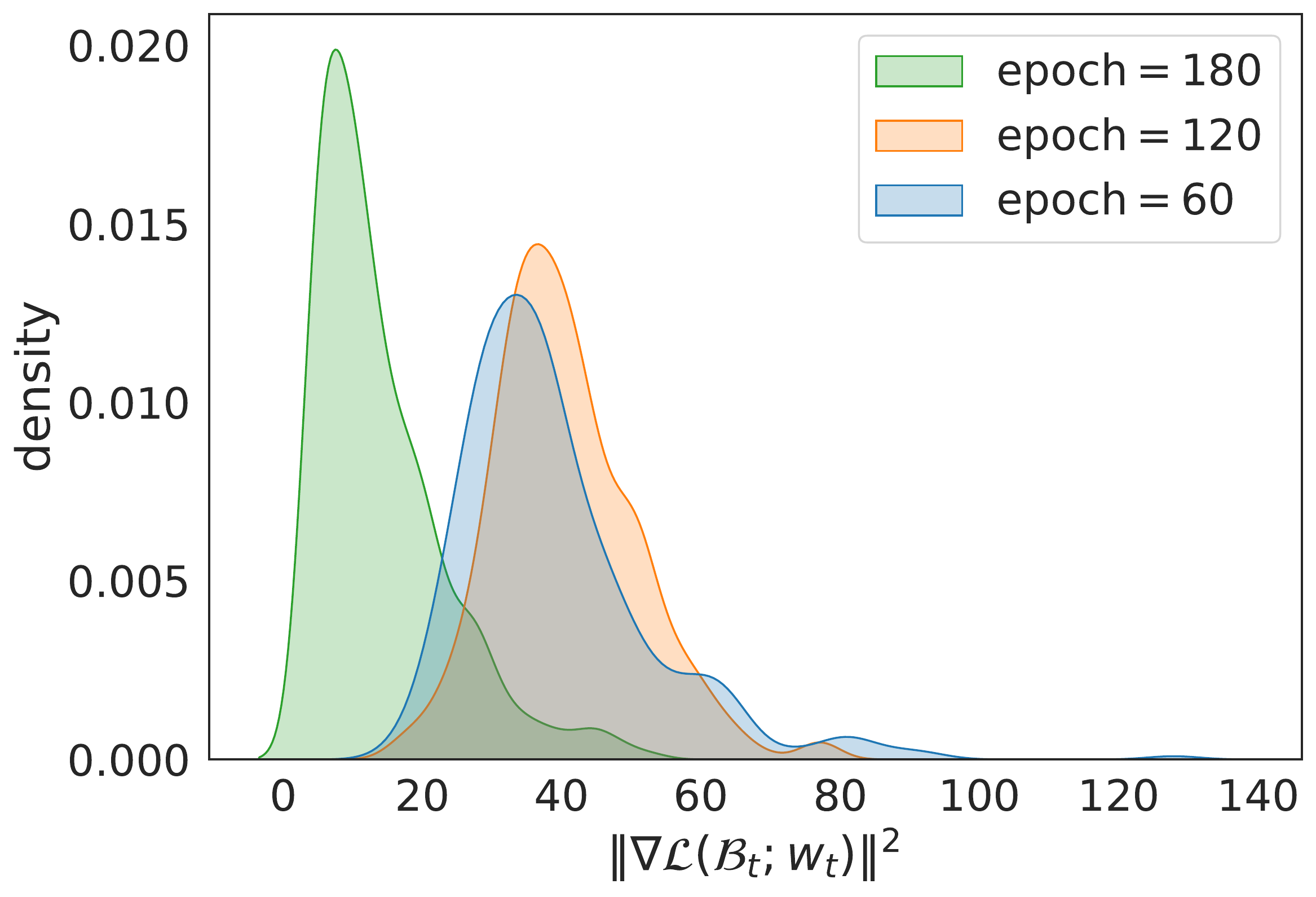}} \!\!\!
		\subfigure[\textit{PyramidNet-110}. \label{fig:gradnorm-dist-cifar100-pyramidnet}]{\includegraphics[width=0.335\textwidth]{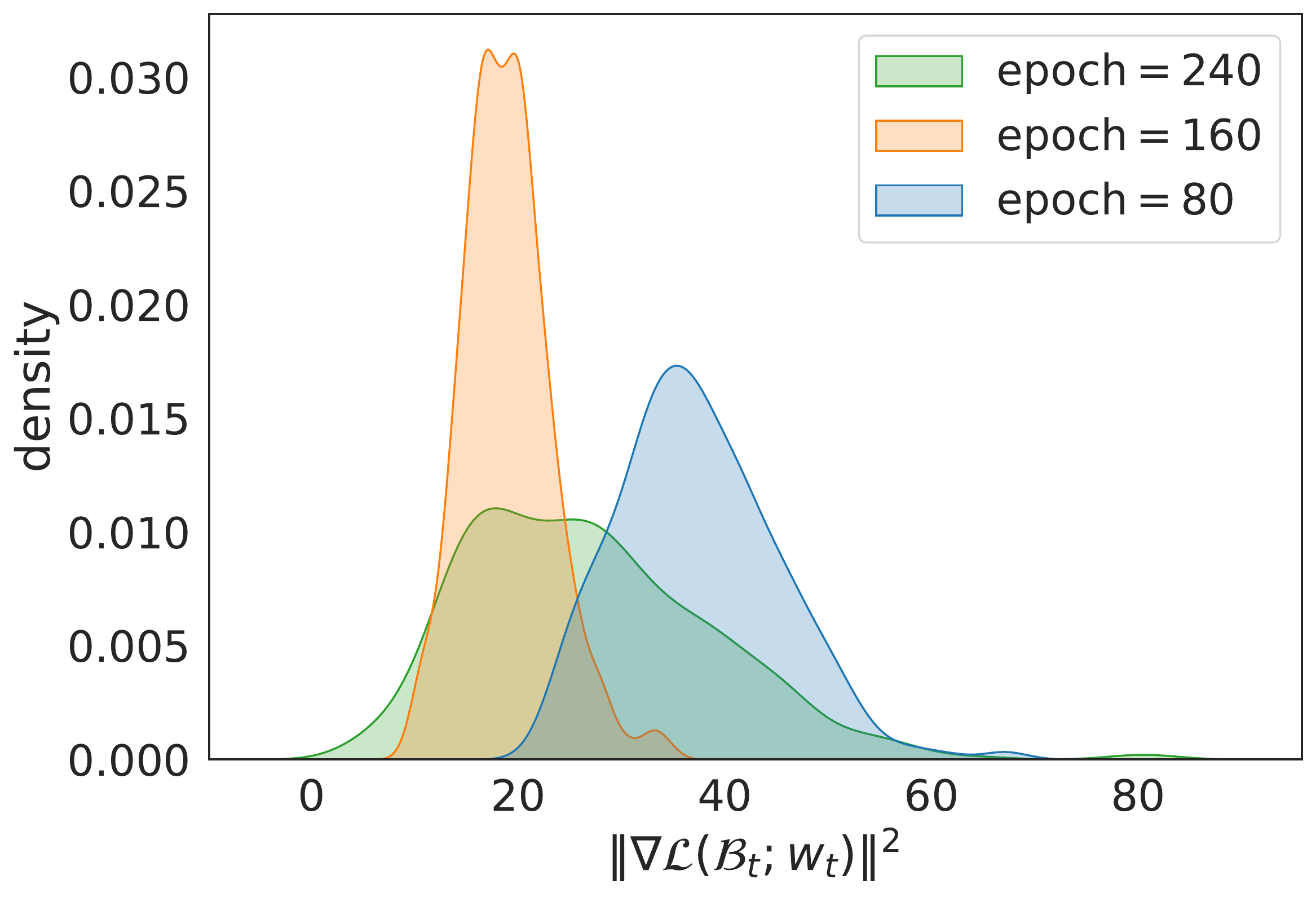}} \\
		\vskip -.15in
		\caption{
			Distributions of stochastic gradient norms on
			\textit{CIFAR-10} (top) and 
			\textit{CIFAR-100} (bottom).
			Best viewed in color.
		}
		\label{fig-apd:gradnorm-dist-cifar}
	\end{figure}
	
	\begin{figure}[!h]
		\vskip -.15in
		\centering
		\!\!\!
		\subfigure[\textit{ResNet-18}. \label{fig:qq-gradnorm-dist-cifar10-resnet-18}]{\includegraphics[width=0.33\textwidth]{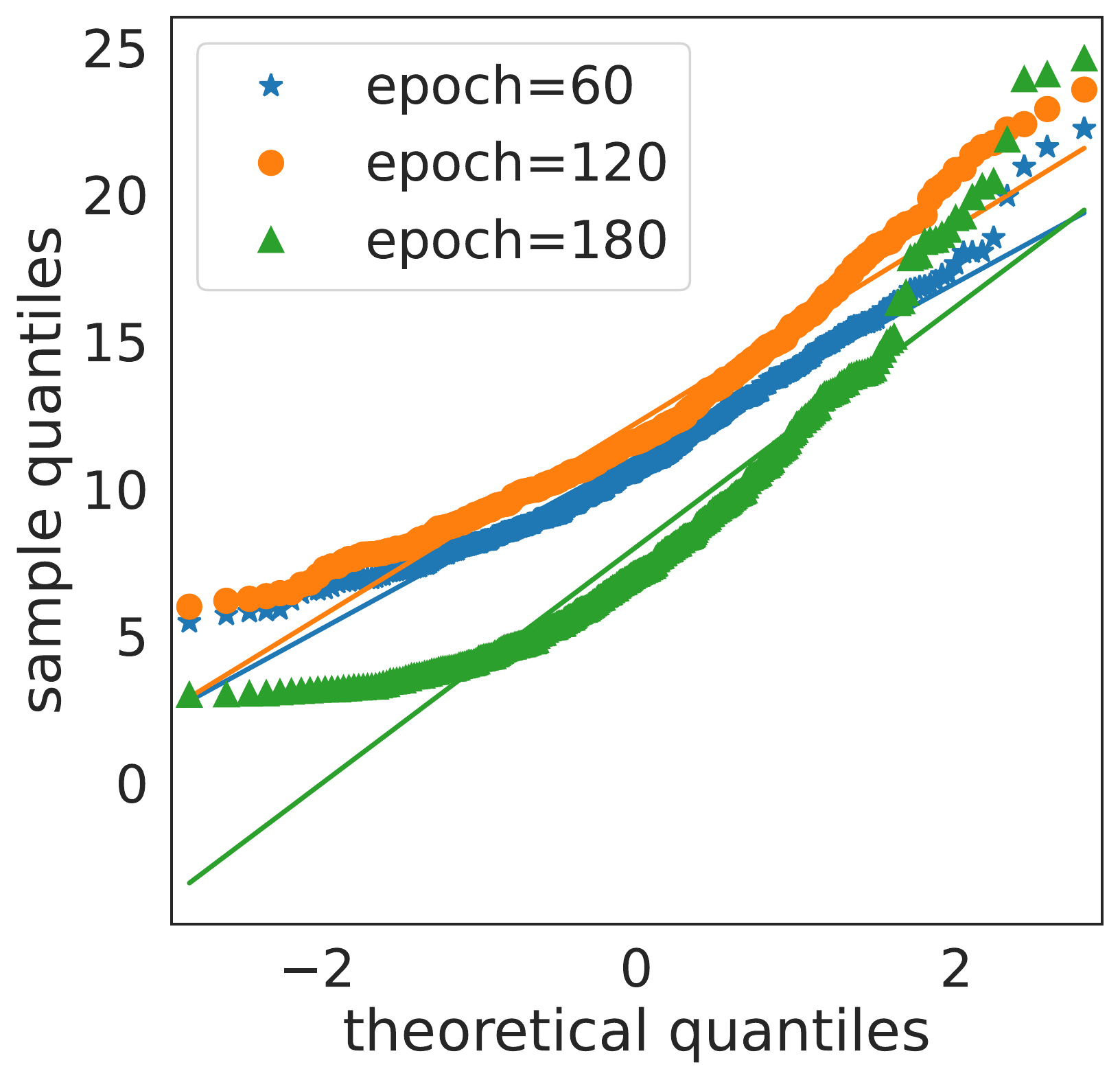}}
		\subfigure[\textit{WRN-28-10}. \label{fig:qq-gradnorm-dist-cifar10-wrn2810}]{\includegraphics[width=0.33\textwidth]{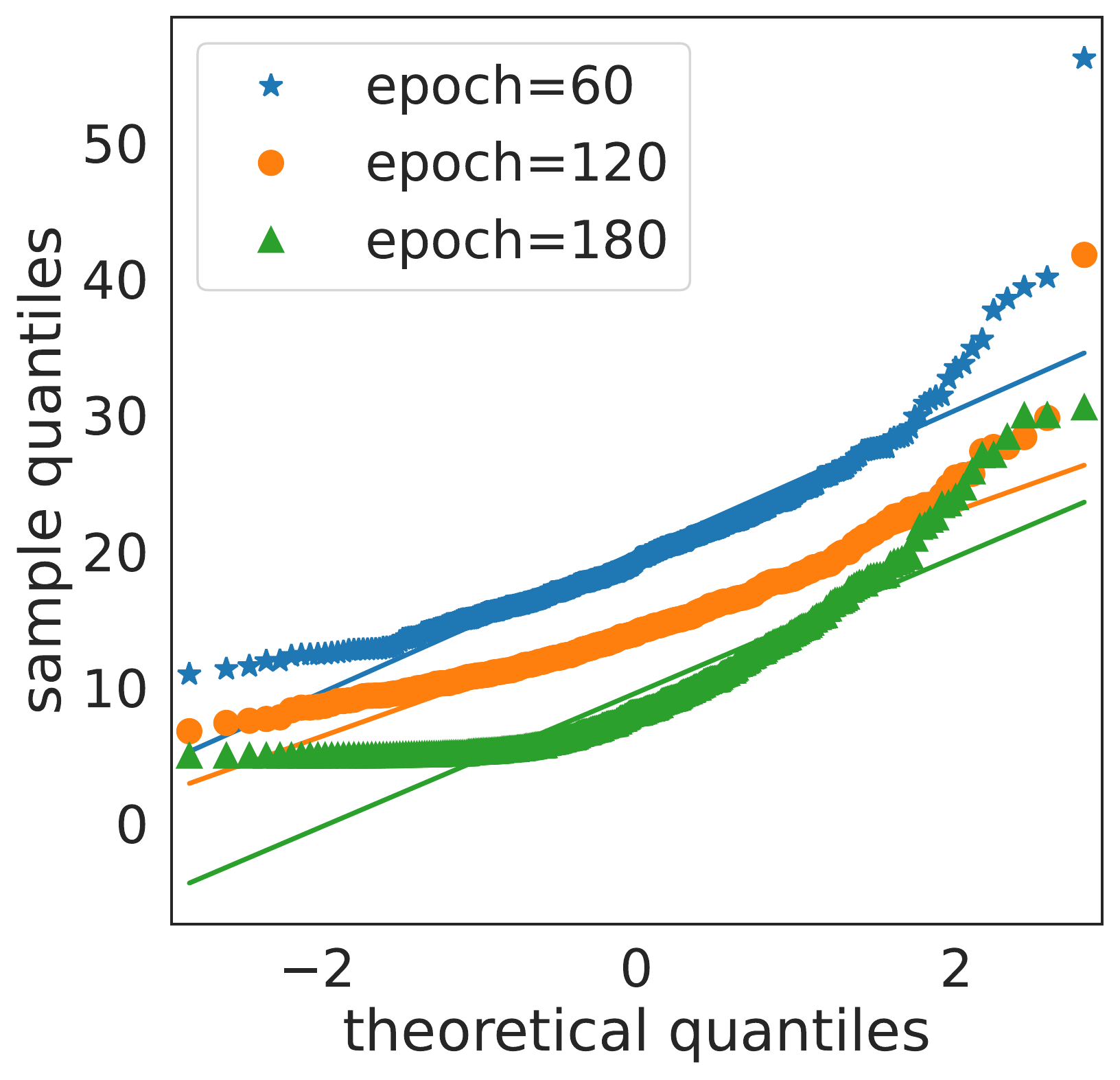}}
		\subfigure[\textit{PyramidNet-110}. \label{fig:qq-gradnorm-dist-cifar10-pyramidnet}]{\includegraphics[width=0.33\textwidth]{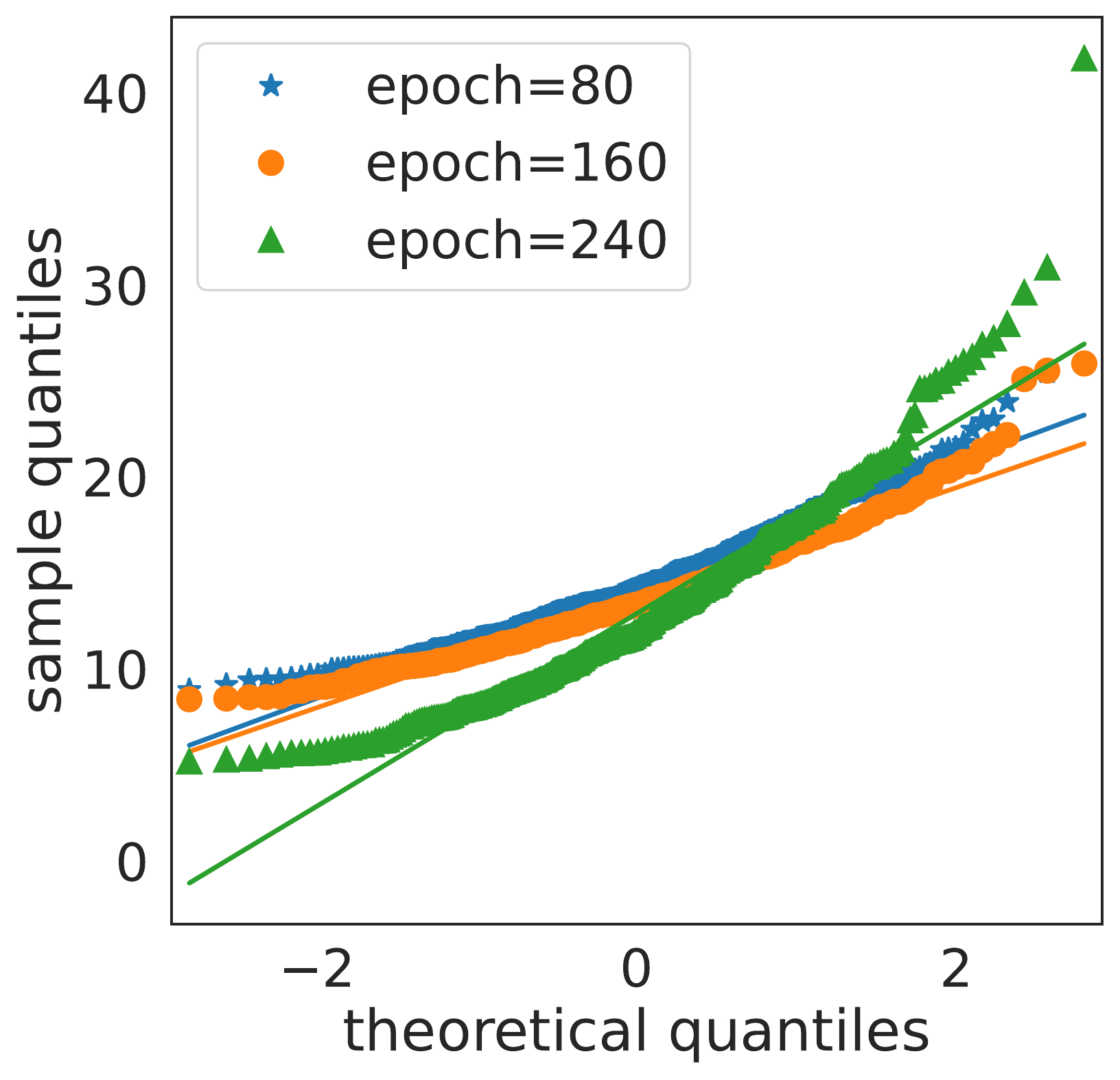}} \!\! \\
		\vskip -.15in
		\!\!
		\subfigure[\textit{ResNet-18}. \label{fig:qq-apd-grad-norm-dist}]{\includegraphics[width=0.33\textwidth]{figs/gaussian_cifar100_resnet-18_grad_norm_qqplot.pdf}}  \!\!\!
		\subfigure[\textit{WRN-28-10}. \label{fig:qq-gradnorm-dist-cifar100-wrn2810}]{\includegraphics[width=0.34\textwidth]{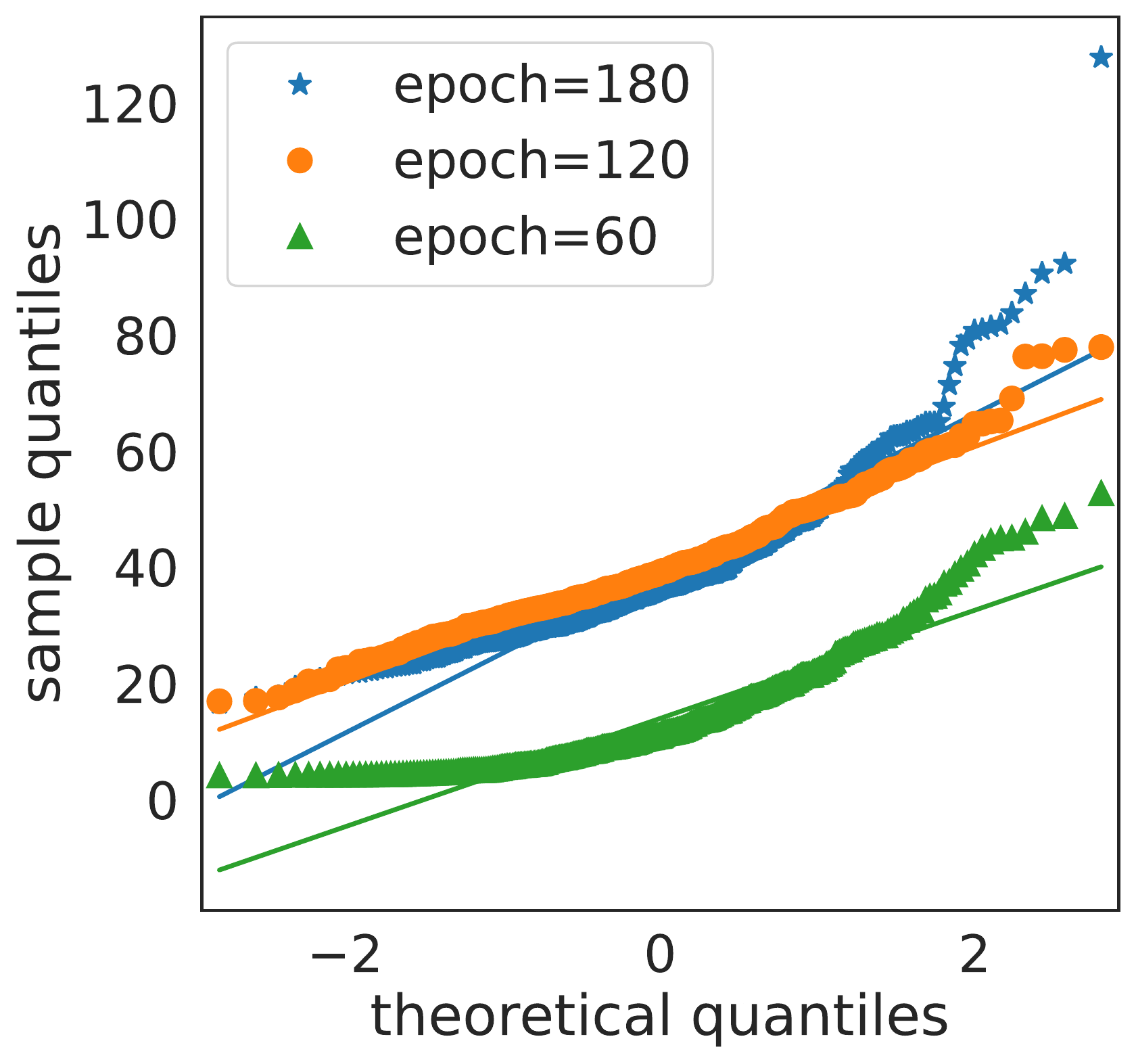}} \!\!\!
		\subfigure[\textit{PyramidNet-110}. \label{fig:qq-gradnorm-dist-cifar100-pyramidnet}]{\includegraphics[width=0.33\textwidth]{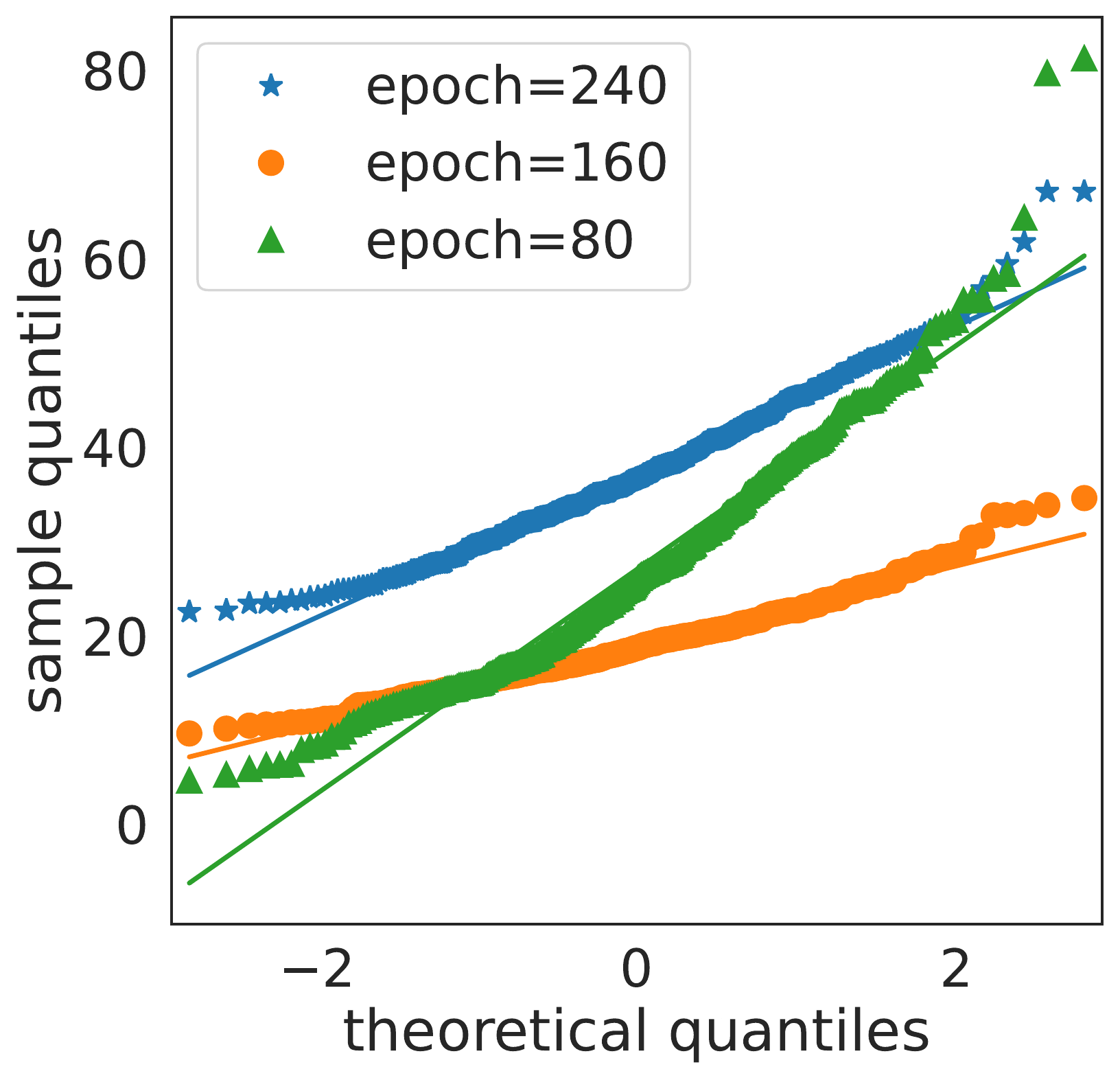}} \\
		\vskip -.15in
		\caption{
			Q-Q plots of stochastic gradient norms on
			\textit{CIFAR-10} (top) and 
			\textit{CIFAR-100} (bottom).
			Best viewed in color.
		}
		\label{fig-apd:gradnorm-qq-cifar}
		\vskip -.1in
	\end{figure}
	
	\subsection{Effect of $k$ on LookSAM}
	\label{sec:appendix-label-noise}
	
	\vskip -.1in
	In this
	experiment,
	we demonstrate that
	LookSAM 
	is sensitive to the choice of $k$.
	Table \ref{table:looksam-noisy-label}
	shows the testing accuracy and fraction of SAM updates when using LookSAM 
	on noisy \textit{CIFAR-10},
	with
	$k\in \{2,3,4,5\}$ and the \textit{ResNet-18} model.
	As can be seen,
	$k=2$ yields much better performance than $k\in \{3,4,5\}$, particularly at higher noise levels (e.g., $80\%$).
	
	\begin{table}[!h]
		\centering
		\vskip -.15in		
		\caption{Effects of $k$ in LookSAM on \textit{CIFAR-10}
			with different levels of label noise using \textit{ResNet-18}.
		}
		\label{table:looksam-noisy-label}
		\resizebox{.98\textwidth}{!}{
			\begin{tabular}{c c c c c c c c c}
				\toprule
				& \multicolumn{2}{c}{$\text{noise }=20\%$}
				& \multicolumn{2}{c}{$\text{noise }=40\%$}
				& \multicolumn{2}{c}{$\text{noise }=60\%$}
				& \multicolumn{2}{c}{$\text{noise }=80\%$} \\
				$k$ &
				accuracy & \fracsam
				&accuracy & \fracsam
				&accuracy & \fracsam
				&accuracy & \fracsam	\\
				\midrule
				2 & $\mathbf{92.72}$ 	& $50.0$  
				& $\mathbf{88.04}$   	&$50.0$   
				& $\mathbf{72.26}$   &$50.0$  
				& $\mathbf{69.72}$  & $50.0$ 
				\\
				3 & $89.07$ & $33.3$  & $75.38$ 
				&$33.3$  & $63.79$  
				&$33.3$ & $	53.87$   & $33.3$ 
				\\
				4 & $89.00$ & $25.0$  & $74.12$ 
				&$25.0$  & $58.17$  
				&$25.0$ & $52.28$   & $25.0$  
				\\
				5 & $88.57$ & $20.0$  & $73.90$ 
				&$20.0$  & $56.80$  
				&$20.0$ & $51.82$   & $20.0$  
				\\
				\bottomrule
			\end{tabular}
		}
	\end{table}

	\subsection{More Results on    Robustness to Label Noise}
	\label{sec:robust}
	
	Figure \ref{fig:curve-noisy-resnet18-apd} (resp.
	\ref{fig:curve-noisy-resnet32-apd})
	shows the curves of accuracies at noise levels 
	of $20\%$, $40\%$, $60\%$, and $80\%$ with
	\textit{ResNet-18} (resp.
	\textit{ResNet-32}).
	As can be seen, in all settings,
	AE-LookSAM
	is as robust to label noise as SAM.
	
	\begin{figure}[!h]
		\centering
		\vskip -.15in
		\!\!\!\!\!\!
		\subfigure[20\% (Training). \label{fig:noisy-cifar10-20-resnet18-train}]{\includegraphics[width=0.25\textwidth]{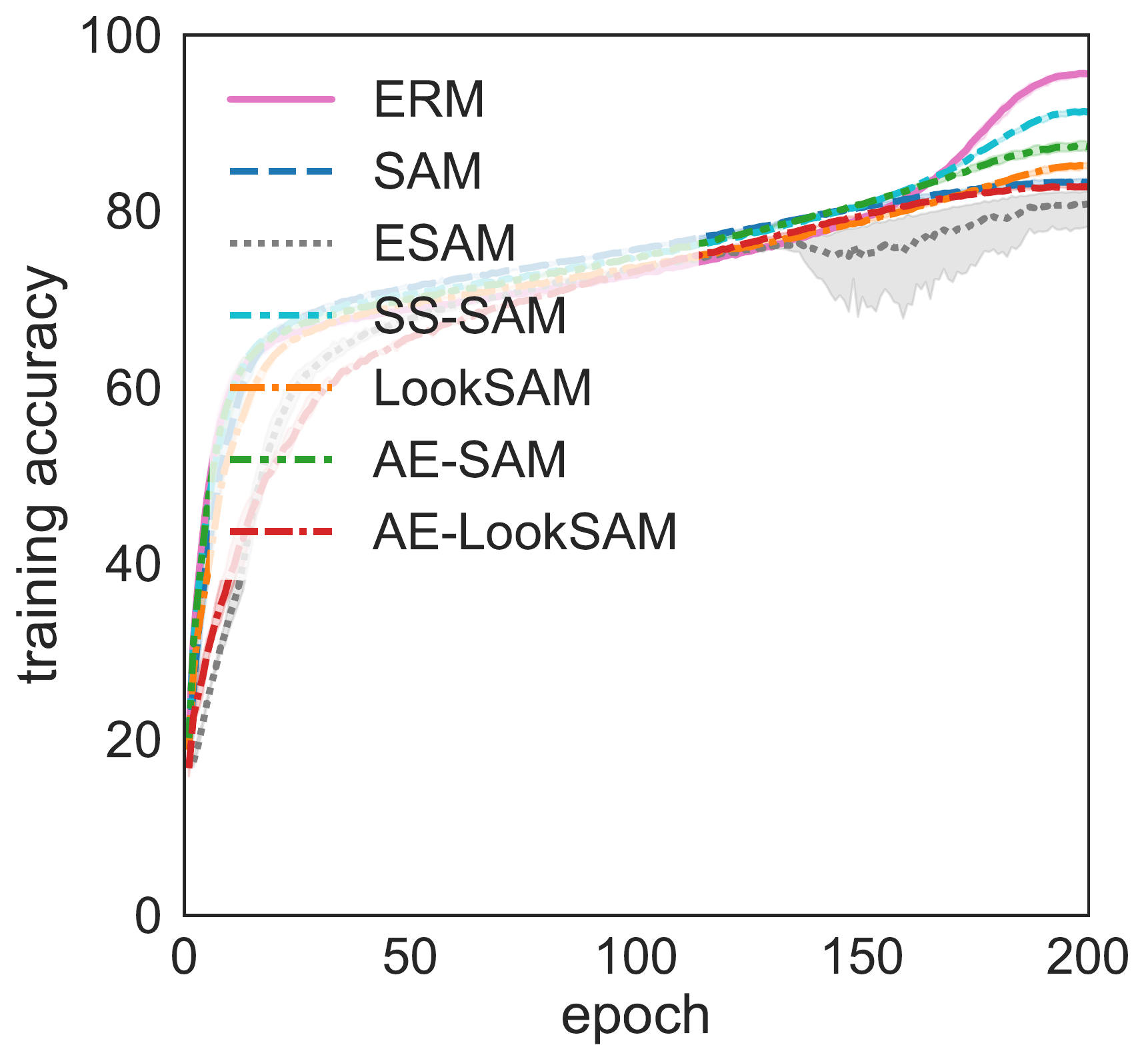}}\!\!
		\subfigure[20\% (Testing). \label{fig:noisy-cifar10-20-resnet18-valid}]{\includegraphics[width=0.25\textwidth]{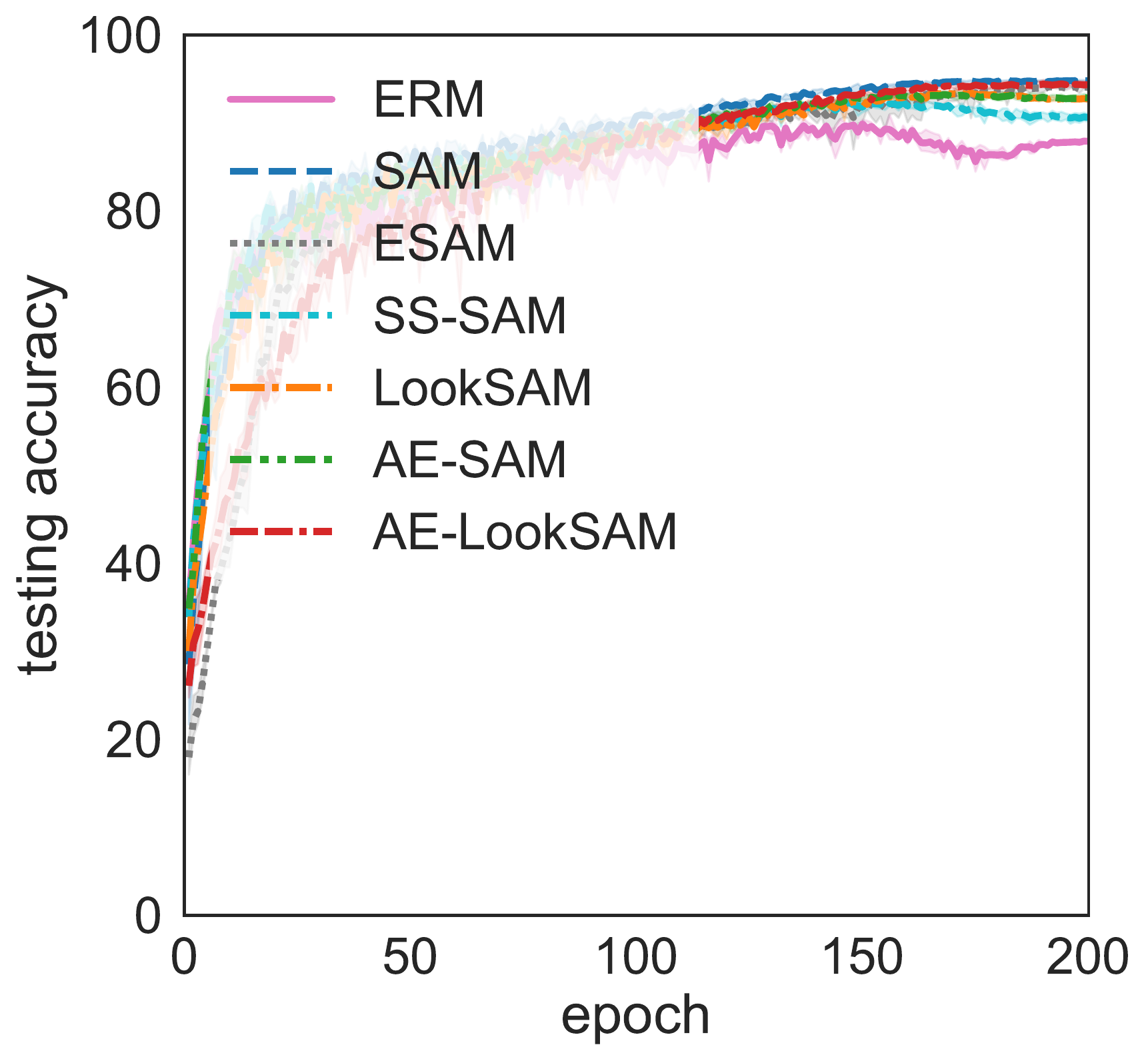}} \!\!
		\subfigure[40\% (Training).
		\label{fig:noisy-cifar10-40-resnet18-train}]{\includegraphics[width=0.25\textwidth]{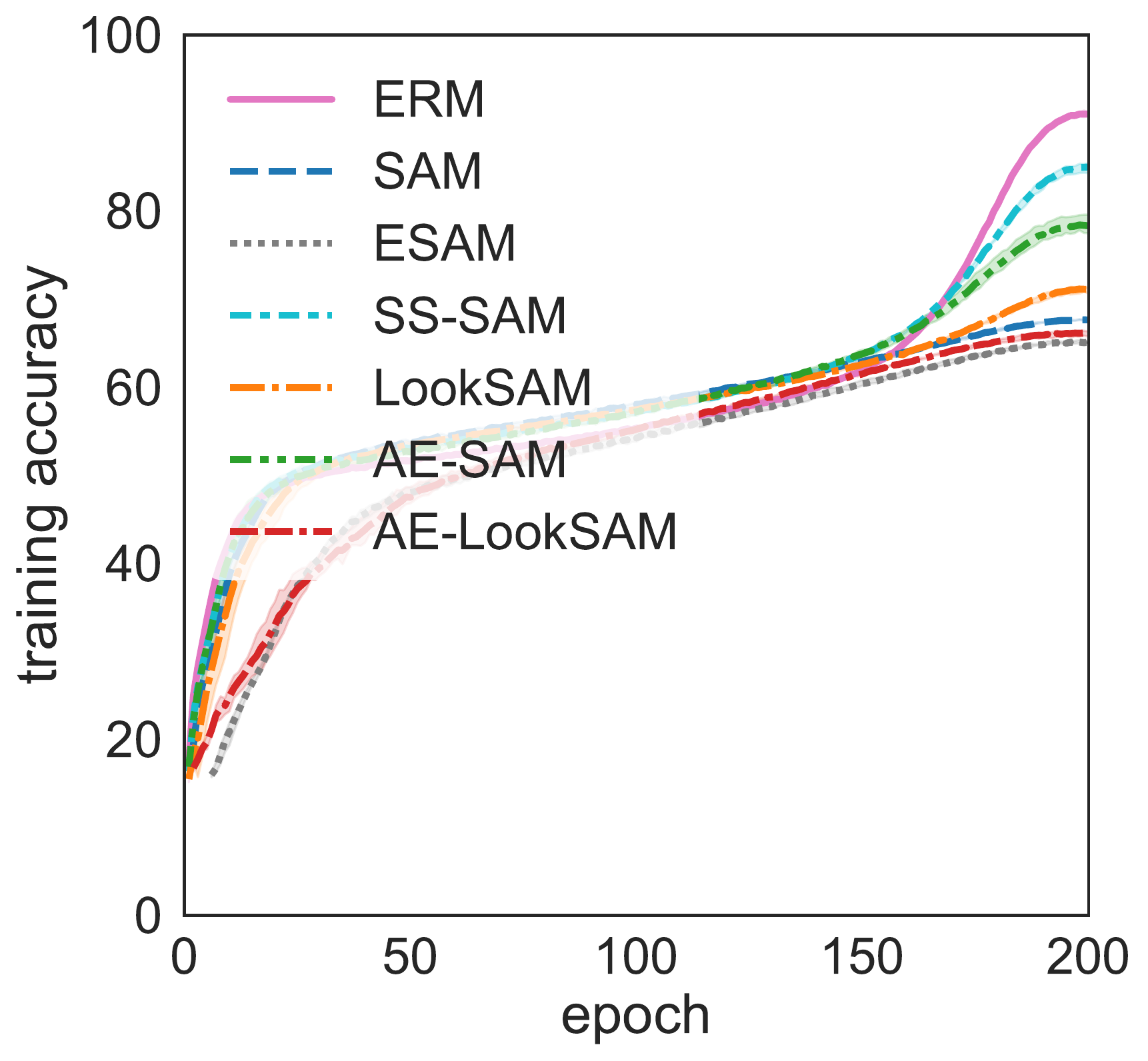}} \!\! 
		\subfigure[40\% (Testing). \label{fig:noisy-cifar10-40-resnet18-valid}]{\includegraphics[width=0.25\textwidth]{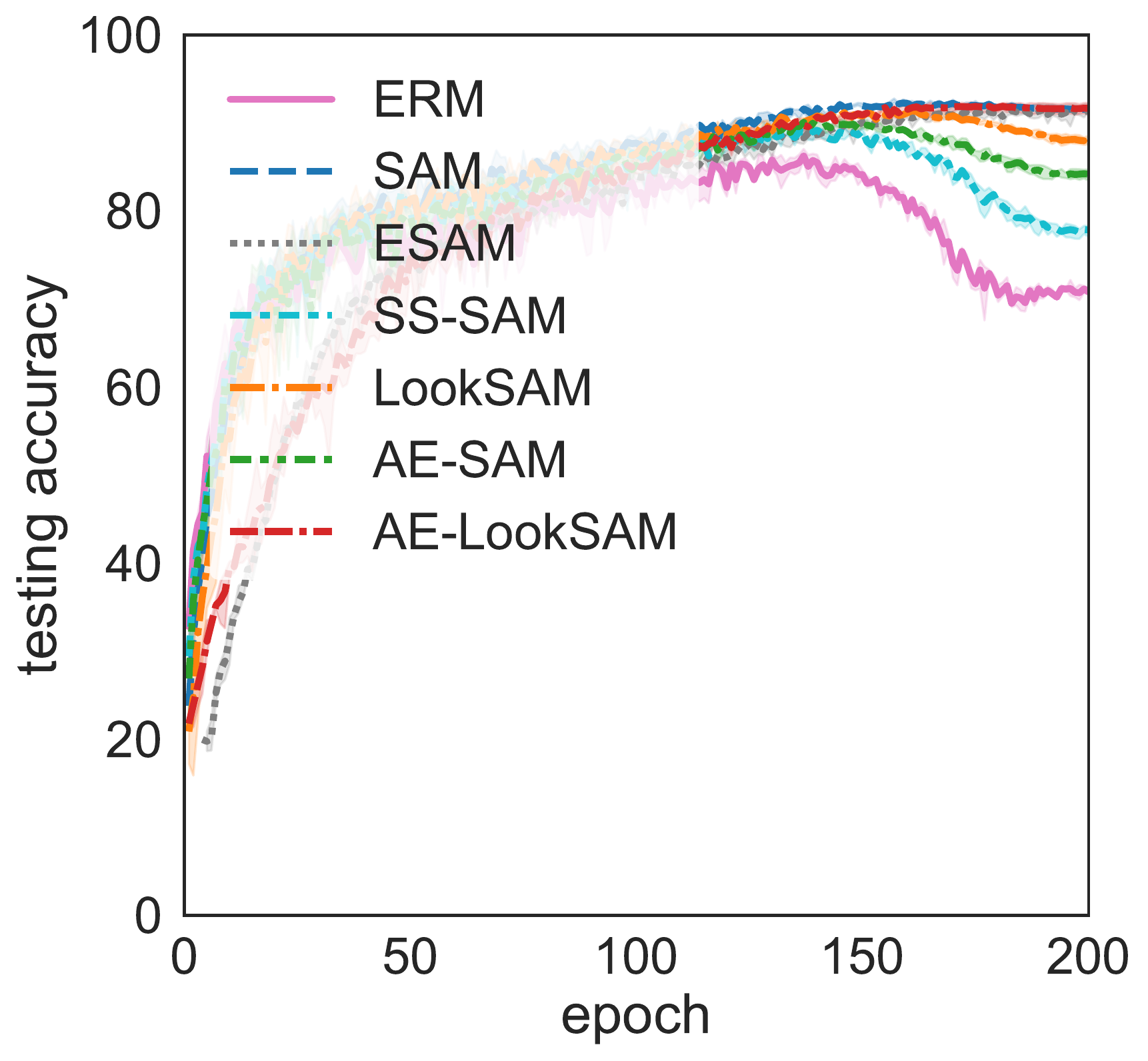}} \!\! \\
		\vskip -.15in
		\!\!\!\!\!\!
		\subfigure[60\% (Training).
		\label{fig:noisy-cifar10-60-resnet18-train}]{\includegraphics[width=0.25\textwidth]{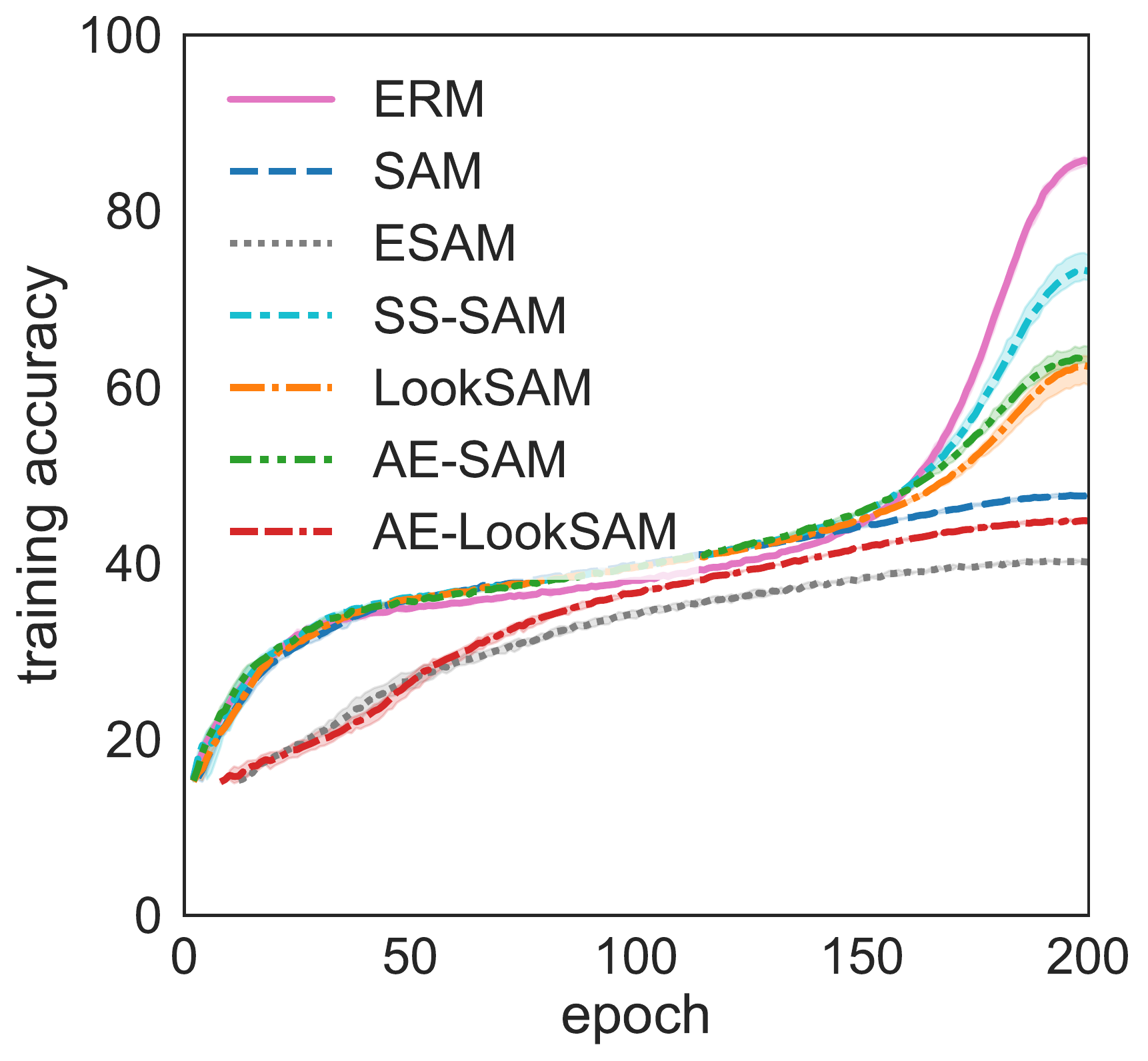}} \!\!
		\subfigure[60\% (Testing). \label{fig:noisy-cifar10-60-resnet18-valid}]{\includegraphics[width=0.25\textwidth]{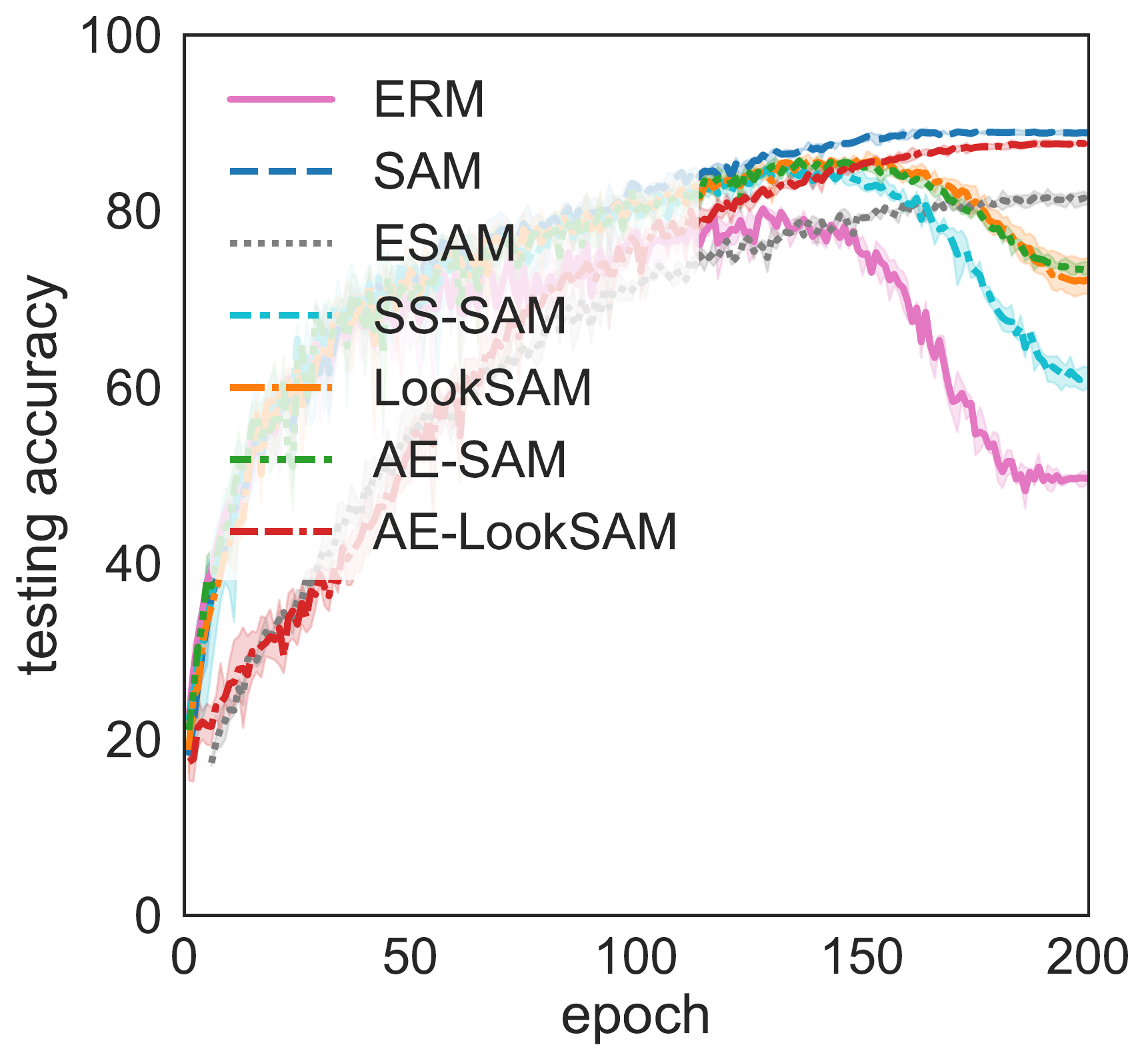}} \!\!
		\subfigure[80\% (Training). \label{fig:noisy-cifar10-80-resnet18-train-apd}]{\includegraphics[width=0.245\textwidth]{figs/noisy_cifar10_80_resnet18_train}} \!\!
		\subfigure[80\% (Testing). \label{fig:noisy-cifar10-80-resnet18-valid-apd}]{\includegraphics[width=0.245\textwidth]{figs/noisy_cifar10_80_resnet18_valid}} \!\!\!
		\vskip -.15in
		\caption{Accuracies with number of epochs on CIFAR-10 with $20\%, 40\%$, $60\%$, and $80\%$ noise level using \textit{ResNet-18}.
			Best viewed in color.
		}
		\label{fig:curve-noisy-resnet18-apd}
	\end{figure}
	
	\begin{figure}[!h]
		\centering
		\vskip -.15in
		\!\!\!\!\!\!
		\subfigure[20\% (Training). \label{fig:noisy-cifar10-20-resnet32-train}]{\includegraphics[width=0.25\textwidth]{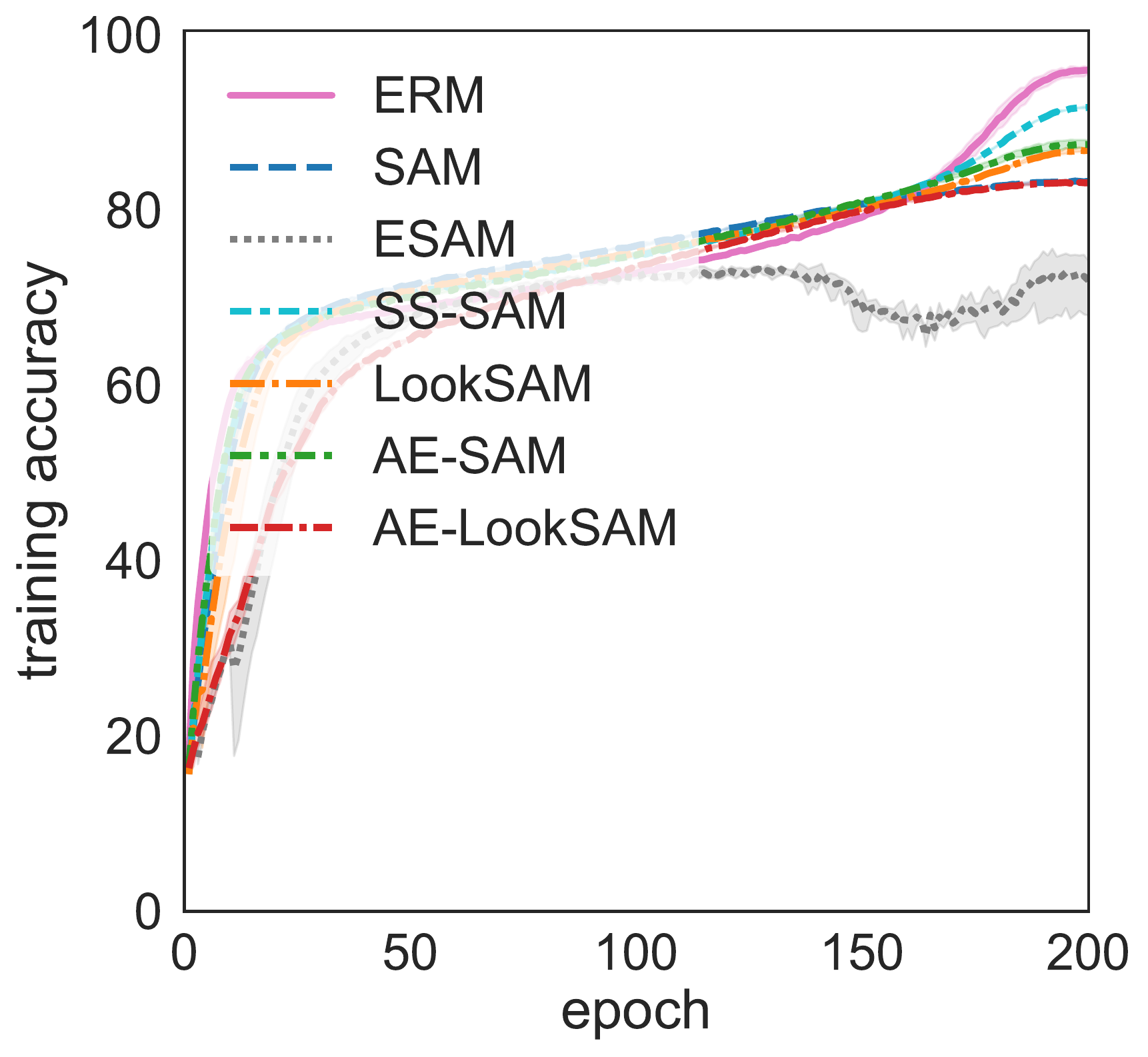}}\!\!
		\subfigure[20\% (Testing). \label{fig:noisy-cifar10-20-resnet32-valid}]{\includegraphics[width=0.25\textwidth]{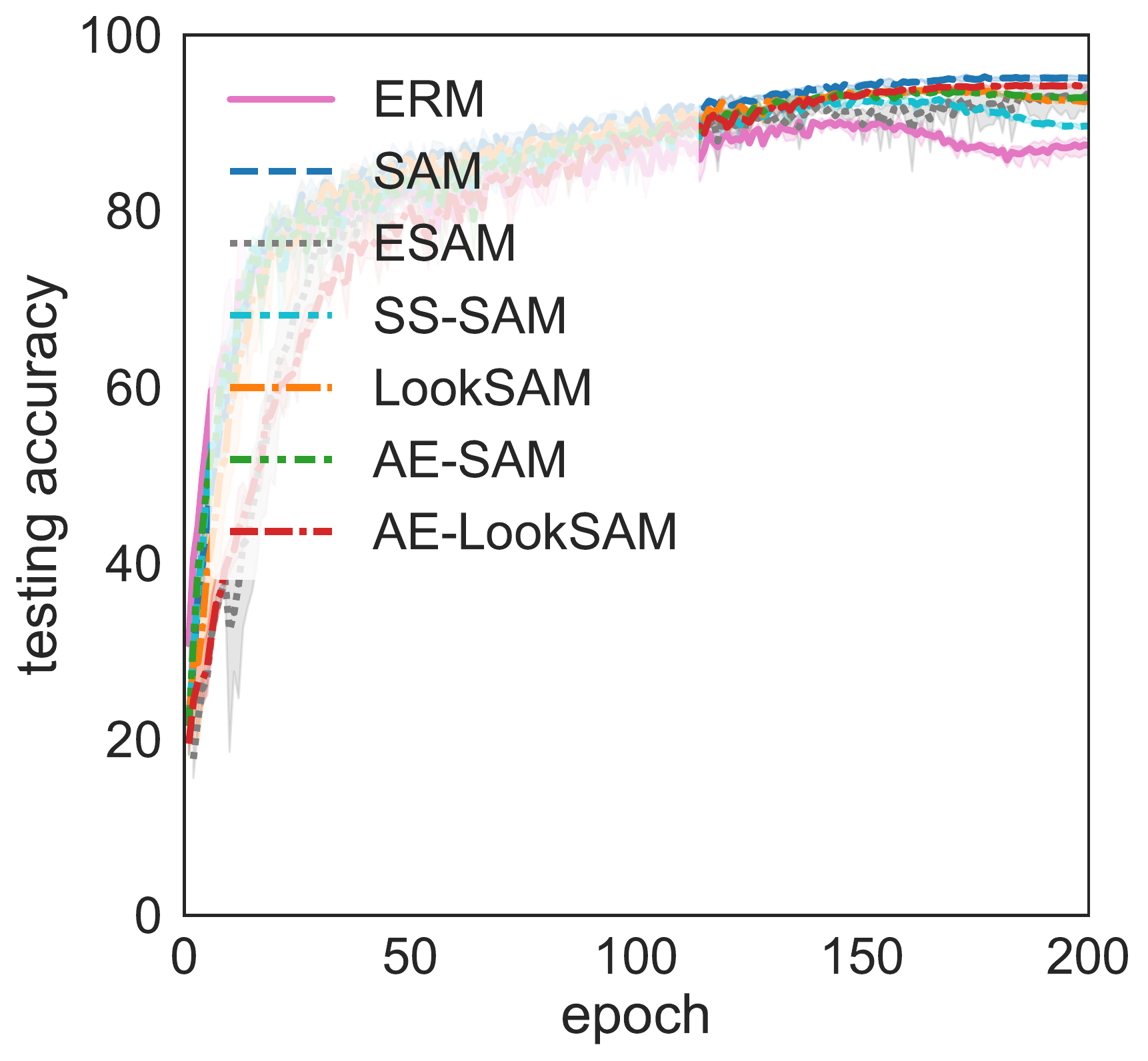}} \!\!
		\subfigure[40\% (Training).
		\label{fig:noisy-cifar10-40-resnet32-train}]{\includegraphics[width=0.25\textwidth]{figs/noisy_cifar10_40_resnet18_train}} \!\! 
		\subfigure[40\% (Testing). \label{fig:noisy-cifar10-40-resnet32-valid}]{\includegraphics[width=0.25\textwidth]{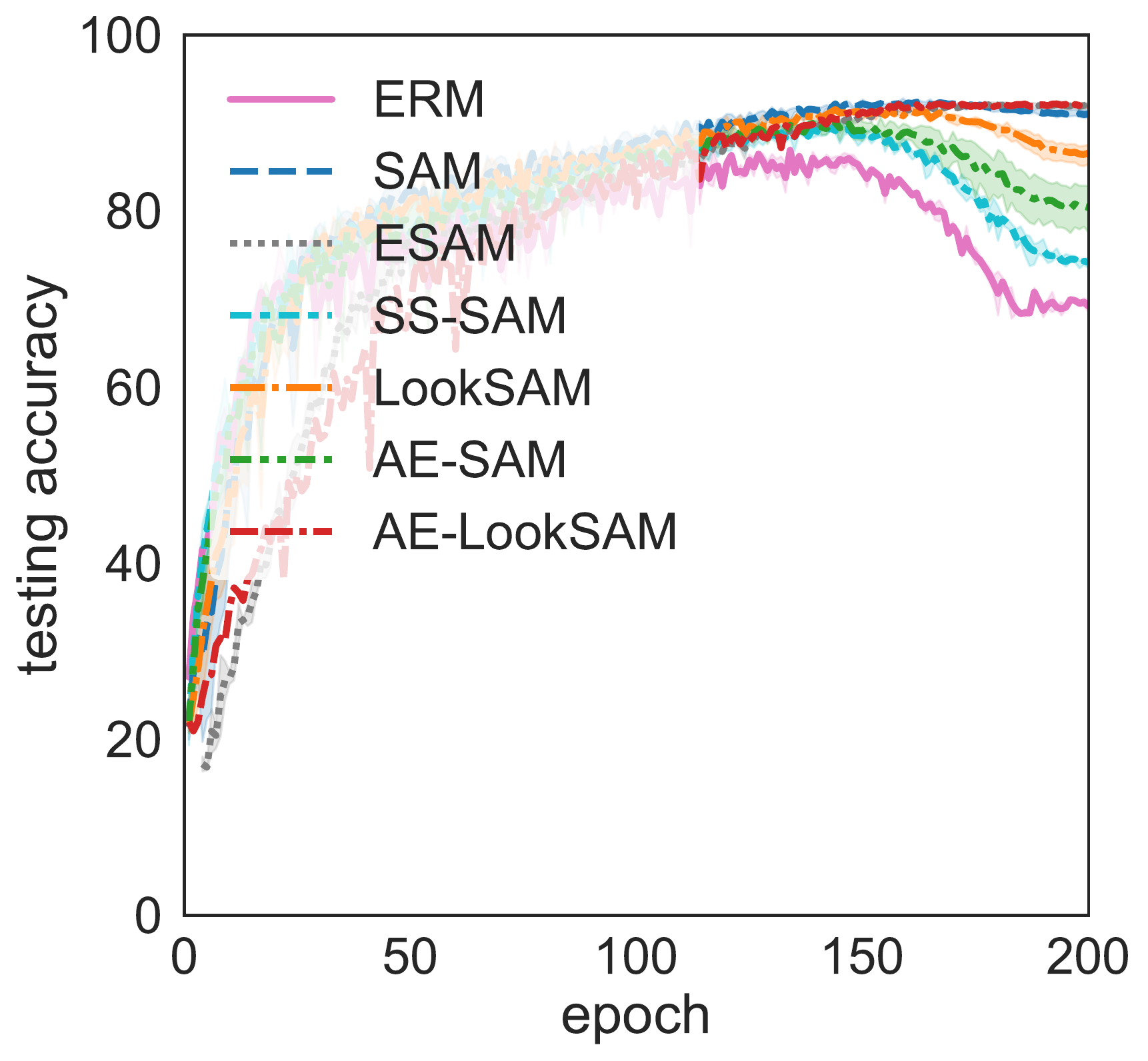}} \!\! \\
		\vskip -.15in
		\!\!\!\!\!\!
		\subfigure[60\% (Training).
		\label{fig:noisy-cifar10-60-resnet32-train}]{\includegraphics[width=0.25\textwidth]{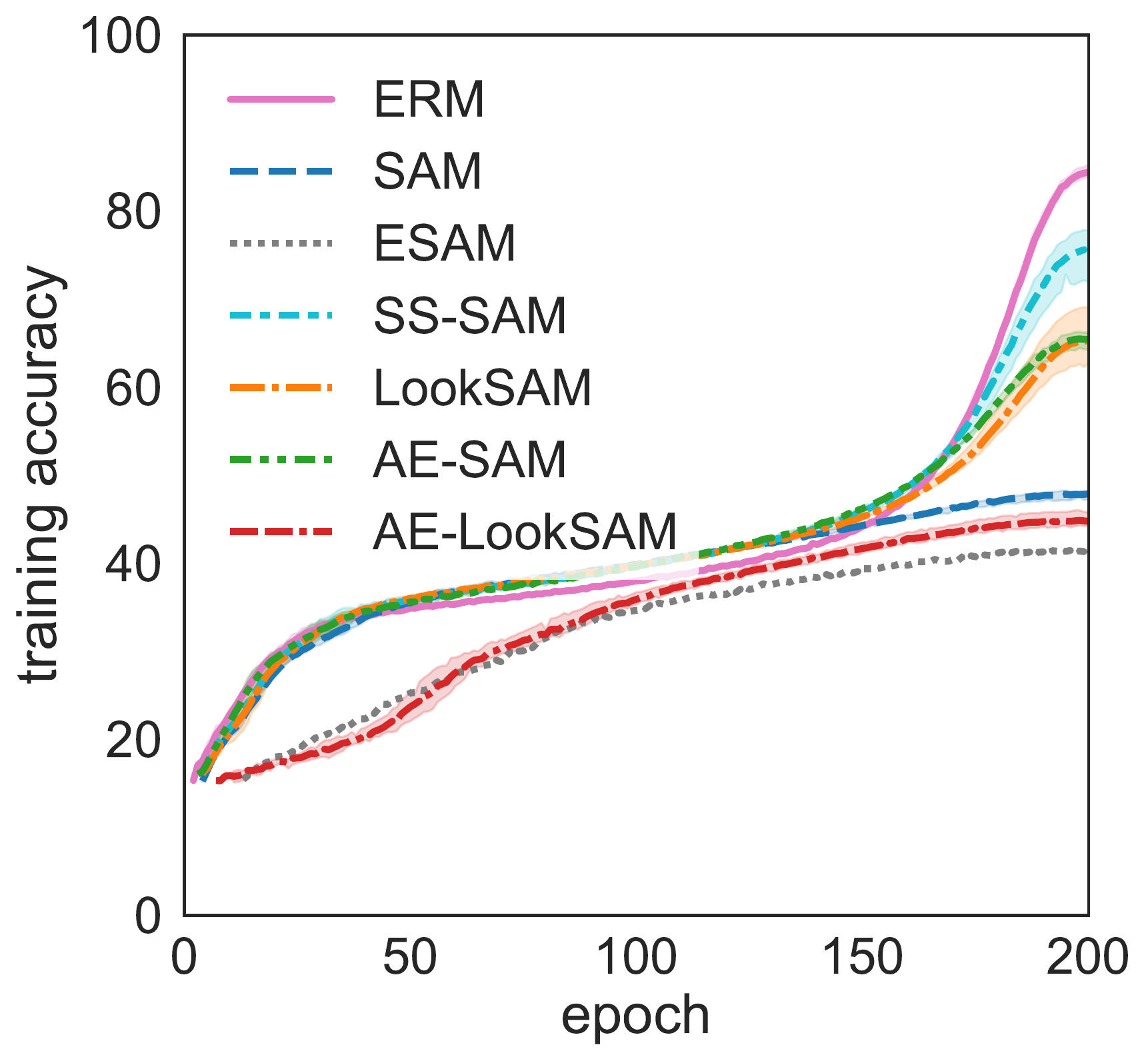}} \!\!
		\subfigure[60\% (Testing). \label{fig:noisy-cifar10-60-resnet32-valid}]{\includegraphics[width=0.25\textwidth]{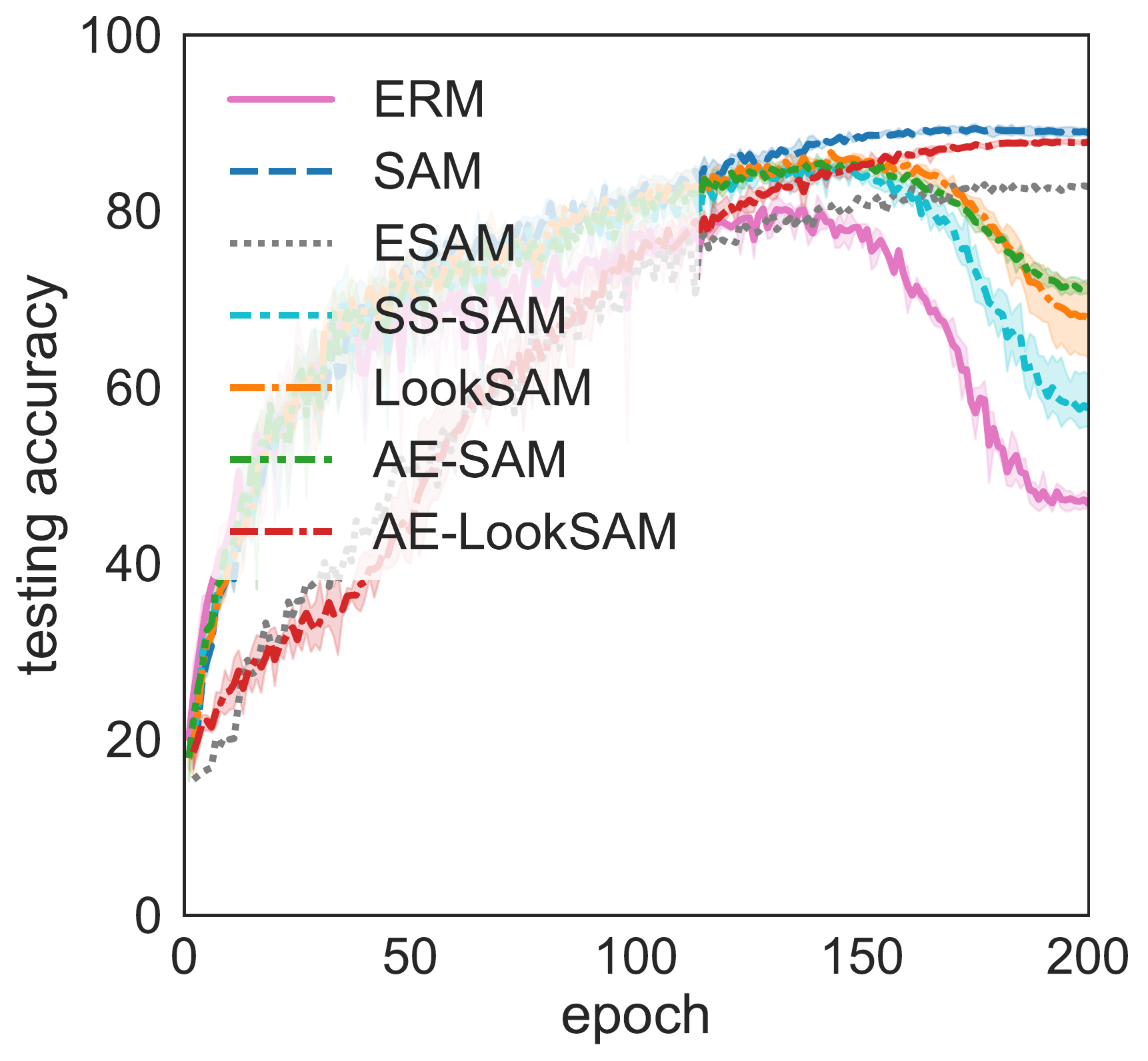}} \!\!
		\subfigure[80\% (Training). \label{fig:noisy-cifar10-80-resnet32-train-apd}]{\includegraphics[width=0.25\textwidth]{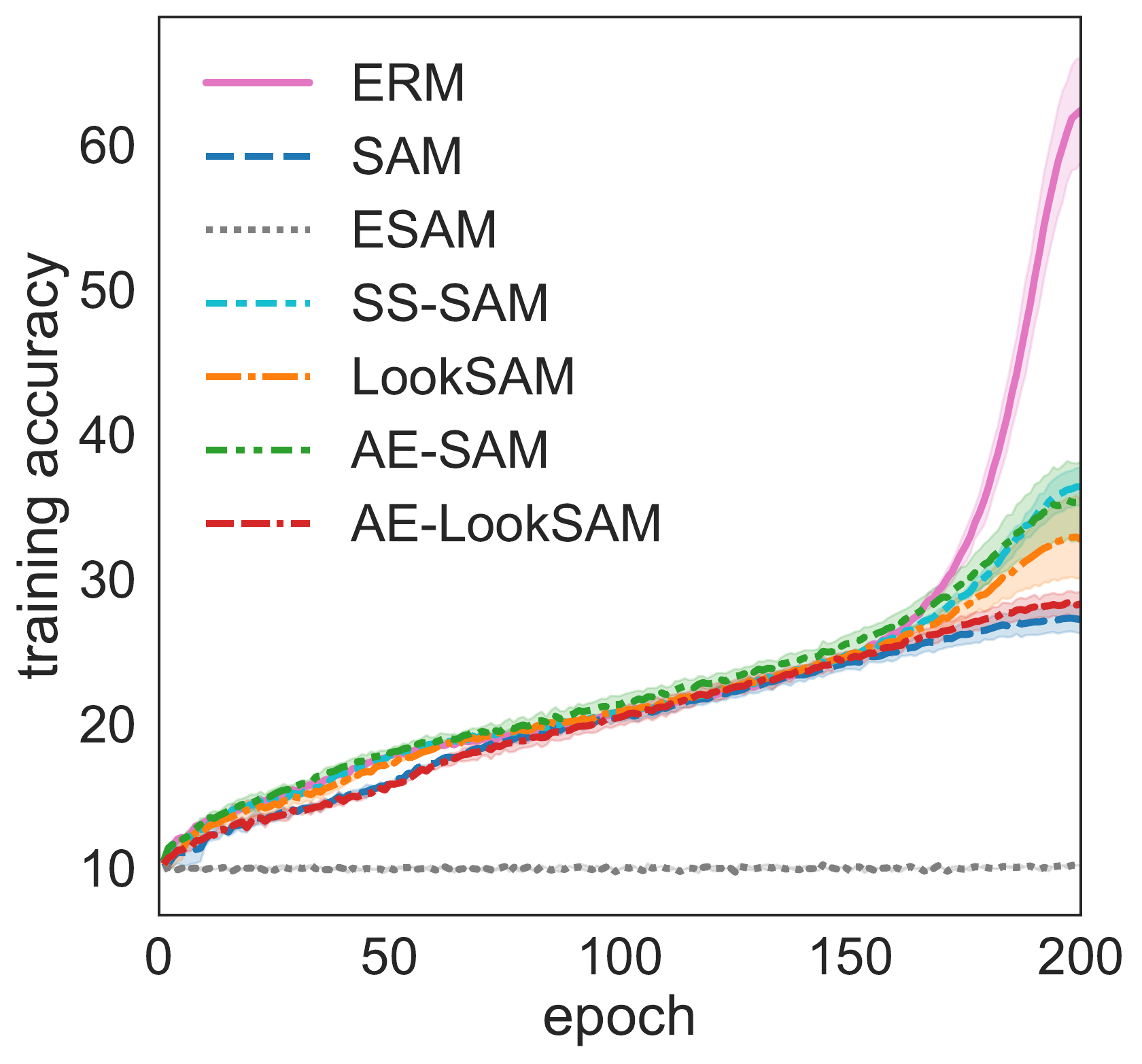}} \!\!
		\subfigure[80\% (Testing). \label{fig:noisy-cifar10-80-resnet32-valid-apd}]{\includegraphics[width=0.25\textwidth]{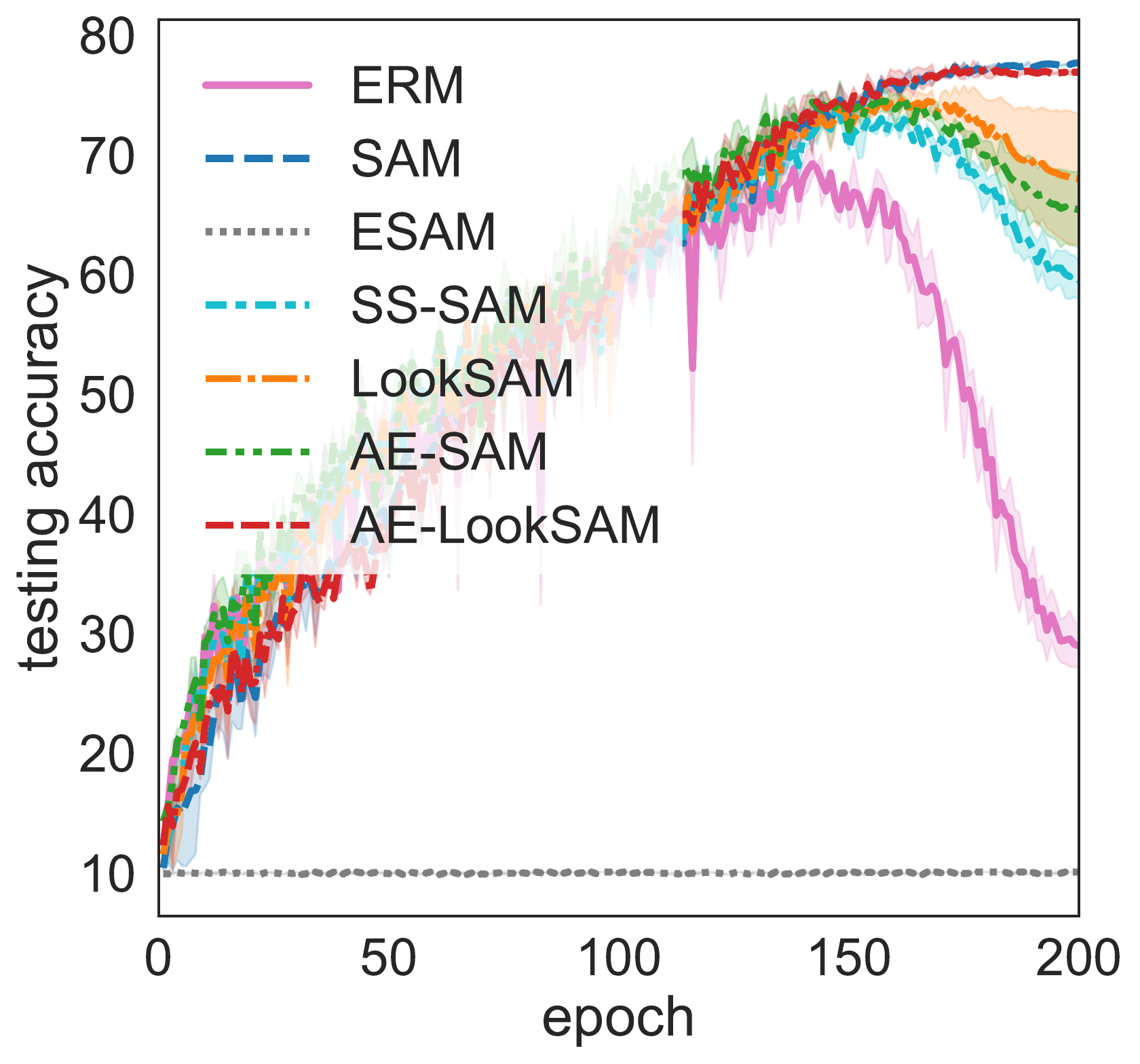}} \!\!\!
		\vskip -.15in
		\caption{Accuracies with number of epochs on CIFAR-10 with $20\%, 40\%$, $60\%$, and $80\%$ noise level using \textit{ResNet-32}.
			Best viewed in color.
		}
		\label{fig:curve-noisy-resnet32-apd}
		\vskip -.1in
	\end{figure}
	
	\subsection{Effects of $\lambda_1$ and $\lambda_2$ on AE-LookSAM}
	\label{sec:lambda}
	In this experiment,
	we study the effects of 
	$\lambda_1$ and $\lambda_2$ on AE-LookSAM.
	Experiment is performed 
	on 
	\textit{CIFAR-10}
	with label noise ($80\%$ noisy labels), 
	using the same setup as in Section \ref{sec:noisy-level}.
	
	Figure \ref{fig:ablation-samfrac-noisy}
	shows the effects of $\lambda_1$
	and $\lambda_2$ 
	on the fraction of SAM updates.
	Again, as in Section \ref{sec:ablation study ae-sam},
	for a fixed $\lambda_2$,
	increasing $\lambda_1$ always 
	reduces the fraction of SAM updates.
	Figure \ref{fig:ablation-acc-noisy}
	shows the effects of $\lambda_1$
	and $\lambda_2$ 
	on the testing accuracy of AE-SAM.
	As can be seen,  the observations are similar to those
	in Section \ref{sec:ablation study ae-sam}.

	\begin{figure}[!h]
		\centering
		\begin{minipage}{.48\textwidth}
			\centering
			\vskip -.1in 
			\!\!\!
			\subfigure[\textit{ResNet-18}.\label{fig:ablation-noisycifar-resnet18-samfrac}]{\includegraphics[width=0.48\textwidth]{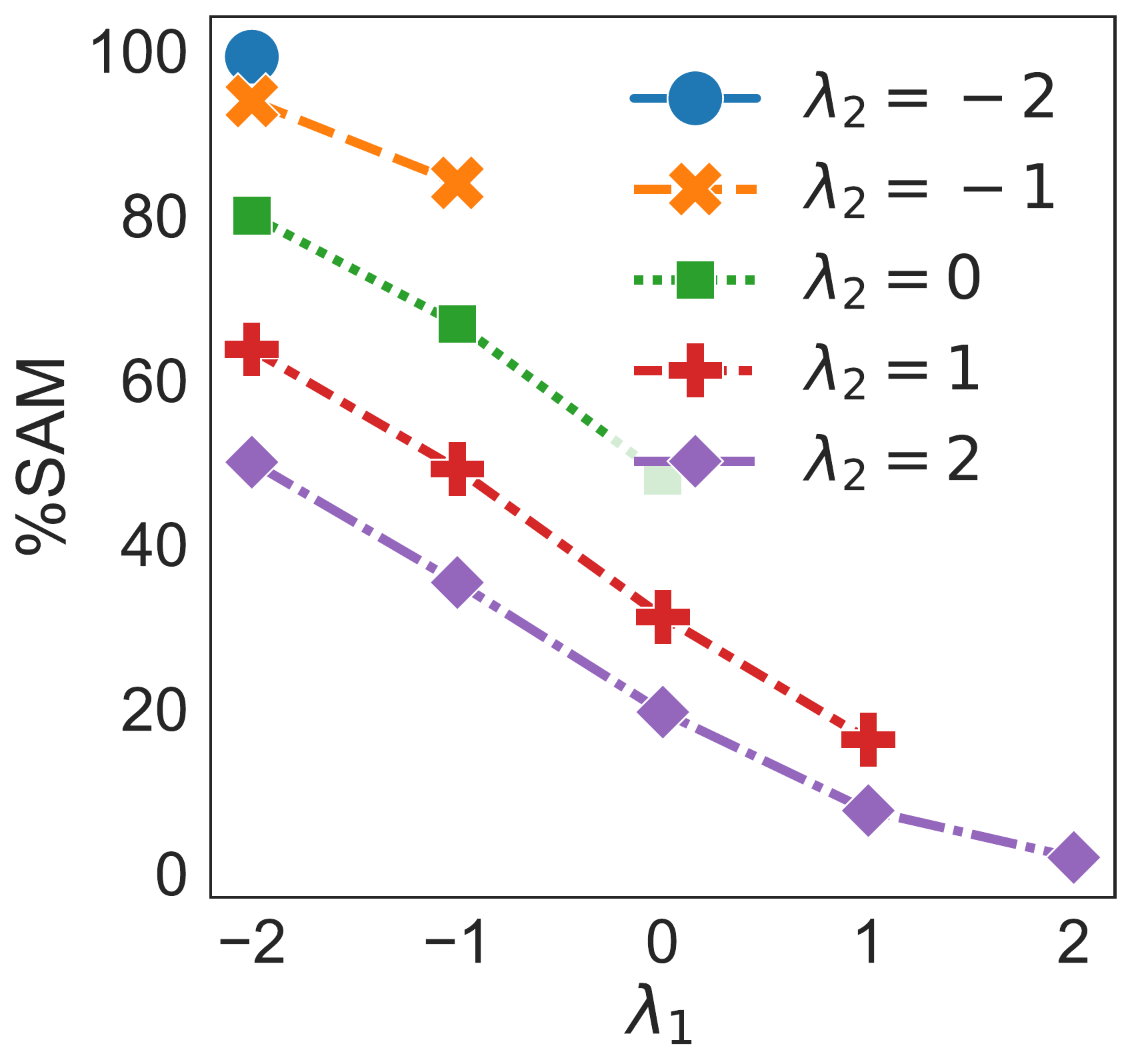}}	
			\subfigure[\textit{ResNet-32}. \label{fig:ablation-noisycifa-resnett34r-samfrac}]{\includegraphics[width=0.48\textwidth]{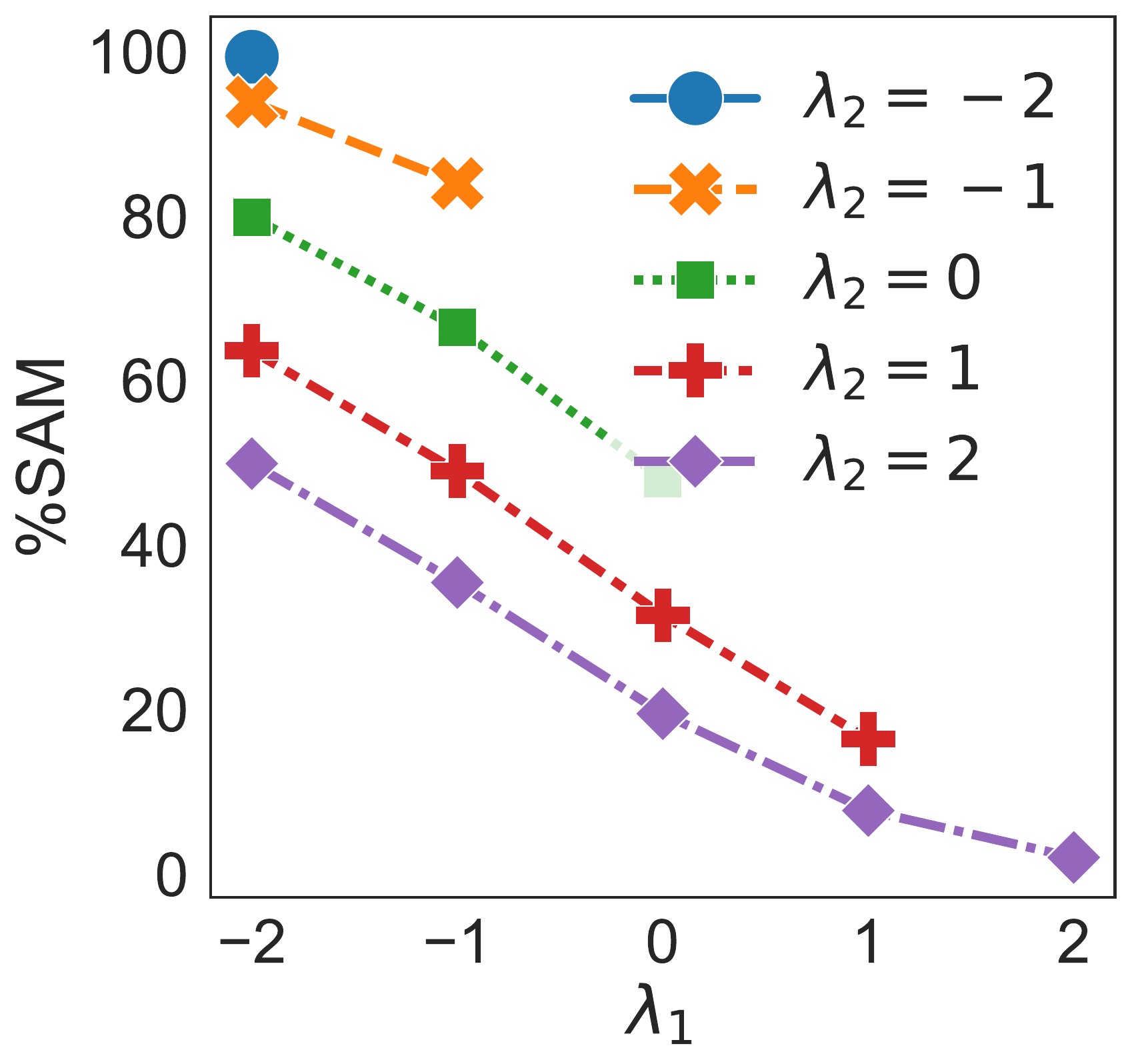}}
			\!\!\!
			\vskip -.15in
			\caption{
				Effects of $\lambda_1$ and $\lambda_2$ on fraction of SAM updates
				on \textit{CIFAR-10} (with $80\%$ noisy labels). 
				Best viewed in color.
			}
			\label{fig:ablation-samfrac-noisy}
		\end{minipage} \hfill
		\begin{minipage}{.48\textwidth}
			\centering
			\vskip -.1in 
			\!\!\!
			\subfigure[\textit{ResNet-18}.\label{fig:ablation-noisycifar-resnet18-acc}]{\includegraphics[width=0.48\textwidth]{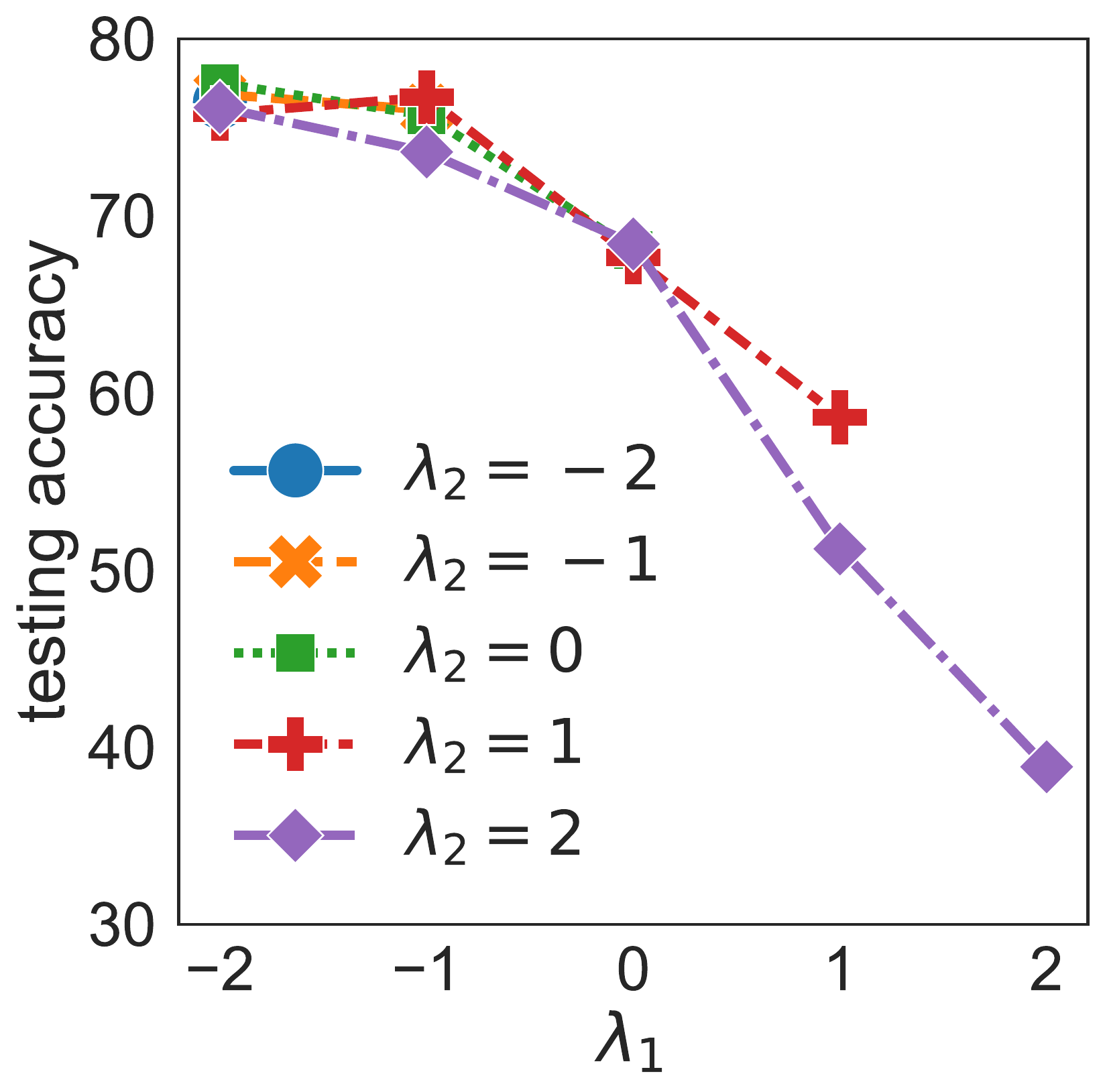}}	
			\subfigure[\textit{ResNet-32}. \label{fig:ablation-noisycifa-resnett34r-acc}]{\includegraphics[width=0.48\textwidth]{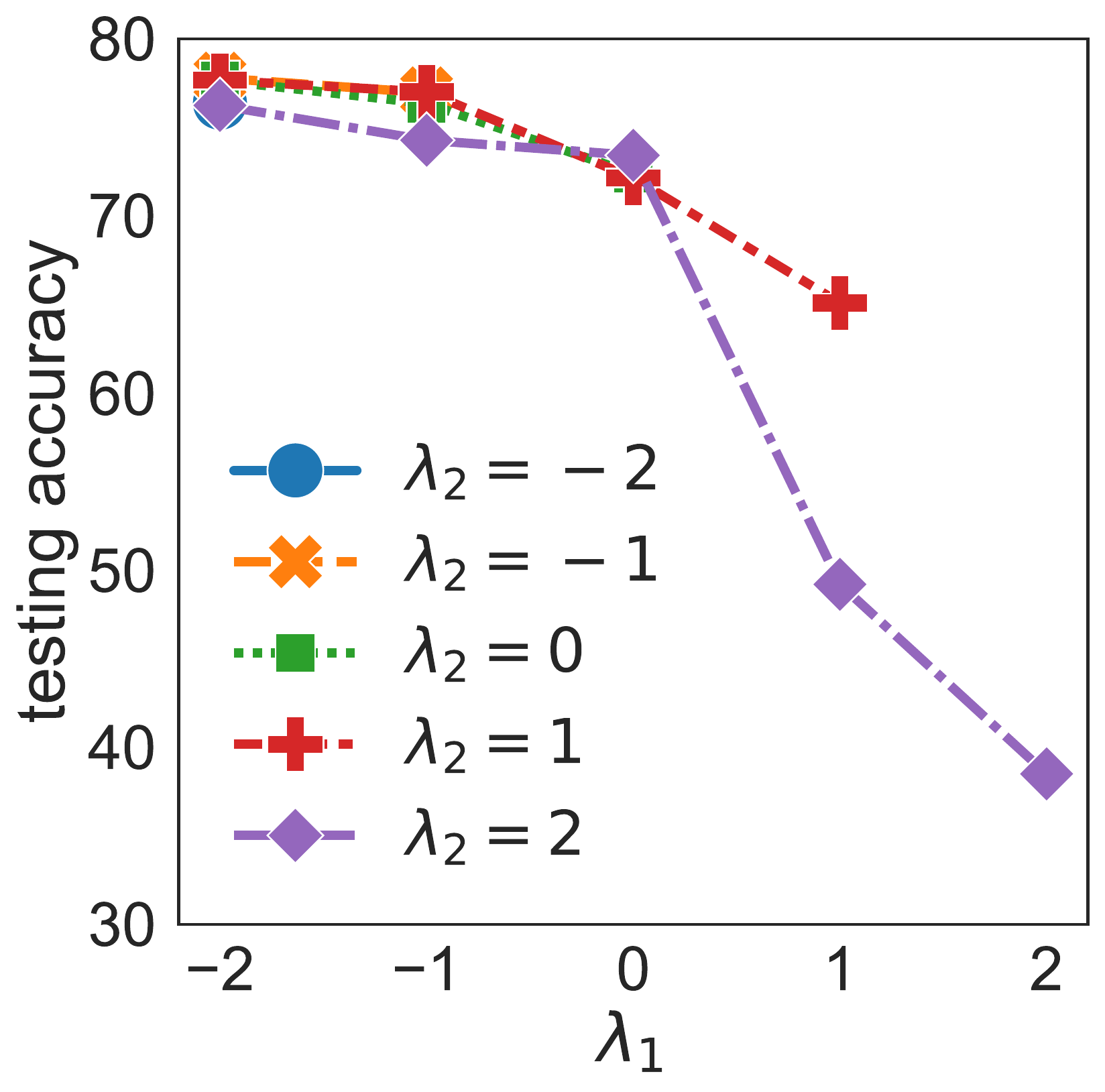}}
			\!\!\!
			\vskip -.15in
			\caption{
				Effects of $\lambda_1$ and $\lambda_2$ on testing accuracy
				of \textit{CIFAR-10} (with $80\%$ noisy labels). 
				Note that 
				the curves for $\lambda_2 \in \{-2, -1\}$ overlap completely with that of $\lambda_2 = 1$.
				Best viewed in color.
			}
			\label{fig:ablation-acc-noisy}
		\end{minipage}
	\end{figure}

	\subsection{Additional Convergence Results on \textit{CIFAR-10} and \textit{CIFAR-100}}
	\label{sec:appendix-expt-result}
	
	\vskip -.1in
	Figure \ref{fig:loss-trend}
	shows convergence of AE-SAM's training loss on the \textit{CIFAR-10} and \textit{CIFAR-100} datasets.
	As can be seen, AE-SAM achieves convergence with various network architectures.
	
	Figure \ref{fig:conv-compar-loss-trend}
	shows the training losses w.r.t. the number of epochs for AE-SAM and SS-SAM.
	As can be seen,
	AE-SAM and SS-SAM converge with comparable speeds,
	which agrees with Theorem \ref{thm:conv-sgd}
	as both of them have comparable fractions of SAM updates (Table \ref{table:result-cifar}).
	
	\begin{figure}[!h]
		\centering
		\vskip -.15in
		\!\!
		\subfigure[\textit{ResNet-18}. \label{fig:loss-trend-resnet}]{\includegraphics[width=0.28\textwidth]{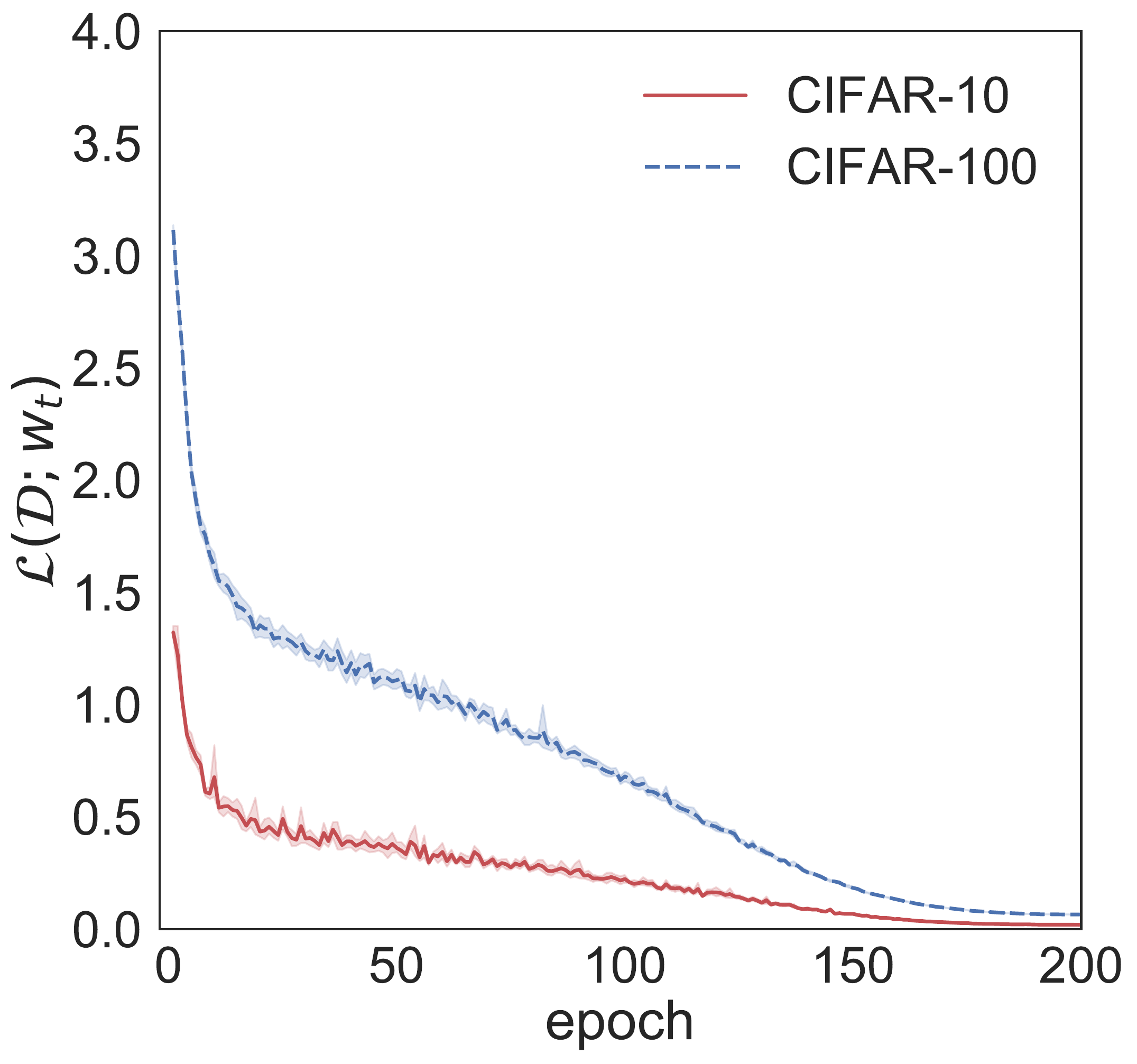}}
		\!\!\!
		\subfigure[\textit{WRN-28-10}. \label{fig:loss-trend-wrn}]{\includegraphics[width=0.28\textwidth]{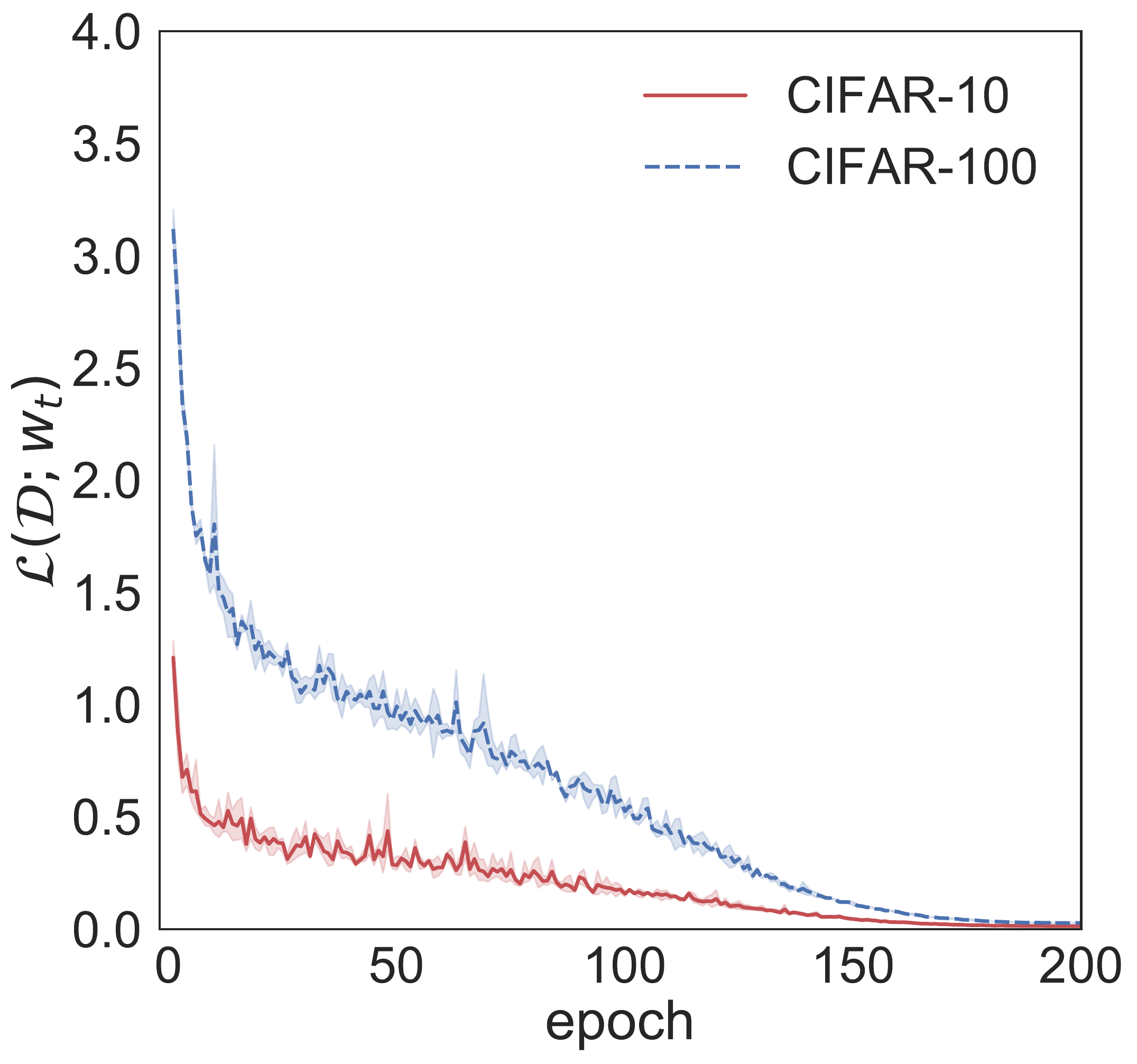}}
		\!\!\!
		\subfigure[\textit{PyramidNet-110}. \label{fig:loss-trend-pyramidnet}]{\includegraphics[width=0.28\textwidth]{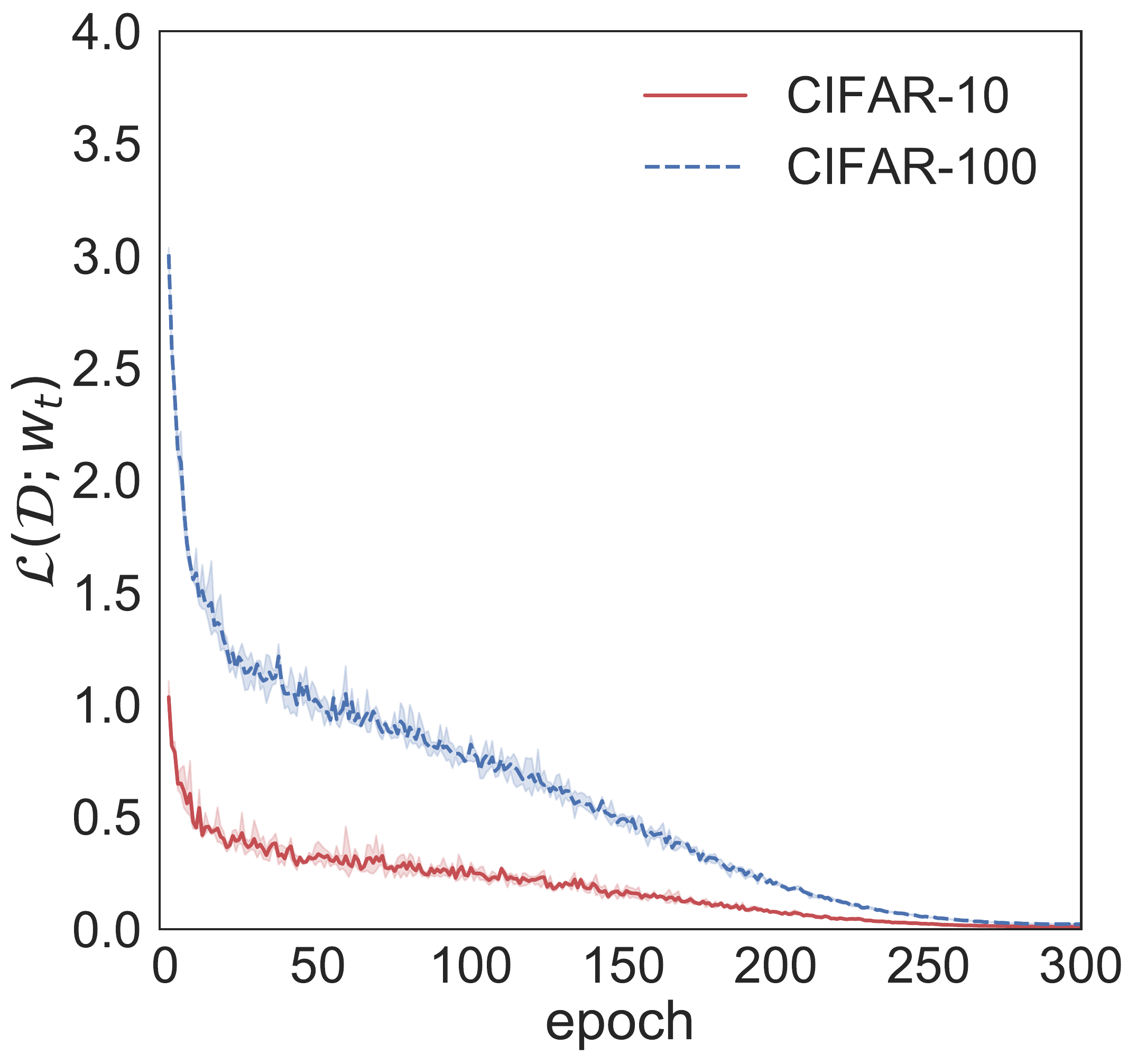}}
		\!\!\!
		\vskip -.15in
		\caption{Training loss of AE-SAM with number of epochs on \textit{CIFAR-10} and \textit{CIFAR-100}.
			Best viewed in color.
		}
		\label{fig:loss-trend}
	\end{figure}
	
	\begin{figure}[!h]
		\centering
		\vskip -.15in
		\!\!
		\subfigure[\textit{ResNet-18}. \label{fig:conv-loss-resnet-cifar10}]{\includegraphics[width=0.28\textwidth]{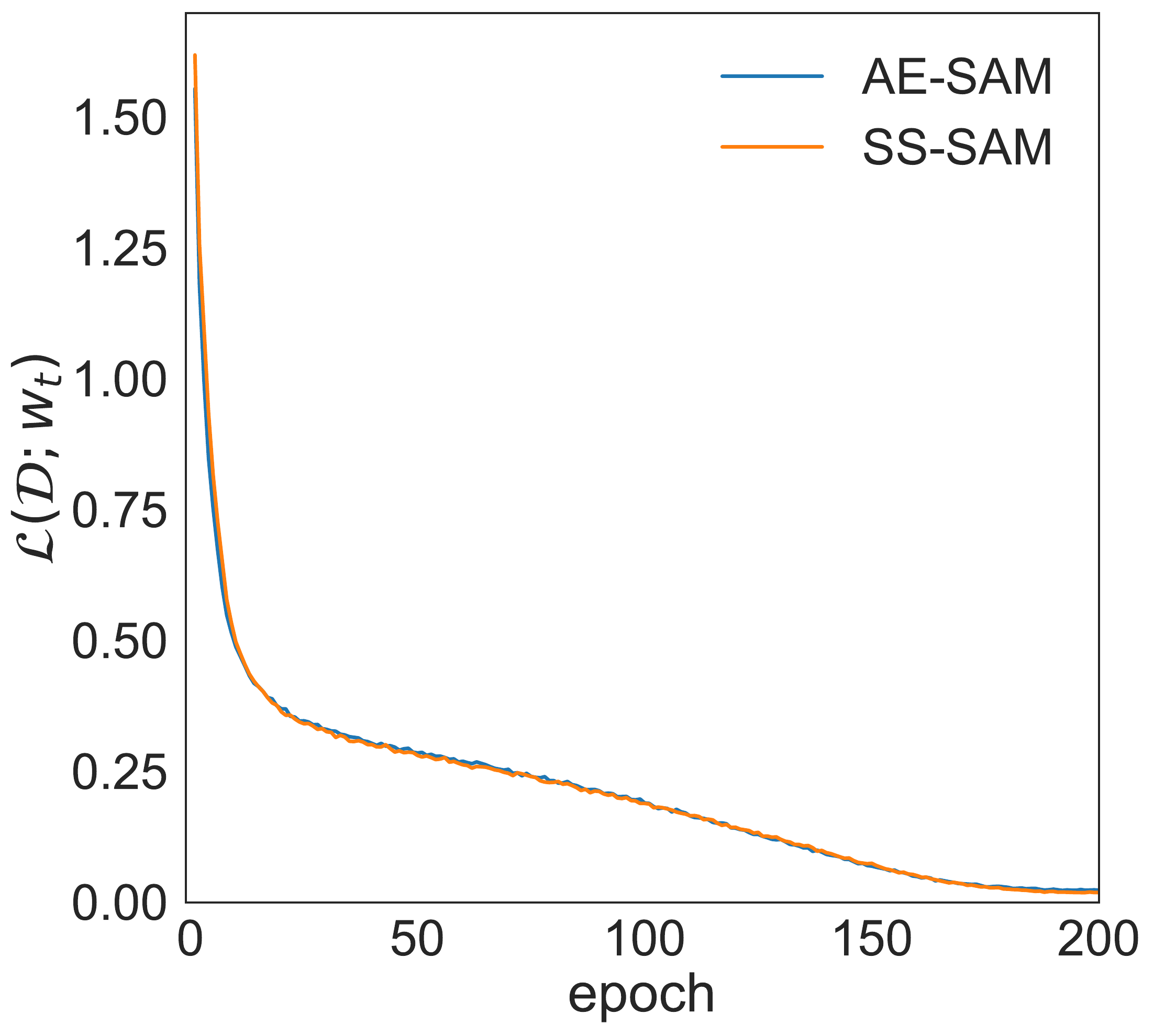}}
		\!\!\!
		\subfigure[\textit{WRN-28-10}. \label{fig:conv-loss-wrn-cifar10}]{\includegraphics[width=0.28\textwidth]{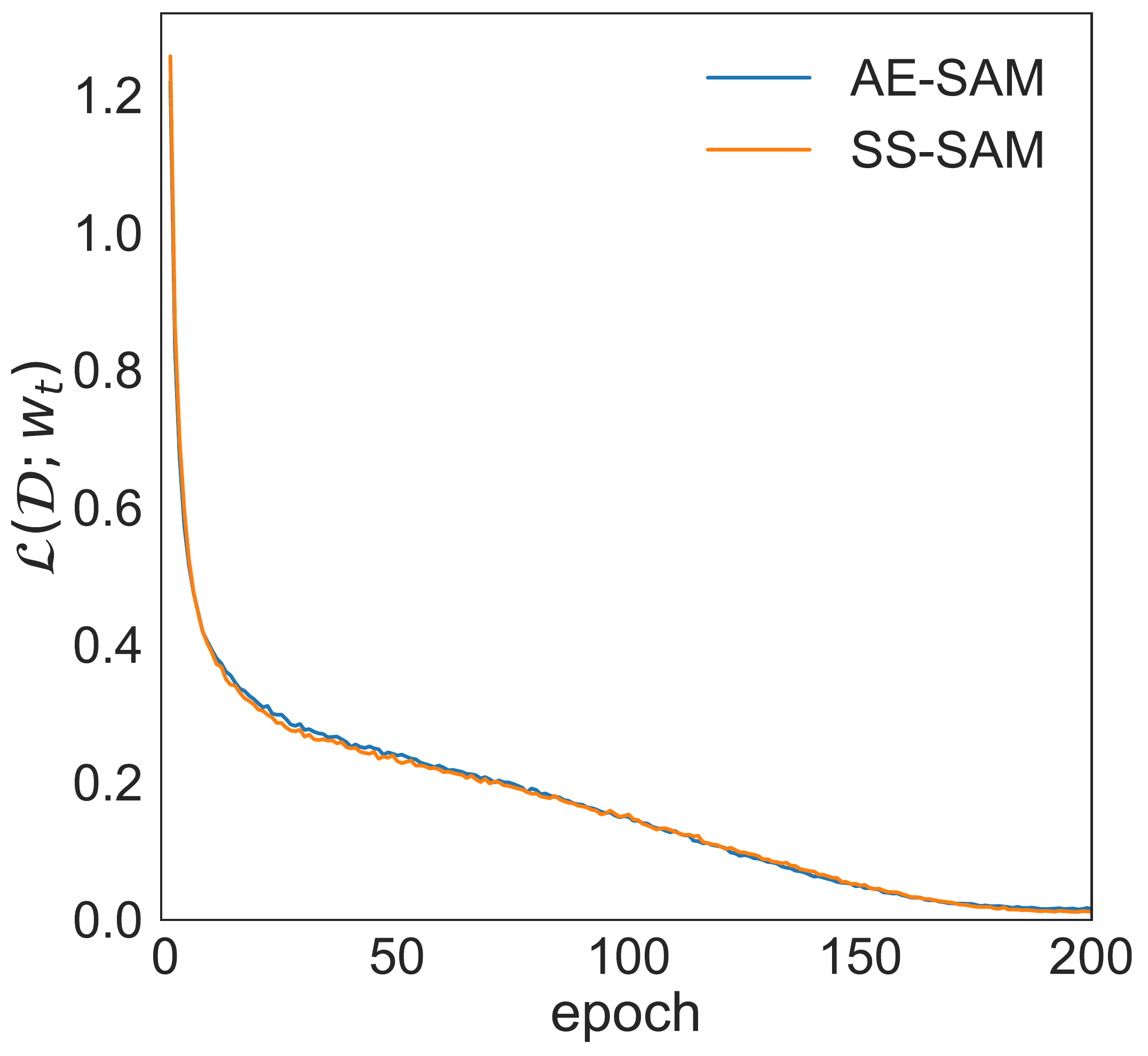}}
		\!\!\!
		\subfigure[\textit{PyramidNet-110}. \label{fig:conv-loss-pyramidnet-cifar10}]{\includegraphics[width=0.28\textwidth]{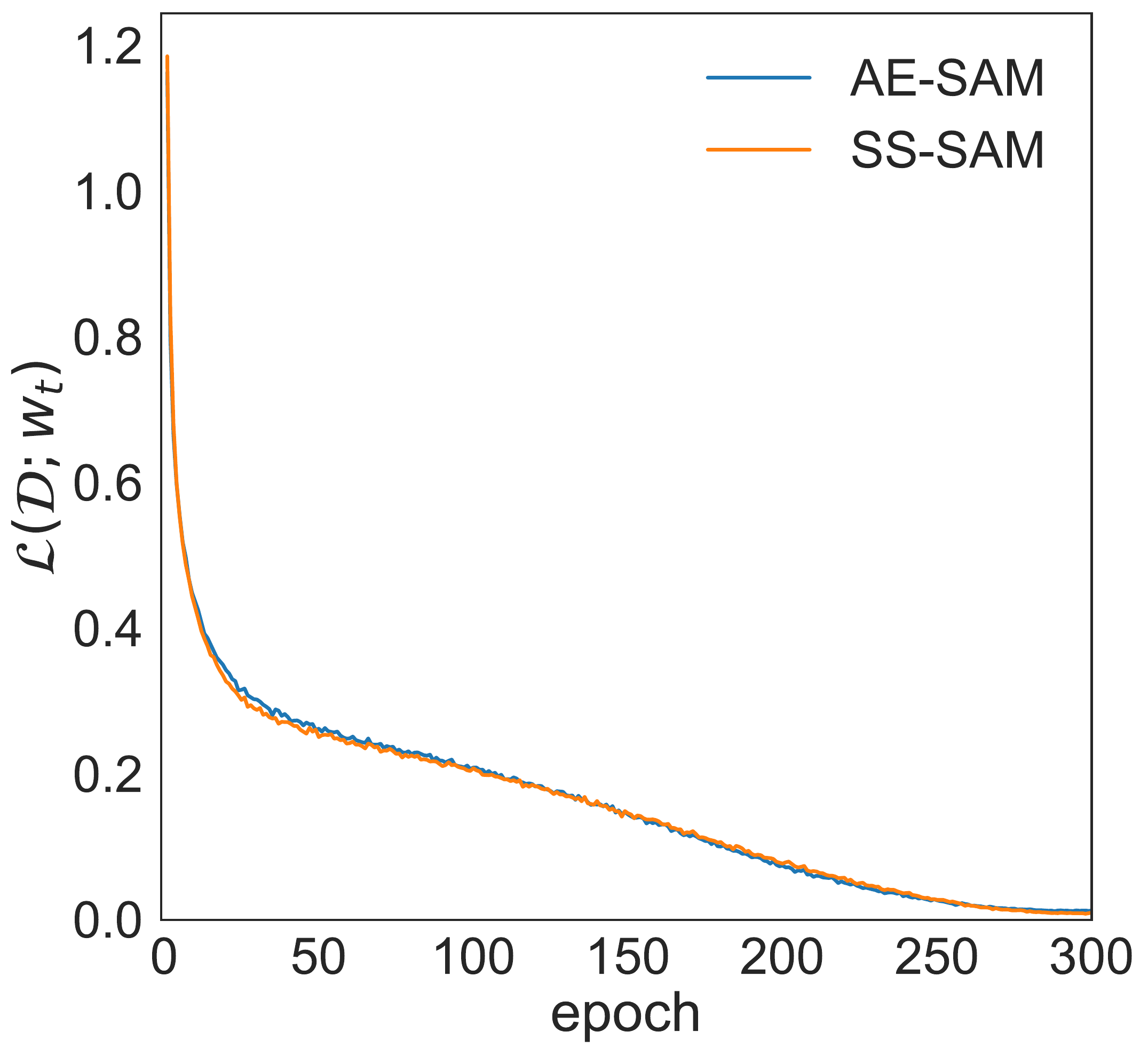}}
		\!\!
		\vskip -.15in
		\caption{
			Training losses of AE-SAM and SS-SAM with number of epochs on
			\textit{CIFAR-10}. Note that the two curves almost completely overlap.
			Best viewed in color.
		}
		\label{fig:conv-compar-loss-trend}
	\end{figure}

\end{document}